\documentclass{article}
\usepackage[main,final,nonatbib]{neurips_2026}
\usepackage{cite}
\usepackage[utf8]{inputenc} 
\usepackage[T1]{fontenc}    
\usepackage{url}            
\usepackage{booktabs}       
\usepackage{amsfonts}       
\usepackage{nicefrac}       
\usepackage{microtype}      
\usepackage{xcolor}         
\usepackage{pifont}
\usepackage[export]{adjustbox}
\usepackage{wrapfig}
\usepackage{graphicx}
\usepackage{colortbl}
\usepackage{amsmath}
\usepackage{amsthm}
\usepackage[table]{xcolor}
\usepackage{multirow}
\usepackage{bm}
\usepackage{array}
\usepackage{makecell}
\usepackage[algo2e,ruled]{algorithm2e}
\usepackage{algorithm}
\definecolor{theyellow}{RGB}{255,240,193}
\colorlet{myyellow}{theyellow!60}
\colorlet{lightyellow}{theyellow!30}
\definecolor{DarkBlue}{RGB}{0,0,128}
\definecolor{LightBlue}{RGB}{64,101,149}
\definecolor{cvprblue}{rgb}{0.21,0.49,0.74}
\definecolor{thegray}{RGB}{230,235,235}
\colorlet{mygray}{thegray!40}
\definecolor{theblue}{RGB}{193,230,255}
\colorlet{lightblue}{theblue!36}
\colorlet{myblue}{theblue!70}
\definecolor{DeepRed}{rgb}{0.8, 0, 0}
\definecolor{norblue}{HTML}{4C78A8}
\definecolor{aboorange}{HTML}{C98A4A}
\usepackage{threeparttable}
\newcolumntype{I}{!{\vrule width 0.8pt}}
\makeatletter
\newcommand{\thickhline}{%
    \noalign {\ifnum 0=`}\fi \hrule height 0.5pt
    \futurelet \reserved@a \@xhline
}
\makeatother
\definecolor{RedOrange}{rgb}{1.0, 0.27, 0.0}
\newcommand{\myred}[1]{$_{\color{RedOrange} \uparrow #1}$}
\newcommand{\mydown}[1]{$_{\color{RedOrange} \downarrow #1}$}

\definecolor{cvprblue}{rgb}{0.21,0.49,0.74}
\usepackage[pagebackref,breaklinks,colorlinks,citecolor=cvprblue]{hyperref}

\newtheorem{theorem}{Theorem}[section]

\newtheorem{corollary}[theorem]{Corollary}

\newtheorem{assumption}{Assumption}[section]

\newcommand{\eg}{\textit{e.g.}}
\newcommand{\ie}{\textit{i.e.}}

\title{Beyond Normal References: Discriminative Few-Shot Anomaly Detection}
\author{%
  Huan~Wang$^{1,2}$\thanks{This work was done while Huan was visiting SMU (\texttt{hw226@uowmail.edu.au}).},\;\; 
  Jun~Shen$^{2\dagger}$,\;\; Jun~Yan$^{2}$,\;\; Guansong~Pang$^{1}$\thanks{Corresponding Authors: Jun Shen (\texttt{jshen@uow.edu.au}), Guansong Pang (\texttt{gspang@smu.edu.sg}).}\\
  $^1$Singapore Management University, Singapore\\
  $^2$University of Wollongong, Australia
}

\begin{document}

\vspace*{-3.5em} 
\maketitle

\begin{center}
\vspace{-3em}
\raisebox{-0.15em}{\includegraphics[height=1em]{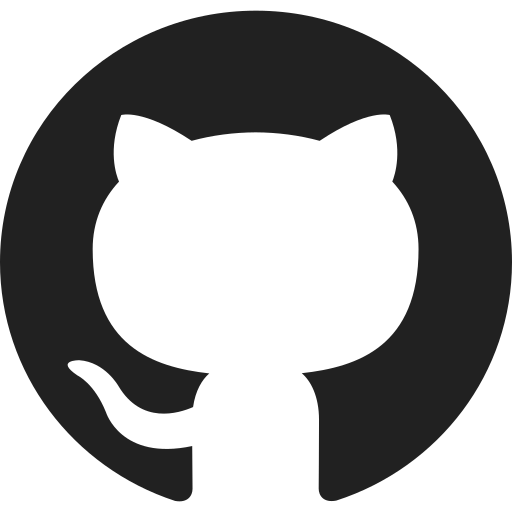}}
\hspace{0.2em}
\href{https://github.com/mala-lab/IDEAL}
{\texttt{https://github.com/mala-lab/IDEAL}}
\end{center}

\vspace{1em}

\begin{abstract}
  This paper considers a practical few-shot anomaly detection (FSAD) setting, termed \textbf{discriminative FSAD}, where a limited number of both normal and anomalous examples are available as references during inference. 
  Existing FSAD methods rely on normal-only references through normality matching, ignoring the discriminative clues in anomalous references, while directly fitting both references can overfit to the seen anomalies. 
  We introduce \textbf{IDEAL}, an \textbf{i}ntrinsic \textbf{de}vi\textbf{a}tion \textbf{l}earning framework that leverages both reference types to learn intrinsic deviation patterns characterizing generalizable abnormality as deviations from normality. 
  IDEAL decomposes the learning process into two novel components: 
  1) a \textbf{Normal Variation Eraser} to suppress nuisance normal variations that may lead to noisy deviations from normality, thereby highlighting anomaly-relevant deviation representations; 
  2) an \textbf{Intrinsic Deviation Encoder} to decompose these denoised deviation representations into intrinsic deviation vectors capturing the most discriminative orthogonal deviation directions. 
  During inference, IDEAL scores query-to-normal deviations preserved after projection onto the learned intrinsic deviation vectors, enabling generalization for both seen and unseen anomalies. 
  Extensive experiments on eight real-world datasets show that IDEAL generalizes effectively to unseen anomalies and consistently outperforms existing state-of-the-art FSAD methods.
\end{abstract}

\section{Introduction} \label{sec:myintro}
Anomaly detection (AD) \cite{pang2021deep,cao2024survey} aims to identify data points that significantly depart from expected patterns of normal data and has achieved a bright promise in many applications, such as industrial inspection \cite{wang2023multimodal,li2026iad} and medical image diagnosis \cite{zhang2020viral,fernando2021deep,huang2024adapting}. 
Conventional AD methods \cite{cohen2020sub,defard2021padim,roth2022towards,wang2019gods,yao2024hierarchical,yao2023one,you2022unified,he2024diffusion}, which adopt a full-shot setting, rely on access to a large collection of normal training samples. 
To alleviate this requirement, Few-Shot Anomaly Detection (FSAD), which assumes the availability of only a limited number of training samples, has emerged as a popular direction \cite{zhu2024toward,huang2022registration,li2024promptad,lv2025one,yao2024resad,zhai2026foundation,wang2022few,jeong2023winclip,lv2025metacan}. 
However, most existing FSAD methods still follow a specialist paradigm, where a dataset-specific training is required for each target data using few-shot normal samples \cite{huang2022registration,li2024promptad,xie2023pushing,fang2023fastrecon}. 
Although effective, this one-model-per-dataset paradigm requires substantial re-engineering and repeated re-training for each new deployment. 
This motivates the development of generalist FSAD \cite{zhu2024toward,yao2024resad,zhai2026foundation}, which aims to learn a single detector capable of generalizing across diverse datasets without specific retraining/tuning. 
It achieves this by leveraging few-shot normal samples from the target data as in-context normal references at test time.

While these normal references provide a profile of normality, they do not offer explicit discriminative evidence for abnormality \cite{pang2019deep,ding2022catching,yao2023explicit}. 
In practice, however, a limited number of anomalous examples can often be made available from historical records (\eg, images of confirmed tumor cases or defective products flagged by human inspectors). 
These anomalous examples encode valuable prior knowledge about anomalies in the target data, yet existing FSAD methods—designed to rely solely on normal references—cannot leverage this information.

Motivated by this gap, we consider a practical FSAD setting, termed \textbf{discriminative FSAD}, where a limited number of both normal and anomalous examples are available as references at inference time. 
Despite their potential utility, the use of these few-shot anomalous samples as additional references introduces a fundamental challenge: the few-shot abnormal references are inherently sparse and often fail to illustrate the full spectrum of anomaly types due to the strong heterogeneity in their distributions \cite{pang2021deep}. 
Consequently, naively conditioning on such references can bias the detector toward seen anomaly patterns, leading to overfitting and degraded generalization to unseen anomalies. 
As illustrated in Fig.~\ref{fig:fig_pro} (\textbf{b}) and (\textbf{c}), we implement KNN similarity matching to examine the feasibility of directly using the abnormal references in FSAD. 
Direct matching of the input image to either anomalous-only references or both normal-anomalous references can produce noisy and spurious anomaly activations, due to some major differences of abnormality in the input image and anomalous reference image. 
This indicates that abnormal references should not be treated as matching templates to identify anomalies, because they cover only some specific anomalous types.



\begin{wrapfigure}{r}{0.48\linewidth}
    \centering
    \setlength{\abovecaptionskip}{1.2mm}
    \vspace{-4mm}
    \includegraphics[width=\linewidth]{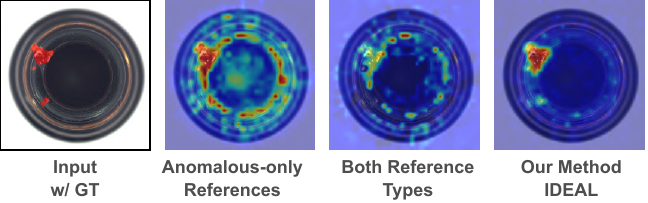}
    \caption{Anomaly score maps of \textbf{(a)} an input image using direct similarity matching of \textbf{(b)} anomalous-only references and \textbf{(c)} both anomalous and normal references, compared to \textbf{(d)} our method IDEAL under a 1-shot normal and anomalous reference setting. Two simple reference matching methods produce noisy and spurious anomaly responses, whereas IDEAL yields substantially cleaner activations. Note that, to illustrate the anomaly heterogeneity issue, the anomalous reference was intentionally chosen to be different from the abnormality in this input.}
\label{fig:fig_pro}
\vspace{-1mm}
\end{wrapfigure}
To address this issue, we propose \textbf{IDEAL}, an \textbf{i}ntrinsic \textbf{de}vi\textbf{a}tion \textbf{l}earning framework for discriminative FSAD, which leverages both normal and abnormal references to learn intrinsic deviation patterns that characterize generalizable abnormality as deviations from normality, thereby supporting the detection of both seen and unseen anomalies. 
The key insight of IDEAL is to learn a set of \textit{intrinsic deviation patterns} that capture discriminative, relative deviations from normality shared by both seen and unseen anomalies, even when unseen anomalies do not resemble the seen anomalies. 
Specifically, IDEAL decomposes the learning process into two novel components: 
1) a \textbf{Normal Variation Eraser (NVE)}, which suppresses nuisance normal variations (\textit{e.g.}, illumination/texture shifts) that may lead to noisy deviations from normality, thereby highlighting anomaly-relevant deviation representations; and 2) an \textbf{Intrinsic Deviation Encoder (IDE)}, which learns to decompose these denoised deviations into intrinsic deviation vectors capturing the most discriminative orthogonal deviation directions. 
This allows IDEAL to identify anomalies by measuring how well each query-to-normal deviation is preserved after projection onto these learned intrinsic deviation vectors, rather than by matching the query to the limited references, enabling effective generalization to unseen anomalies, as shown in Fig. \ref{fig:fig_pro} (\textbf{d}). 

The main contributions of this paper are summarized as follows:
\begin{itemize}
    \item We formulate discriminative FSAD, a practical extension of conventional FSAD, where limited normal and anomalous samples are both available as in-context references, aiming to exploit the discriminative clues of abnormal references relative to normal references.
    \item We propose an intrinsic deviation learning (IDEAL) framework for discriminative FSAD, which leverages both reference types to learn intrinsic deviation patterns characterizing generalizable abnormality as deviations from normality. 
    By measuring query-to-normal deviations preserved after projection onto these learned intrinsic deviation vectors, IDEAL effectively detects both seen and unseen anomalies.
    \item Comprehensive experiments on eight real-world datasets show the superiority of IDEAL in terms of performance, efficiency, and generalizability against state-of-the-art FSAD methods. Extensive ablation studies also validate the contribution of the two proposed components.
\end{itemize}


\section{Related Work}
\paragraph{Conventional Anomaly Detection.}
Anomaly detection approaches have been developed for diverse data and application scenarios \cite{pang2021deep,cao2024survey,chalapathy2019deep,samariya2023comprehensive}. 
In particular, one-class classification methods aim to learn a compact normality boundary from the given category using support vectors \cite{bergman2020classification,chen2022deep,Ruff2020Deep,yi2020patch}. 
Reconstruction-based methods \cite{akcay2018ganomaly,guo2025dinomaly,park2020learning,schlegl2019f,hou2021divide,liu2023diversity,luo2025exploring,damm2025anomalydino} capture normal patterns by reconstructing normal samples in feature space, with reconstruction discrepancies serving as indicators of abnormality. 
Distance-based methods detect anomalies by measuring the feature-level discrepancy between the query and the normal training distribution \cite{cohen2020sub,pang2018learning,roth2022towards,defard2021padim,tayeh2020distance}. 
Knowledge distillation methods \cite{bergmann2020uninformed,cao2023anomaly,salehi2021multiresolution,zhang2023destseg,deng2022anomaly} train a student network on normal samples to mimic the feature responses of a pre-trained teacher network, and detect anomalies using the resulting teacher-student discrepancies. 
Synthetic-based methods \cite{li2021cutpaste,zavrtanik2021draem,liu2023simplenet,zhang2024realnet} introduce anomaly-like supervision by corrupting normal samples or perturbing normal features, and then train the model to distinguish normal from synthetic abnormal patterns. 
Further, there have been supervised methods \cite{pang2019deep,ding2022catching,zhu2024anomaly,Ruff2020Deep,wang2025distribution} that explored the use of few-shot labeled anomalies to improve anomaly detection. 
However, these methods require large normal training data and dataset-specific training for each deployment, limiting their applicability in real-world scenarios where target data are difficult/costly to access \cite{alabdulatif2017privacy,mayer2020privacy}.
\paragraph{Few-Shot Anomaly Detection.}
Few-shot anomaly detection (FSAD) aims to reduce the dependence on full-shot normal training data by using a limited number of training samples. 
Existing FSAD methods are mainly divided into two aspects: specialist FSAD and generalist FSAD. 
Specialist FSAD \cite{huang2022registration,li2024promptad,wang2022few,xie2023pushing,fang2023fastrecon,liao2024coft,tian2024foct,tao2025kernel} focuses on adapting a detector for each category/dataset using few available normal samples. For example, RegAD \cite{huang2022registration} learns normal representations via a registration proxy task and detects anomalies by comparing registered query-support features. 
However, these methods still rely on category- or dataset-specific fitting, leading to a specialist one-model-per-dataset paradigm with limited scalability. 
To overcome this limitation, generalist FSAD \cite{zhu2024toward,yao2024resad,zhai2026foundation,lv2025one,jeong2023winclip,lv2025metacan,dong2026dual} aims to learn a single detector that can generalize across diverse datasets without target-specific retraining. InCTRL \cite{zhu2024toward} performs in-context residual learning with few-shot normal sample prompts and ResAD \cite{yao2024resad} applies residual features to eliminate domain dependencies, to achieve generalizable detection. 
From the perspective of using references, NAGL \cite{wang2025normal} extends FSAD by introducing anomalous references and performing residual mining and learning to detect anomalies. 
In addition, many zero-shot AD methods \cite{jeong2023winclip,zhou2023anomalyclip,zhu2024fine,li2023zeroshot,xu2025towards,ma2025aaclip,gu2024filo} have been introduced recently, sharing the spirit of adapting large pre-trained models to detect anomalies, but they deal with a fundamentally different scenario where no reference samples are available.

\section{Problem Formulation}
\paragraph{Generalist FSAD.}
According to \cite{zhu2024toward}, given a set of training data $\mathcal{D}_{tr}=\{\mathcal{X}, \mathcal{Y}\}$, where $\mathcal{X}$ consists of several normal and abnormal images and $\mathcal{Y}$ denotes corresponding labels. At inference, $x$ is a test query from a target class, accompanied by few-shot normal references $\mathcal{S}^{n}$. The goal is to train a detection model on $\mathcal{D}_{tr}$ that can generalize to detect anomalies of $x$ with $\mathcal{S}^{n}$.
\paragraph{Discriminative FSAD.}
Given a set of auxiliary training data $\mathcal{D}_{tr}=\{\mathcal{X}, \mathcal{Y}\}$, where $\mathcal{X}=\{\mathbf{x}_z\}_{z=1}^{Z}$ consists of $Z$ normal and abnormal samples, $\mathcal{Y}=\{\mathbf{y}_z\}_{z=1}^{Z}$ as the corresponding labels ($\mathbf{y}_z=0$ indicates normal, $\mathbf{y}_z=1$ signifies abnormal),
during \textbf{training}, we first randomly sample a few-shot reference set $\mathcal{S}=\{\mathcal{S}^{n}, \mathcal{S}^{a}\}$ from $\mathcal{D}_{tr}$ for query input $\mathbf{x}^{q} \in \mathcal{X}$, where $\mathcal{S}^{n}=\{\mathbf{s}_{l}^{n}\}_{l=1}^{L_{1}}$ and $\mathcal{S}^{a}=\{\mathbf{s}_{l}^{a}\}_{l=1}^{L_{2}}$ represent $L_{1}$-shot normal and $L_{2}$-shot abnormal references, respectively; 
we then train a detection model $\Phi:\mathcal{D}_{tr} \rightarrow \mathbb{R}$ based on different input \textbf{episodes} $\mathbf{E}=\{\mathbf{x}^{q}, \mathcal{S}^{n}, \mathcal{S}^{a}\}$. 
At \textbf{inference}, the model $\Phi$ is directly evaluated on the target dataset without any further tuning/retraining. 
Note that to replicate real-world applications where reference examples are the same to all test samples, for any new dataset $c$, the few-shot reference set $\mathcal{S}_c=\{\mathcal{S}^{n}_c, \mathcal{S}^{a}_c\}$ is fixed, where $\mathcal{S}^{a}_c$ is drawn from one randomly selected anomaly type and $\mathcal{S}^{n}_c$ consists of the few-few normal references sampled from the normal class. 
This setup differs from \cite{wang2025normal}, which dynamically resamples the abnormal references for each test query according to its anomaly type, which may leak information about the anomaly type. We set $L_{1} > L_{2}$, motivated by the fact that anomalies are rare in real-world applications.

Further, for each episode $\mathbf{E}=\{\mathbf{x}^{q}, \mathcal{S}^{n}, \mathcal{S}^{a}\}$, including a query image $\mathbf{x}^{q} \in \mathbb{R}^{H \times W \times 3}$ with a label $\mathbf{y}^{q}$ and a ground-truth mask $\mathcal{M}^{q}$ (each normal image is equipped with a zero mask), $L_{1}$-shot normal references $\mathcal{S}^{n}$ and $L_{2}$-shot abnormal references $\mathcal{S}^{a}$. We use a pre-trained encoder $\mathcal{E}(\cdot)$ to extract query features $\mathcal{F}^{q}=\mathcal{E}(\mathbf{x}^{q})=\{f_{i}^{q}\}_{i=1}^{N}$, normal features $\mathcal{F}^{n}=\mathcal{E}(\mathcal{S}^{n})=\{f_{i}^{n}\}_{i=1}^{L_{1}\cdot N}$, and abnormal features $\mathcal{F}^{a}=\mathcal{E}(\mathcal{S}^{a})=\{f_{i}^{a}\}_{i=1}^{L_{2}\cdot N}$, where $f_{i}^{q}, f_{i}^{n}, f_{i}^{a} \in \mathbb{R}^{\mathbb{C}}$, $N$ is the number of patches, $\mathbb{C}$ is the dimension. For $\mathcal{S}^{a}$, we down-sample its anomaly masks to match $\mathcal{F}^{a}$, yielding $\mathcal{M}^{a} \in \{0,1\}^{L_{2} \cdot N}$.

\begin{figure*}[t]
    \centering
    \includegraphics[width=0.88\textwidth]{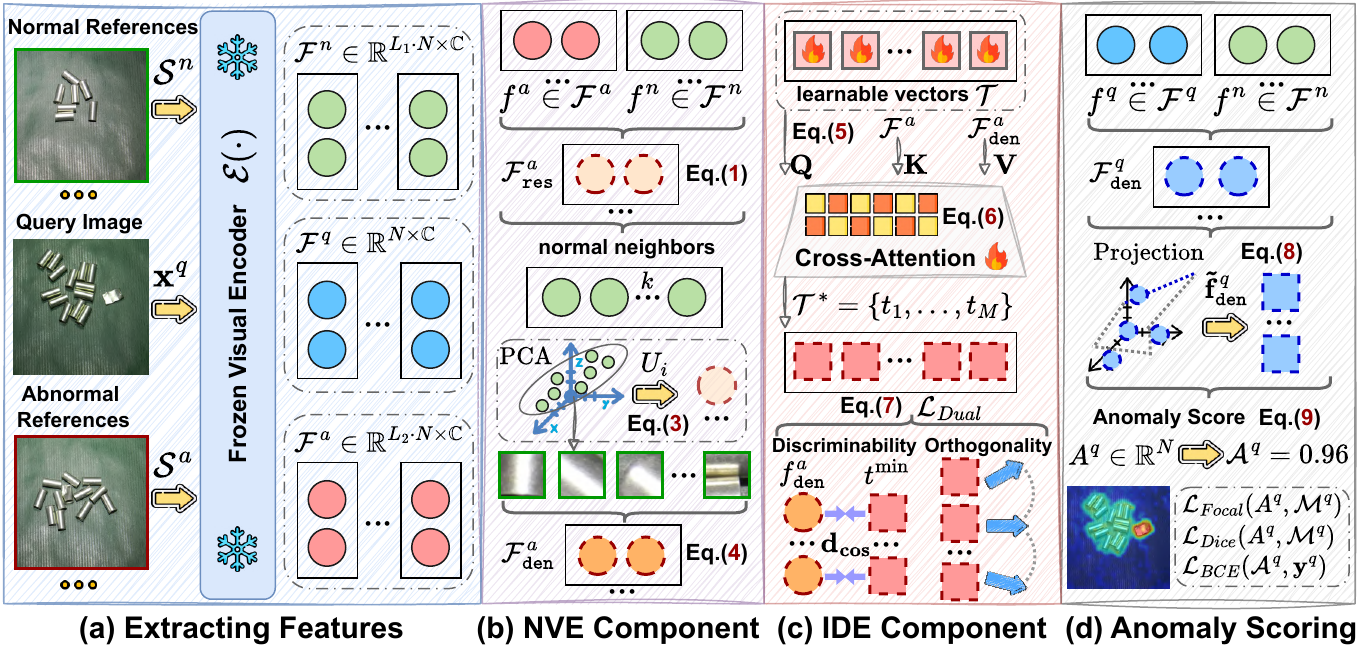}
    \caption{Overview of IDEAL (Algorithm~\ref{alg:ours}). 
    \textbf{(a)} Given each input $\mathbf{E}=\{\mathbf{x}^{q}, \mathcal{S}^{n}, \mathcal{S}^{a}\}$, we employ a pre-trained encoder $\mathcal{E}(\cdot)$ to extract features. We then introduce \textbf{(b)} a Normal Variation Eraser (Sec.~\ref{sec:nve}) and \textbf{(c)} an Intrinsic Deviation Encoder (Sec.~\ref{sec:ace}) to learn intrinsic deviation vectors from both references. 
    Finally, \textbf{(d)} anomalies are detected by measuring query-to-normal deviations preserved after projection onto these learned intrinsic deviation vectors (Sec.~\ref{sec:score}).}
    \label{fig:fig3}
\end{figure*}

\section{IDEAL: Intrinsic Deviation Learning for Discriminative FSAD}
\noindent\textbf{Overview of Our Proposed IDEAL Framework.} 
To effectively leverage the discriminative clues of abnormal references relative to normal references, IDEAL learns intrinsic deviation patterns characterizing generalizable abnormality as deviations from normality, to detect both seen and unseen anomalies. 
As shown in Fig.~\ref{fig:fig3}, given a query sample accompanied by few-shot normal and abnormal references, IDEAL first extracts patch-level features via a pre-trained visual encoder. 
Then, IDEAL proceeds with two key components: 
1) NVE (Sec.~\ref{sec:nve}), which suppresses nuisance normal variations that may lead to noisy deviations from normality to highlight anomaly-relevant deviation representations; 
2) IDE (Sec.~\ref{sec:ace}), which learns to decompose these denoised deviation representations into intrinsic deviation vectors capturing the most discriminative orthogonal deviation directions. 
Finally, the anomaly score is obtained by measuring query-to-normal deviations preserved after projection onto the learned intrinsic deviation vectors (Sec.~\ref{sec:score}).

\subsection{Normal Variation Eraser (NVE)} \label{sec:nve}
The key intuition of NVE is to isolate anomaly-relevant deviation representations from irrelevant/noisy deviations by eliminating nuisance normal variations that may lead to these noisy deviations. 
Hence, NVE operates in two steps: 1) it extracts feature residual-based deviations by contrasting abnormal and normal references, canceling out class-specific content; 2) it suppresses nuisance normal variations embedded in the residual deviations, producing denoised deviations that more faithfully reflect anomalous departures from normality.
\paragraph{Residual Deviations.}
For each patch-level abnormal feature $f_{i}^{a}$ of $\mathcal{F}^{a}$, we can search the nearest normal reference feature $f_{i,\min}^{n} := \mathtt{argmin}_{f^{n} \in \mathcal{F}^{n}} \mathbf{d_{cos}}(f^{n}, f_{i}^{a})$, where $\mathbf{d_{cos}}(\cdot,\cdot)$ denotes the cosine function and $i \in \{1, 2, \ldots, L_{2} \cdot N\}$. 
We define residual deviations between $\mathcal{F}^{a}$ and $\mathcal{F}^{n}$ as:
\begin{equation}
    \label{eq:eq1}
    \mathcal{F}^{a}_{\mathtt{res}} = \{f_{i}^{a} - \mathtt{argmin}_{f^{n} \in \mathcal{F}^{n}} \mathbf{d_{cos}}(f^{n}, f_{i}^{a})\}_{i=1}^{L_{2} \cdot N} \in \mathbb{R}^{L_{2} \cdot N \times \mathbb{C}}.
\end{equation}
\paragraph{Denoised Deviations.}
Although the residual deviations $\mathcal{F}^{a}_{\mathtt{res}}$ reveal departures from normality, they may still contain nuisance normal variations that may lead to noisy deviations from normality. 
To estimate those normal variation around each abnormal feature $f_{i}^{a}$ of $\mathcal{F}^{a}$, we first retrieve its top-$k$ nearest normal neighbors $\mathcal{F}^{n}_{i,k}$ from the normal features $\mathcal{F}^{n}$ as follows:
\begin{equation}
    \label{eq:eq2}
    \mathcal{F}^{n}_{i,k} = \mathtt{argsort}_{f^{n} \in \mathcal{F}^{n}}(\mathbf{d_{cos}}(f^{n}, f_{i}^{a}), k) = \{f_{i,j}^{n}\}_{j=1}^{k} \in \mathbb{R}^{k \times \mathbb{C}}.
\end{equation}
Based on $\mathcal{F}^{n}_{i,k}$, we can calculate the empirical mean $\mu_{i} \in \mathbb{R}^{\mathbb{C}}$ and covariance matrix $\Sigma_{i} \in \mathbb{R}^{\mathbb{C} \times \mathbb{C}}$. 
We then estimate a local normal variation subspace $\mathcal{R}_{i}$ from the centered neighbors $\{f_{i,j}^{n}-\mu_{i}\}_{j=1}^{k}$ using Principal Component Analysis (PCA) \cite{abdi2010principal,tipping1999prob}, which yields an orthonormal basis $U_{i}$ for $\mathcal{R}_{i}$. Accordingly, each local normal neighbor $f_{i,j}^{n}$ is approximated as:
\begin{equation}
    \label{eq:eq3}
    f_{i,j}^{n}=\mu_{i} + U_{i} v_{i,j} + \epsilon_{i,j}, \;\; \textnormal{s.t.} \;\; v_{i,j} \sim \mathcal{N}(0, \mathbf{I}_{r}), \; \epsilon_{i,j} \sim \mathcal{N}(0, \sigma^{2} \mathbf{I}_{\mathbb{C}}),
\end{equation}
where $U_{i} \in \mathbb{R}^{\mathbb{C} \times r}$ contains top-$r$ eigenvectors of $\Sigma_{i}$, $v_{i,j} \in \mathbb{R}^{r}$ is a latent variable, $\epsilon_{i,j}$ is an isotropic noise term, $\mathbf{I}_{r}$ is a $r \times r$ identity matrix, and $\mathbf{I}_{\mathbb{C}}$ is a $\mathbb{C} \times \mathbb{C}$ identity matrix. 
The columns of $U_{i} \in \mathbb{R}^{\mathbb{C} \times r}$ form an orthonormal basis for the local normal variation subspace $\mathcal{R}_{i}$, where $r$ denotes the number of retained principal components in $\mathcal{R}_{i}$ and typically satisfies $r < \min(k-1, \mathbb{C})$. 
Further, for each residual deviation $f_{i,\mathtt{res}}^{a} = (f_{i}^{a} - f_{i,\min}^{n}) \in \mathcal{F}^{a}_{\mathtt{res}}$ in Eq.~\eqref{eq:eq1}, we suppress its noisy component aligned with the local normal variation subspace $\mathcal{R}_{i}$, obtaining the denoised deviations:
\begin{equation}
    \label{eq:eq4}
    \mathcal{F}^{a}_{\mathtt{den}}=\{ f_{i,\mathtt{res}}^{a} - \alpha U_{i} U_{i}^{\top} f_{i,\mathtt{res}}^{a} \}_{i=1}^{L_{2} \cdot N}=\{f_{i,\mathtt{den}}^{a}\}_{i=1}^{L_{2} \cdot N} \in \mathbb{R}^{L_{2} \cdot N \times \mathbb{C}},
\end{equation}
where $\alpha \in (0,1]$ is used to controls the degree of elimination. 
In this way, the denoised deviations $\mathcal{F}^{a}_{\mathtt{den}} \in \mathbb{R}^{L_{2} \cdot N \times \mathbb{C}}$ suppress noisy components aligned with the local normal variation subspace while preserving anomaly-relevant components that are less explainable by normal variations.

\subsection{Intrinsic Deviation Encoder (IDE)} \label{sec:ace}
While denoised deviations highlight the anomalous departures from normality, directly employing similarity matching to compare query-to-normal deviations with these denoised deviations remains fragile, as it still depends on the limited abnormal references. 
To address this limitation, we introduce an Intrinsic Deviation Encoder (IDE) to learn intrinsic deviation vectors decomposed from these denoised deviation representations across different input episodes, thereby capturing the most discriminative orthogonal deviation directions shared by different anomalies. 
Specifically, we employ a cross-attention \cite{vaswani2017attention} mechanism to capture such intrinsic deviation vectors, where the three matrices $\mathbf{Q}$, $\mathbf{K}$, and $\mathbf{V}$ in the attention are defined as:
\begin{equation}
    \label{eq:eq5}
    \mathbf{Q} =\omega^{Q}(\mathcal{T}) \in \mathbb{R}^{M \times \mathbb{C}}, \; 
    \mathbf{K} =\omega^{K}(\mathcal{F}^{a}) \in \mathbb{R}^{L_{2} \cdot N \times \mathbb{C}}, \; 
    \mathbf{V} =\omega^{V}(\mathcal{F}^{a}_{\mathtt{den}}) \in \mathbb{R}^{L_{2} \cdot N \times \mathbb{C}},
\end{equation}
where $\mathcal{T} \in \mathbb{R}^{M \times \mathbb{C}}$ consists of $M$ learnable vectors, $\{\omega^{Q}, \omega^{K}, \omega^{V}\}$ transform the input features into query, key, and value vectors. 
Then, IDE learns intrinsic deviation vectors $\mathcal{T}^{*}$ by:
\begin{equation}
    \label{eq:eq6}
    \mathcal{T}^{*} = \{t_{m}\}_{1 \le m \le M} = 
    \mathtt{FFN} (\mathtt{Softmax}(\mathbf{Q} \mathbf{K}^{\top} / \sqrt{\mathbb{C}} + 
    \gamma (1 - \mathcal{M}^{a})) \mathbf{V}),
\end{equation}
where $\mathcal{M}^{a} \in \{0,1\}^{L_{2} \cdot N}$ is the patch-level binary anomaly masks aligned with $\mathcal{F}^{a}$, $\gamma \to -\infty$ is used to ensure that the attention weights of normal patches in $\mathcal{F}^{a}$ are close to zero, and $\mathtt{FFN}(\cdot)$ represents a feed forward network. 
$\mathcal{T}^{*} \in \mathbb{R}^{M \times \mathbb{C}}$ represents $M$ intrinsic deviation vectors of abnormality relative to normality. For brevity, we omit the residual-path operations in Eq.~\eqref{eq:eq6}.

On the one hand, we expect to narrow the distance between each individual element in $\mathcal{F}^{a}_{\mathtt{den}}$ and the corresponding nearest $t^{\min}$ in $\mathcal{T}^{*}$ to enhance anomaly-aware discriminative ability. On the other hand, we emphasize mutual orthogonality among these intrinsic deviation vectors $\mathcal{T}^{*}$ to reduce redundancy. 
Therefore, we encourage the learning of $\mathcal{T}^{*}$ by:
\begin{equation}
\label{eq:eq7}
{\small
    \mathcal{L}_{Dual} = \lambda_{1} \underbrace{\frac{1}{\sum_{i=1}^{L_{2} \cdot N} \mathcal{M}_{i}^{a}} \sum_{i=1}^{L_{2} \cdot N} \mathcal{M}_{i}^{a} \mathbf{d_{cos}}(f_{i,\mathtt{den}}^{a}, t_{i}^{\min})}_{\textnormal{Discriminability}} + 
    \lambda_{2} \underbrace{\frac{1}{M(M-1)} \sum_{m_{1} \neq m_{2}} (\frac{t_{m_{1}} \cdot t_{m_{2}}}{\|t_{m_{1}}\| \times \|t_{m_{2}}\|})^{2}}_{\textnormal{Orthogonality}},
}
\end{equation}
where $\mathcal{M}_{i}^{a} \in \{0,1\}$ is the corresponding mask value in $\mathcal{M}^{a}$, $t_{i}^{\min} := \mathtt{argmin}_{t_{m} \in \mathcal{T}^{*}} \mathbf{d_{cos}}(t_{m}, f_{i,\mathtt{den}}^{a})$ represents the nearest deviation pattern to the $f_{i,\mathtt{den}}^{a}$, $\lambda_{1}$ and $\lambda_{2}$ are hyper-parameters used to balance these loss terms. 
$\mathcal{L}_{Dual}$ promotes the learning of the intrinsic deviation vectors $\mathcal{T}^{*}$ that capture most discriminative orthogonal deviation directions from normality while maintaining mutual diversity, thereby improving the generalizable ability for detecting diverse abnormal patterns.

\subsection{Anomaly Scoring with Projection} \label{sec:score}
For anomaly scoring, we first extract the denoised deviations for each query $\mathbf{x}^{q}$: $\mathcal{F}^{q}_{\mathtt{den}}=\{f^{q}_{i,\mathtt{den}}\}_{i=1}^{N} \in \mathbb{R}^{N \times \mathbb{C}}$. Next, we project these denoised deviations $f^{q}_{i,\mathtt{den}}$ onto the learned intrinsic deviation vectors:
\begin{equation}
    \label{eq:eq8}
    \mathbf{\tilde{f}}^{q}_{i,\mathtt{den}} = \mathtt{Proj}(f^{q}_{i,\mathtt{den}}) := \sum_{m=1}^{M} 
    (\langle f^{q}_{i,\mathtt{den}}, t_m \rangle / \langle t_m, t_m \rangle) t_m,
\end{equation}
where $\langle \cdot, \cdot \rangle$ denotes the inner product operation. 
This projection retains only the relevant attributes of anomalous deviations, as irrelevant attributes are nearly orthogonal to the learned intrinsic deviation vectors. 
Finally, the anomaly score for each patch in the query $\mathbf{x}^{q}$ is treated as a complementary score of normal matching and deviation alignment, which can be calculated by:
\begin{equation}
    \label{eq:eq9}
    A_{i}^{q} = \frac{1}{2} (1 - \mathbf{d_{cos}} (f^{q}_{i,\mathtt{den}}, \mathbf{\tilde{f}}^{q}_{i,\mathtt{den}}) + \mathbf{d_{cos}} (f_{i}^{q}, f_{i,\min}^{n})) 
    \; \Rightarrow \; A^{q} = \{A_{i}^{q}\}_{i=1}^{N} \in \mathbb{R}^{N},
\end{equation}
where $A_{i}^{q}=0$ indicates the most normal region and $A_{i}^{q}=1$ indicates the most anomalous region, and $f_{i,\min}^{n}$ is the nearest normal reference feature to $f_{i}^{q}$. 
Following \cite{hwang2024anomaly}, we obtain the final image-level anomaly score $\mathcal{A}^{q}$ by averaging the top $1\%$ highest anomaly patch scores in the $A^{q}$. 

\subsection{Training and Inference}
\paragraph{Training on Auxiliary Datasets.}
IDEAL is trained on auxiliary data $\mathcal{D}_{tr}$ in an episodic manner. For each training episode $\mathbf{E}=\{\mathbf{x}^{q}, \mathcal{S}^{n}, \mathcal{S}^{a}\}$ from $\mathcal{D}_{tr}$, we optimize IDEAL with two objectives: 1) enforcing consistency between the predicted score map $A^{q}$ in Eq.~\eqref{eq:eq9} and the query's ground-truth mask $\mathcal{M}^{q}$; and 2) supervising the image-level anomaly score $\mathcal{A}^{q}$ with the query's label $\mathbf{y}^{q}$. 
Therefore, we employ Focal loss \cite{lin2017focal} and Dice loss \cite{milletari2016v} to optimize the segmentation branch, while using BCE (Binary Cross-Entropy) loss \cite{mannor2005cross} to supervise the classification branch, formulated as:
\begin{equation}
    \label{eq:eq10}
    \mathcal{L}_{train} = \mathcal{L}_{Focal}(A^{q},\mathcal{M}^{q})+\mathcal{L}_{Dice}(A^{q},\mathcal{M}^{q})+
    \mathcal{L}_{BCE}(\mathcal{A}^{q},\mathbf{y}^{q})+\mathcal{L}_{Dual}.
\end{equation}
\paragraph{Inference.}
During inference, for any new dataset $c$, we assume a fixed reference set $\mathcal{S}_c=\{\mathcal{S}^{n}_c, \mathcal{S}^{a}_c\}$, where $\mathcal{S}^{a}_c$ is drawn from one randomly selected anomaly type and $\mathcal{S}^{n}_c$ consists of a few normal images randomly sampled from the normal class. The resulting reference set $\mathcal{S}_c$ is then fixed and used in the anomaly scoring of all test samples. 
Given a test query $x$, IDEAL yields a deviation-guided anomaly score map $A^{x}$ based on the input $\{x, \mathcal{S}^{n}_c,\mathcal{S}^{a}_c\}$. Finally, $A^{x}$ is upsampled to the original image resolution $\mathbb{R}^{H \times W}$ for anomaly localization.

\subsection{Generalization Bound}
We provide insights into the source-target generalization analysis of IDEAL. A detailed description and derivations of this analysis can be found in Appendix~\ref{sec:bound}.
\begin{theorem}[Generalization Bound of IDEAL] \label{main_bound}
    Consider there are $K$ labeled source domains. Let $\mathbb{P}_{S_1},\ldots,\mathbb{P}_{S_K}$ be the corresponding source episode distributions, and $\mathbb{P}_{T}$ be the target episode distribution. Denote the model $\Phi$ as the hypothesis from $\mathcal{H}$ and $d$ be the VC-dimension of $\mathcal{H}$. The total number of episodes over all $K$ source domains is $n$. 
    Then, for any source weighting vector $\beta=(\beta_1,\ldots,\beta_K) \in \Delta_K$, with probability at least $1-\delta$, the following holds for any $\Phi \in \mathcal{H}$,
    \begin{equation}
        \label{eq:eq11}
        \epsilon_{T}(\Phi) \le \hat{\epsilon}_{\beta}(\Phi) + 
        \sqrt{\frac{1}{2} \left( \sum_{k=1}^{K} \frac{\beta_k^2}{n_k} \right) \log \frac{M}{\delta}} + 
        2\sum_{k=1}^{K} \beta_k \sqrt{\frac{d}{n_k} \log \frac{e n_k}{d}} + 
        \mathrm{D}_{\mathcal{H}, \ell}(\mathbb{P}_{T}, \mathbb{P}_{\beta}) + \Omega_{\beta}.
    \end{equation}
\end{theorem}
Theorem~\ref{main_bound} indicates that the generalization bound of IDEAL is associated with the empirical source risk $\hat{\epsilon}_{\beta}(\Phi)$, the number $M$ of intrinsic deviation vectors, the source-target data discrepancy $\mathrm{D}_{\mathcal{H},\ell}$, and the joint approximation term $\Omega_{\beta}$. 
Therefore, a tighter bound can be expected when $\Phi$ attains low empirical source risk and the number $M$ of intrinsic deviation vectors is properly chosen.

\begin{table*}[t!]
\caption{\textbf{Comparison results (\textit{General}) with AUROC metric} under various few-shot AD settings, where ($\cdot / \cdot$) means image-level and pixel-level AUROCs, \texttt{N}$i$ and \texttt{A}$i$ denote the $i$-shot of normal and abnormal references. 
Best results are in \textcolor{DeepRed}{\textbf{red bold}}, with second best results in \textcolor{blue}{\textbf{blue bold}}.}
\label{tab:tab1}
\vspace{2mm}
\centering
\begingroup
\setlength{\tabcolsep}{2.5pt} 
\renewcommand{\arraystretch}{1.1} 
\setlength{\arrayrulewidth}{0.1mm} 
\resizebox{\linewidth}{!}{%
    \begin{tabular}{cl I I c I c I c I c I c I c}
    \hline\thickhline
    \rowcolor{mygray} 
    \textsc{\textbf{Setups}} & \textsc{\textbf{Methods}} 
    & \textsc{\textbf{MVTecAD}} & \textsc{\textbf{VisA}} & \textsc{\textbf{AITEX}} & \textsc{\textbf{MPDD}} 
    & \textsc{\textbf{BTAD}} & \textsc{\textbf{BraTS}} \\
    \hline\hline
    & RegAD~\cite{huang2022registration} 
        & (91.6 / 93.1) & (83.3 / 91.6) & (72.0 / 80.2) & (69.7 / 92.5) & (86.6 / 94.5) & (64.2 / 92.2)\\
    & PromptAD~\cite{li2024promptad} 
        & (92.8 / 93.4) & (83.8 / 89.5) & (71.9 / \textcolor{blue}{\textbf{81.4}}) & (75.8 / 95.1) & (89.6 / 92.4) & (68.7 / 91.9)\\
    & WinCLIP~\cite{jeong2023winclip} 
        & (93.1 / 93.4) & (78.7 / 94.7) & (73.1 / 78.9) & (70.2 / 94.3) & (84.5 / 94.3) & (68.8 / 91.0)\\
    & InCTRL~\cite{zhu2024toward} 
        & (92.5 / 93.4) & (82.2 / 93.1) & (\textcolor{blue}{\textbf{74.9}} / 78.3) & (74.8 / 93.1) & (90.5 / 92.9) & (69.5 / 93.1)\\
    & ResAD~\cite{yao2024resad} 
        & (80.7 / 84.3) & (77.2 / 88.3) & (70.5 / 72.7) & (67.7 / 92.7) & (78.4 / 92.3) & (59.7 / 89.5)\\
    \multirow{-6}{*}{\makecell[c]{\texttt{N1}}} & FoundAD~\cite{zhai2026foundation} 
        & (89.9 / 90.7) & (85.3 / 96.2) & (71.4 / 80.4) & (70.6 / 95.4) & (91.4 / 91.9) & (62.7 / 87.8)\\
    \hline
    & DRA~\cite{ding2022catching} 
        & (92.2 / 92.9) & (82.3 / 90.3) & (72.3 / 75.3) & (73.4 / 93.9) & (90.5 / 93.5) & (67.2 / 87.5)\\
    \multirow{-2}{*}{\makecell[c]{\texttt{A1}}} & AHL~\cite{zhu2024anomaly} 
        & (91.7 / 93.6) & (82.0 / 91.9) & (70.7 / 76.3) & (73.3 / 94.3) & (90.9 / 93.8) 
        & (\textcolor{blue}{\textbf{70.4}} / 90.4)\\
    \hline
    & NAGL~\cite{wang2025normal} 
        & (\textcolor{blue}{\textbf{95.1}} / \textcolor{blue}{\textbf{96.1}}) 
        & (\textcolor{blue}{\textbf{88.5}} / \textcolor{blue}{\textbf{97.5}}) 
        & (71.6 / 75.5) & (\textcolor{blue}{\textbf{77.1}} / \textcolor{blue}{\textbf{96.5}}) 
        & (\textcolor{blue}{\textbf{92.0}} / \textcolor{blue}{\textbf{95.8}}) & (67.7 / \textcolor{blue}{\textbf{93.7}})\\
    \multirow{-2}{*}{\makecell[c]{\texttt{N1A1}}} & \cellcolor{lightyellow}\textbf{IDEAL} (ours) 
        & (\textcolor{DeepRed}{\textbf{96.3}} / \textcolor{DeepRed}{\textbf{96.8}}) 
        & (\textcolor{DeepRed}{\textbf{90.8}} / \textcolor{DeepRed}{\textbf{97.8}}) 
        & (\textcolor{DeepRed}{\textbf{75.8}} / \textcolor{DeepRed}{\textbf{82.7}}) 
        & (\textcolor{DeepRed}{\textbf{78.5}} / \textcolor{DeepRed}{\textbf{97.9}}) 
        & (\textcolor{DeepRed}{\textbf{93.4}} / \textcolor{DeepRed}{\textbf{97.0}}) 
        & (\textcolor{DeepRed}{\textbf{74.9}} / \textcolor{DeepRed}{\textbf{94.8}})\\
    \hline\hline
    & RegAD~\cite{huang2022registration} 
        & (92.6 / 93.5) & (87.4 / 92.7) & (73.5 / \textcolor{blue}{\textbf{82.0}}) & (71.9 / 94.5) & (90.8 / 95.6) & (70.7 / 94.1)\\
    & PromptAD~\cite{li2024promptad} 
        & (93.5 / 93.8) & (86.6 / 89.8) & (73.8 / 81.8) & (76.8 / 96.1) & (90.3 / 92.6) & (74.4 / 93.7)\\
    & WinCLIP~\cite{jeong2023winclip} 
        & (93.4 / 93.6) & (84.3 / 94.8) & (73.5 / 79.3) & (70.8 / 94.7) & (85.0 / 95.5) & (74.8 / 92.7)\\
    & InCTRL~\cite{zhu2024toward} 
        & (92.7 / 93.7) & (85.6 / 93.7) & (\textcolor{blue}{\textbf{75.2}} / 78.7) & (75.9 / 93.5) & (93.4 / 93.0) & (\textcolor{blue}{\textbf{75.0}} / 94.1)\\
    & ResAD~\cite{yao2024resad} 
        & (85.7 / 86.3) & (82.4 / 89.8) & (75.0 / 75.1) & (68.4 / 94.4) & (80.1 / 92.9) & (65.4 / 92.2)\\
    \multirow{-6}{*}{\makecell[c]{\texttt{N2}}} & FoundAD~\cite{zhai2026foundation} 
        & (90.3 / 92.8) & (88.5 / 96.5) & (72.3 / 81.4) & (73.9 / 95.7) 
        & (\textcolor{blue}{\textbf{94.0}} / 93.2) & (72.5 / 91.9)\\
    \hline
    & NAGL~\cite{wang2025normal} 
        & (\textcolor{blue}{\textbf{96.0}} / \textcolor{blue}{\textbf{96.6}}) 
        & (\textcolor{blue}{\textbf{89.1}} / \textcolor{blue}{\textbf{97.5}}) 
        & (73.5 / 79.8) & (\textcolor{blue}{\textbf{80.2}} / \textcolor{blue}{\textbf{97.3}}) 
        & (93.5 / \textcolor{blue}{\textbf{96.3}}) & (74.6 / \textcolor{blue}{\textbf{94.3}})\\
    \multirow{-2}{*}{\makecell[c]{\texttt{N2A1}}} & \cellcolor{lightyellow}\textbf{IDEAL} (ours) 
        & (\textcolor{DeepRed}{\textbf{97.4}} / \textcolor{DeepRed}{\textbf{97.5}}) 
        & (\textcolor{DeepRed}{\textbf{91.3}} / \textcolor{DeepRed}{\textbf{97.9}}) 
        & (\textcolor{DeepRed}{\textbf{77.6}} / \textcolor{DeepRed}{\textbf{85.3}}) 
        & (\textcolor{DeepRed}{\textbf{81.6}} / \textcolor{DeepRed}{\textbf{98.0}}) 
        & (\textcolor{DeepRed}{\textbf{95.6}} / \textcolor{DeepRed}{\textbf{97.2}}) 
        & (\textcolor{DeepRed}{\textbf{78.5}} / \textcolor{DeepRed}{\textbf{95.6}})\\
    \hline\hline
    & RegAD~\cite{huang2022registration} 
        & (94.0 / 95.1) & (89.5 / 93.3) & (74.6 / \textcolor{blue}{\textbf{83.7}}) & (76.3 / 95.2) & (91.7 / 96.3) & (75.8 / \textcolor{blue}{\textbf{94.5}})\\
    & PromptAD~\cite{li2024promptad} 
        & (94.9 / 94.6) & (87.1 / 90.6) & (74.1 / 83.4) & (78.3 / 96.5) & (90.8 / 93.2) & (74.5 / 94.0)\\
    & WinCLIP~\cite{jeong2023winclip} 
        & (93.9 / 93.8) & (84.5 / 95.1) & (73.8 / 79.5) & (71.0 / 95.4) & (87.4 / 95.9) & (75.7 / 93.2)\\
    & InCTRL~\cite{zhu2024toward} 
        & (94.3 / 94.0) & (89.5 / 94.2) & (\textcolor{blue}{\textbf{75.6}} / 80.4) & (79.4 / 93.7) & (93.9 / 93.6) & (\textcolor{blue}{\textbf{77.8}} / 93.8)\\
    & ResAD~\cite{yao2024resad} 
        & (87.1 / 86.9) & (83.1 / 90.1) & (75.2 / 76.7) & (70.3 / 95.7) & (81.5 / 95.1) & (69.6 / 93.0)\\
    \multirow{-6}{*}{\makecell[c]{\texttt{N4}}} & FoundAD~\cite{zhai2026foundation} 
        & (92.9 / 93.2) & (91.0 / 96.7) & (73.2 / 82.8) & (77.6 / 95.8) & (\textcolor{blue}{\textbf{94.3}} / 94.0) & (73.5 / 92.8)\\
    \hline
    & NAGL~\cite{wang2025normal} 
        & (\textcolor{blue}{\textbf{96.9}} / \textcolor{blue}{\textbf{97.0}}) 
        & (\textcolor{blue}{\textbf{91.2}} / \textcolor{blue}{\textbf{97.7}}) 
        & (\textcolor{blue}{\textbf{75.6}} / 80.9) 
        & (\textcolor{blue}{\textbf{81.4}} / \textcolor{blue}{\textbf{97.7}}) 
        & (93.7 / \textcolor{blue}{\textbf{96.5}}) 
        & (74.9 / 94.4)\\
    \multirow{-2}{*}{\makecell[c]{\texttt{N4A1}}} & \cellcolor{lightyellow}\textbf{IDEAL} (ours) 
        & (\textcolor{DeepRed}{\textbf{98.2}} / \textcolor{DeepRed}{\textbf{97.5}}) 
        & (\textcolor{DeepRed}{\textbf{92.7}} / \textcolor{DeepRed}{\textbf{98.0}}) 
        & (\textcolor{DeepRed}{\textbf{79.5}} / \textcolor{DeepRed}{\textbf{86.1}}) 
        & (\textcolor{DeepRed}{\textbf{87.6}} / \textcolor{DeepRed}{\textbf{98.5}}) 
        & (\textcolor{DeepRed}{\textbf{95.9}} / \textcolor{DeepRed}{\textbf{97.5}}) 
        & (\textcolor{DeepRed}{\textbf{80.5}} / \textcolor{DeepRed}{\textbf{95.8}})\\
    \hline\hline
    & RegAD~\cite{huang2022registration} 
        & (95.0 / 95.8) & (\textcolor{blue}{\textbf{91.9}} / 94.7) & (75.8 / 83.9) & (80.8 / 95.7) & (92.2 / \textcolor{blue}{\textbf{96.5}}) & (76.3 / 95.0)\\
    & PromptAD~\cite{li2024promptad} 
        & (96.0 / 94.8) & (87.9 / 92.2) & (74.4 / \textcolor{blue}{\textbf{84.7}}) & (80.6 / 96.5) & (91.2 / 94.1) & (75.6 / 94.6)\\
    & WinCLIP~\cite{jeong2023winclip} 
        & (95.5 / 94.2) & (86.1 / 95.1) & (74.8 / 80.5) & (74.7 / 95.6) & (89.6 / 96.0) & (76.9 / 93.8)\\
    & InCTRL~\cite{zhu2024toward} 
        & (95.8 / 94.4) & (89.8 / 95.3) & (75.9 / 80.6) & (\textcolor{blue}{\textbf{83.2}} / 94.1) & (94.3 / 94.9) & (\textcolor{blue}{\textbf{79.5}} / 94.9)\\
    & ResAD~\cite{yao2024resad} 
        & (89.4 / 87.8) & (83.7 / 90.9) & (75.4 / 79.3) & (75.1 / 96.1) & (83.6 / 95.5) & (71.7 / 94.3)\\
    \multirow{-6}{*}{\makecell[c]{\texttt{N8}}} & FoundAD~\cite{zhai2026foundation} 
        & (93.2 / 93.5) & (91.3 / 96.8) & (74.3 / 83.0) & (80.3 / 96.2) & (\textcolor{blue}{\textbf{94.6}} / 95.1) & (76.7 / 93.5)\\
    \hline
    & DRA~\cite{ding2022catching} 
        & (93.7 / 94.0) & (86.8 / 90.7) & (75.7 / 77.8) & (81.8 / 94.9) & (93.0 / 93.9) & (73.5 / 89.9)\\
    \multirow{-2}{*}{\makecell[c]{\texttt{A4}}} & AHL~\cite{zhu2024anomaly} 
        & (92.9 / 94.5) & (85.9 / 92.5) & (75.1 / 78.8) & (81.0 / 95.2) & (93.3 / 94.5) & (72.9 / 90.3)\\
    \hline
    & NAGL~\cite{wang2025normal} 
        & (\textcolor{blue}{\textbf{97.2}} / \textcolor{blue}{\textbf{97.1}}) 
        & (91.5 / \textcolor{blue}{\textbf{97.9}}) 
        & (\textcolor{blue}{\textbf{80.5}} / 82.3) 
        & (83.1 / \textcolor{blue}{\textbf{97.9}}) 
        & (94.0 / \textcolor{blue}{\textbf{96.5}}) & (75.1 / \textcolor{blue}{\textbf{95.2}})\\
    \multirow{-2}{*}{\makecell[c]{\texttt{N8A4}}} & \cellcolor{lightyellow}\textbf{IDEAL} (ours) 
        & (\textcolor{DeepRed}{\textbf{98.6}} / \textcolor{DeepRed}{\textbf{97.8}}) 
        & (\textcolor{DeepRed}{\textbf{93.1}} / \textcolor{DeepRed}{\textbf{98.3}}) 
        & (\textcolor{DeepRed}{\textbf{83.9}} / \textcolor{DeepRed}{\textbf{86.8}}) 
        & (\textcolor{DeepRed}{\textbf{88.0}} / \textcolor{DeepRed}{\textbf{98.7}}) 
        & (\textcolor{DeepRed}{\textbf{96.2}} / \textcolor{DeepRed}{\textbf{97.7}}) 
        & (\textcolor{DeepRed}{\textbf{82.8}} / \textcolor{DeepRed}{\textbf{96.9}})\\
    \hline\thickhline
    \end{tabular}
}
\endgroup
\end{table*}

\section{Experiments} \label{sec:myexp}
\subsection{Experimental Setup} \label{sec:exp_setups}
\paragraph{Datasets.}
To verify the effectiveness of IDEAL, we conduct experiments across eight AD datasets, including five industrial datasets: MVTecAD \cite{bergmann2019mvtec}, VisA \cite{zou2022spot}, AITEX \cite{silvestre2019public}, MPDD \cite{jezek2021deep}, BTAD \cite{mishra2021vt}; and three medical image datasets: BraTS \cite{menze2014multimodal} (for brain tumor segmentation), Liver \cite{bao2024bmad} (for liver tumor segmentation), RESC \cite{hu2019automated} (for retinal edema segmentation). To assess the AD performance, we adopt a cross-dataset evaluation: the model is trained on MVTecAD as the source dataset and then directly evaluated on the test sets of the other seven datasets without any further training, following prior work \cite{zhu2024toward,yao2024resad}. When evaluating on MVTecAD, the model is trained on VisA. 
During inference, unlike \cite{wang2025normal}, which samples different abnormal references based on the query's anomaly type, we sample them once from each test dataset and keep them fixed for all test samples in the same dataset.
\paragraph{Competing Methods and Metrics.}
We compare IDEAL with two groups of closely related methods: (1)
\textit{Generalist Methods}, \ie, detectors are training-free or trained once on auxiliary data: based on normal-only references (InCTRL \cite{zhu2024toward}, ResAD \cite{yao2024resad}, WinCLIP \cite{jeong2023winclip}, FoundAD \cite{zhai2026foundation}), and normal and anomalous references (NAGL \cite{wang2025normal}); (2) \textit{Specalist Methods}, 
 \ie, detectors require dataset-specific training: based on normal-only references (RegAD \cite{huang2022registration}, PromptAD \cite{li2024promptad}), and based on full-shot normal data and abnormal-only references: DRA \cite{ding2022catching} and AHL \cite{zhu2024anomaly}.
For metrics, following \cite{zhu2024toward,zhou2023anomalyclip}, we use Area Under the Receiver Operating Characteristic (AUROC), maximum F1-score at the optimal threshold (F1-Max), and Per-Region Overlap (PRO) to evaluate pixel-level AD; similarly, we use AUROC, F1-Max, and Average Precision (AP) to evaluate image-level AD.
\paragraph{Implementation Details.}
By default, we adopt a pre-trained ViT-S/14 \cite{dosovitskiy2021an} backbone as the visual encoder $\mathcal{E}(\cdot)$ in our experiments. We freeze the parameters of the visual encoder and only update the parameters of the IDE component. All the training and test images are resized to a $448 \times 448$ resolution. AdamW \cite{loshchilov2018decoupled} is used as the optimizer, and the initial learning rate is set to $0.001$ with a cosine scheduler. The number of training epochs of IDEAL is set to $20$ with a batch size of $16$ on a single NVIDIA GeForce RTX 3090 GPU. We set the number $k$ of normal neighbors in Eq.~\eqref{eq:eq2} as $12$, the number $r$ of retained principal components as $4$, $\alpha$ in Eq.~\eqref{eq:eq4} as $0.8$, $M$ in Eq.~\eqref{eq:eq6} as $45$, $\lambda_1$ and $\lambda_2$ in Eq.~\eqref{eq:eq7} as $1.0$ and $0.8$ by default. 
Following previous works \cite{zhu2024toward,zhu2024anomaly}, we set the number of normal references to $L_1 \in [1,2,4,8]$ and the number of abnormal references to $L_2 \in [1,4]$, with $L_2=1$ as the default to simulate realistic scenarios. 
For AD methods requiring abnormal references, we construct fixed abnormal references by sampling from a randomly selected anomaly type, and consider two setups: \textit{General Setting}, which utilizes the full test set that enables evaluation of detecting both seen and unseen anomalies, and \textit{Hard Setting}, which excludes test samples belonging to the selected anomaly type, enabling exclusive evaluation of detecting unseen anomalies. All results reported are the average of three independent runs. More details in Appendix \ref{sec:edetails}.

\subsection{Main Results}
\paragraph{Benefits of Abnormal References.}
Tab.~\ref{tab:tab1} presents comparison results of image-level and pixel-level AUROCs under various few-shot AD settings. Note that all the reported results are dataset-level average results across their respective subsets. Overall, IDEAL achieves superior performance by incorporating both normal and abnormal references. With just one additional abnormal reference \texttt{N1A1}, IDEAL yields $96.3\%$ and $96.8\%$ for image and pixel AUROCs on MVTecAD, outperforming state-of-the-art FSAD methods that use normal-only references. 
The similar improvements are also observed on other AD datasets. 
These results show the advantage of introducing abnormal examples into the reference set, as they provide discriminative clues about anomaly characteristics that are difficult to obtain from normal-only references. Please see complete results in Appendix~\ref{sec:allres}.

\paragraph{General vs. Hard Settings.}
\begin{wraptable}{r}{0.6\linewidth}
\vspace{-6.6mm}
\caption{\textbf{Comparison results (\textit{Hard}) with AUROC} on MVTecAD and MPDD, together with task difficulty values (\textcolor{purple}{purple}) that measures a seen-to-unseen anomaly similarity.}
\label{tab:tab2}
\vspace{1.5mm}
\centering
\begingroup
\setlength{\tabcolsep}{1pt} 
\renewcommand{\arraystretch}{1.06} 
\setlength{\arrayrulewidth}{0.1mm} 
\resizebox{\linewidth}{!}{%
    \begin{tabular}{cl I I cc I cc}
    \hline\thickhline
    \rowcolor{mygray} 
    \textsc{\textbf{Setups}} & \textsc{\textbf{Methods}} 
    & \multicolumn{2}{cI}{\textsc{\textbf{MVTecAD}}} & \multicolumn{2}{c}{\textsc{\textbf{MPDD}}} \\
    \hline\hline
    & NAGL~\cite{wang2025normal} 
        & (93.2 \mydown{1.9} / 94.0 \mydown{2.1}) & & (63.7 \mydown{13.4} / 91.8 \mydown{4.7}) & \\
    \multirow{-2}{*}{\makecell[c]{\texttt{N1A1}}} & \cellcolor{lightyellow}\textbf{IDEAL} 
        & (95.8 \mydown{0.5} / 96.4 \mydown{0.4}) & \multirow{-2}{*}{\makecell[c]{\textcolor{purple}{0.824}}} & (74.5 \mydown{4.0} / 97.2 \mydown{0.7}) & \multirow{-2}{*}{\makecell[c]{\textcolor{purple}{0.543}}} \\
    \hline
    & NAGL~\cite{wang2025normal} 
        & (93.5 \mydown{2.5} / 94.8 \mydown{1.8}) & & (71.7 \mydown{8.5} / 93.0 \mydown{4.3}) & \\
    \multirow{-2}{*}{\makecell[c]{\texttt{N2A1}}} & \cellcolor{lightyellow}\textbf{IDEAL} 
        & (97.0 \mydown{0.4} / 96.6 \mydown{0.9}) & \multirow{-2}{*}{\makecell[c]{\textcolor{purple}{0.833}}} & (77.0 \mydown{4.6} / 97.4 \mydown{0.6}) & \multirow{-2}{*}{\makecell[c]{\textcolor{purple}{0.515}}}\\
    \hline
    & NAGL~\cite{wang2025normal} 
        & (94.6 \mydown{2.3} / 95.0 \mydown{2.0}) & & (75.8 \mydown{5.6} / 93.8 \mydown{3.9}) & \\
    \multirow{-2}{*}{\makecell[c]{\texttt{N4A1}}} & \cellcolor{lightyellow}\textbf{IDEAL} 
        & (97.3 \mydown{0.9} / 97.0 \mydown{0.5}) & \multirow{-2}{*}{\makecell[c]{\textcolor{purple}{0.839}}} & (84.3 \mydown{3.3} / 97.5 \mydown{1.0}) & \multirow{-2}{*}{\makecell[c]{\textcolor{purple}{0.501}}}\\
    \hline
    & NAGL~\cite{wang2025normal} 
        & (95.1 \mydown{2.1} / 95.4 \mydown{1.7}) & & (76.6 \mydown{6.5} / 94.2 \mydown{3.7}) & \\
    \multirow{-2}{*}{\makecell[c]{\texttt{N8A4}}} & \cellcolor{lightyellow}\textbf{IDEAL} 
        & (97.6 \mydown{1.0} / 97.2 \mydown{0.6}) & \multirow{-2}{*}{\makecell[c]{\textcolor{purple}{0.835}}} & (84.9 \mydown{3.1} / 97.8 \mydown{0.9}) & \multirow{-2}{*}{\makecell[c]{\textcolor{purple}{0.513}}} \\
    \hline\thickhline
    \end{tabular}
}
\endgroup
\end{wraptable}
Tab.~\ref{tab:tab2} reports comparison results of IDEAL and NAGL \cite{wang2025normal} under the hard setting, where test samples belonging to the anomaly type of abnormal references are excluded from the test set. In contrast to the general setting (Tab.~\ref{tab:tab1}), this hard setting evaluates the capability of detecting unseen anomalies. 
To quantify the hardness/task difficulty of unseen AD, we compute the mask-aware cosine similarity between abnormal references and anomalous test samples, where lower similarity indicates a larger seen-to-unseen anomaly gap (\textcolor{purple}{task difficulty values} in Tab.~\ref{tab:tab2}). 
In such a hard setting, NAGL \cite{wang2025normal} exhibits a pronounced performance degradation, as its anomaly scoring relies on the alignment between test query and abnormal references. 
IDEAL maintains robust performance by leveraging normal-abnormal references to learn intrinsic deviation patterns, enabling generalization to unseen anomalies. 

\subsection{Validation Studies} \label{sec:valida}

\begin{wrapfigure}{r}{0.4\linewidth}
\vspace{-4.6mm}
    \centering
    \setlength{\abovecaptionskip}{1mm} 
    \includegraphics[width=1.0\linewidth]{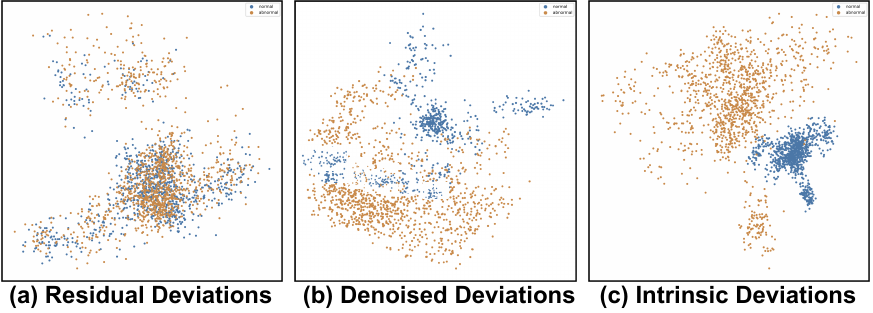}
    \caption{T-SNE visualization for different deviation representations produced by IDEAL on VisA (\textcolor{norblue}{normal} \textcolor{aboorange}{abnormal}).}
\label{fig:fig4}
\vspace{-2mm}
\end{wrapfigure}
\paragraph{Deviation Analysis.}
Fig.~\ref{fig:fig4} presents the t-SNE \cite{van2008visualizing} visualization of different deviation representations produced by IDEAL on VisA (trained on MVTecAD). From these results, we can observe: 
1) residual deviation representations are highly entangled between normal and abnormal samples, 
while denoised deviation representations suppress noisy normal variations contained in these residual deviations; 2) after projection onto the learned intrinsic deviation vectors, normal and abnormal samples become clearly separated. 
These results indicate that IDEAL can effectively extract intrinsic deviations from normality, thereby improving normal-abnormal separability, which in turn enhances detection.

\paragraph{Ablation Studies.}
We provide an ablation analysis of IDEAL's components in Tab.~\ref{tab:tab3} (\textbf{Left}). 
Comparing first two rows, we can observe that naively matching the normal/abnormal reference to perform anomaly scoring is inherently fragile, as such a strategy fails to cover diverse anomaly types. 
The third row indicates that our NVE component effectively reduces noisy activations caused by

\begin{wrapfigure}{r}{0.36\linewidth}
    \centering
    \includegraphics[width=\linewidth]{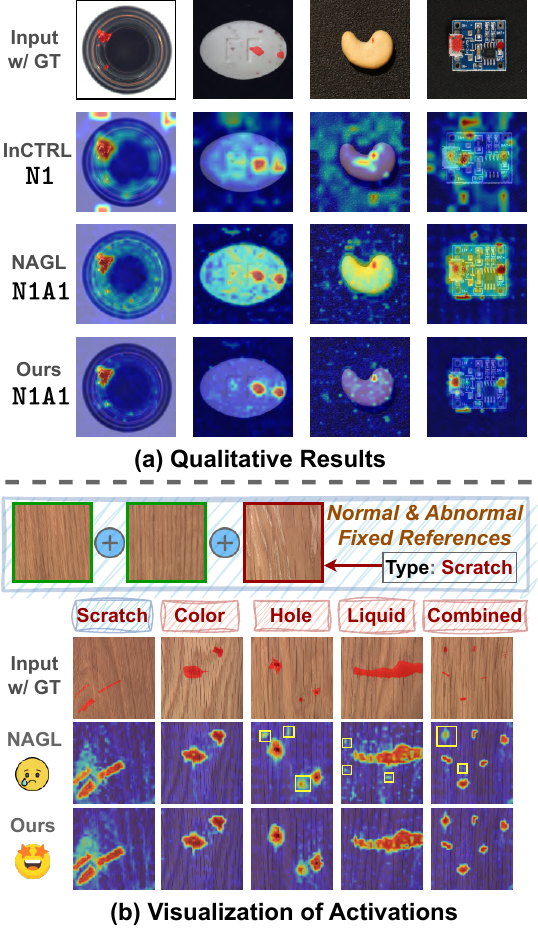}
    \caption{\textbf{(a)} Qualitative results of three generalist FSAD methods. \textbf{(b)} Visualization of anomalous activations on a new dataset (VisA$\to$MVTecAD).}
\label{fig:fig5}
\vspace{-6mm}
\end{wrapfigure}
nuisance normal variations. 
The fourth row verifies that IDE is effective even when applied to noisy residual deviations, showing that learning intrinsic deviation vectors itself provides strong anomaly discriminative ability. 
By further combining NVE and IDE (the fifth row), IDEAL learns intrinsic deviation vectors from denoised deviations, enabling generalization to both seen and unseen anomalies. 
Appendix~\ref{sec:nveplug} further shows that NVE can be easily incorporated with existing residual-based FSAD methods to enhance detection. In Appendix~\ref{sec:hyper} and Appendix~\ref{sec:baloss}, we provide complete ablation studies of IDEAL's key arguments and analyze the effect of loss terms in Eq.~\eqref{eq:eq10}.

\begin{table*}[t!]
\caption{\textbf{Left}: Component ablation of IDEAL. The first two rows use nearest neighbor matching to perform anomaly scoring, the third row adds the similarity score between query-to-normal deviations and denoised deviations, the fourth row uses the residual deviations to learn deviation patterns, and the fifth row is our full IDEAL model. \textbf{Right}: Efficiency comparisons on VisA (trained on MVTecAD). Note that the time of WinCLIP represents the wall-clock time of the complete pipeline.}
\label{tab:tab3}
\vspace{2mm}
\centering
\begin{minipage}{0.493\linewidth}
\centering
\begingroup
\setlength{\tabcolsep}{1.25pt} 
\renewcommand{\arraystretch}{1.2} 
\setlength{\arrayrulewidth}{0.1mm} 
\resizebox{\linewidth}{!}{%
    \begin{tabular}{cccc I ccc}
    \hline\thickhline
    \rowcolor{mygray} \texttt{N1} & \texttt{A1} & NVE & IDE & \textsc{\textbf{MVTecAD}} & \textsc{\textbf{VisA}} & \textsc{\textbf{BraTS}} \\
    \hline\hline
    & \ding{51} & & & (72.4 / 88.0) & (61.3 / 85.6) & (39.7 / 83.3) \\
    \ding{51} & \ding{51} & & & (90.1 / 92.2) & (77.8 / 92.7) & (57.6 / 87.5) \\\hline
    \ding{51} & \ding{51} & \ding{51} & & (93.9 / 94.3) & (82.7 / 95.3) & (63.6 / 92.0) \\
    \ding{51} & \ding{51} & & \ding{51} & (\textcolor{blue}{\textbf{94.5}} / \textcolor{blue}{\textbf{96.0}}) 
    & (\textcolor{blue}{\textbf{87.3}} / \textcolor{blue}{\textbf{96.5}}) 
    & (\textcolor{blue}{\textbf{70.1}} / \textcolor{blue}{\textbf{92.9}}) \\
    \ding{51} & \ding{51} & \ding{51} & \ding{51} 
    & (\textcolor{DeepRed}{\textbf{96.3}} / \textcolor{DeepRed}{\textbf{96.8}}) 
    & (\textcolor{DeepRed}{\textbf{90.8}} / \textcolor{DeepRed}{\textbf{97.8}}) 
    & (\textcolor{DeepRed}{\textbf{74.9}} / \textcolor{DeepRed}{\textbf{94.8}}) \\
    \hline\thickhline
    \end{tabular}
}
\endgroup
\end{minipage}
\hfill
\begin{minipage}{0.493\linewidth}
\centering
\begingroup
\setlength{\tabcolsep}{2.8pt} 
\renewcommand{\arraystretch}{1.22} 
\setlength{\arrayrulewidth}{0.1mm} 
\resizebox{\linewidth}{!}{%
    \begin{tabular}{l I cccc}
    \hline\thickhline
    \rowcolor{mygray} \textsc{\textbf{Methods}} & Param. (M) & Time & FPS & $\textnormal{FLOPs}_{\times 10^{10}}$ \\
    \hline\hline
    WinCLIP~\cite{jeong2023winclip} & 117.3 & \textcolor{DeepRed}{\textbf{0.2 H}} & 1.8 & 41.52 \\
    InCTRL~\cite{zhu2024toward} & 117.5 & 0.9 H & 1.5 & 20.47 \\
    ResAD~\cite{yao2024resad} & 59.31 & 22.8 H & 7.6 & 15.29 \\
    NAGL~\cite{wang2025normal} & \textcolor{blue}{\textbf{24.45}} & 0.6 H & \textcolor{blue}{\textbf{16.8}} & \textcolor{blue}{\textbf{13.98}} \\
    \cellcolor{lightyellow}\textbf{IDEAL} (ours) & \textcolor{DeepRed}{\textbf{23.83}} & \textcolor{blue}{\textbf{0.5 H}} & \textcolor{DeepRed}{\textbf{21.9}} & \textcolor{DeepRed}{\textbf{12.70}} \\
    \hline\thickhline
    \end{tabular}
}
\endgroup
\end{minipage}
\end{table*}

\paragraph{Efficiency Comparison.}
We evaluate the efficiency from four aspects: model parameters, wall-clock time, inference speed (FPS), computation cost (FLOPs). 
As shown in Tab.~\ref{tab:tab3} (\textbf{Right}), IDEAL has $23.83$ million parameters, which is around $5\times$ smaller than InCTRL and WinCLIP. Compared with ResAD, IDEAL requires only $0.5$ hours for training, achieving a $45\times$ speedup. 
In terms of inference speed, IDEAL reaches $21.9$ FPS, which is $14\times$ faster than InCTRL and $1.3\times$ faster than NAGL. These results demonstrate the practical efficiency of IDEAL for real-world deployment.

\paragraph{Qualitative Results.}
Fig.~\ref{fig:fig5} (a) shows qualitative results of three generalist FSAD methods. 
Comparing the anomaly score maps, we observe that IDEAL effectively avoids false positives in normal regions and locates anomalies more accurately. 
Then, Fig.~\ref{fig:fig5} (b) further illustrates the generalizability of IDEAL under limited reference coverage. With a fixed reference set, NAGL is sensitive to the coverage of the given anomalous exemplars and tends to fail when the query exhibits anomaly appearances that are unseen in, or poorly covered by, the references. In contrast, IDEAL learns intrinsic deviation patterns from normal and abnormal references, enabling more robust detection beyond limited references.

\paragraph{Number of Abnormal References.}
To investigate the demand for abnormal examples, we conduct an ablation study by varying $L_{2} \in \{0,1,2,4\}$ while fixing the number $L_1$ of normal examples. 
Tab.~\ref{tab:tab_L2} shows the results of homogeneous (\ie, \texttt{homo}, all anomaly references are sampled from one anomaly type) and heterogeneous anomaly references (\ie, \texttt{hete}, the references are sampled from as many anomaly types as possible), 
we can observe that: 1) adding one abnormal reference already provides substantial gains, and 2) increasing $L_2$ from $1$ to $4$ yields consistent improvement, but the gains are mostly marginal. The largest gain still comes from adding one abnormal reference.

\paragraph{Effect of Reference Anomaly Types.}
We conduct an iterative cross-validation of the performance on the anomaly type, \ie, the abnormal reference is drawn from one anomaly type, and the evaluation is done on all remaining anomaly types, while keeping the normal references unchanged. For a dataset that has $K$ anomaly types, we perform the cross-validation $K$ times across all anomaly types and report the mean$\pm$std results of these $K$ runs on each anomaly type per dataset. 
Tab.~\ref{tab:tab_type} shows that the proposed IDEAL is notably less sensitive to the reference anomaly type than NAGL.

\begin{table*}[t!]
\caption{\textbf{Comparison results (\textit{General}) with different $L_2$}, where \texttt{homo} and \texttt{hete} represent homogeneous and heterogeneous anomaly references. Please see Tab.~\ref{tab:tab1} for \texttt{N1A1} and \texttt{N4A1}.}
\label{tab:tab_L2}
\vspace{2mm}
\centering
\begingroup
\setlength{\tabcolsep}{2.5pt} 
\renewcommand{\arraystretch}{1.1} 
\setlength{\arrayrulewidth}{0.1mm} 
\resizebox{\linewidth}{!}{%
    \begin{tabular}{cl I I c I c I c I c I c I c}
    \hline\thickhline
    \rowcolor{mygray} 
    \textsc{\textbf{Setups}} & \textsc{\textbf{Datasets}} 
    & \textsc{\textbf{\texttt{N1A0}}} & \textsc{\textbf{\texttt{N1A2}}} & \textsc{\textbf{\texttt{N1A4}}} & \textsc{\textbf{\texttt{N4A0}}} 
    & \textsc{\textbf{\texttt{N4A2}}} & \textsc{\textbf{\texttt{N4A4}}} \\
    \hline\hline
    & MVTecAD 
        & (94.0 / 94.5) & (96.5 / 97.0) & (96.9 / \textcolor{blue}{\textbf{97.3}}) & (94.2 / 94.9) & (\textcolor{blue}{\textbf{98.2}} / \textcolor{DeepRed}{\textbf{97.6}}) & (\textcolor{DeepRed}{\textbf{98.5}} / \textcolor{DeepRed}{\textbf{97.6}})\\
    & VisA 
        & (84.3 / 94.3) & (91.0 / 97.8) & (91.3 / 97.9) & (88.3 / 94.5) & (\textcolor{blue}{\textbf{92.9}} / \textcolor{blue}{\textbf{98.2}}) & (\textcolor{DeepRed}{\textbf{93.0}} / \textcolor{DeepRed}{\textbf{98.4}})\\
    \multirow{-3}{*}{\makecell[c]{\texttt{homo}}} & MPDD 
        & (72.6 / 94.8) & (79.4 / \textcolor{blue}{\textbf{98.1}}) & (79.9 / \textcolor{blue}{\textbf{98.1}}) & (83.0 / 96.1) & (\textcolor{blue}{\textbf{87.8}} / \textcolor{DeepRed}{\textbf{98.5}}) & (\textcolor{DeepRed}{\textbf{87.9}} / \textcolor{DeepRed}{\textbf{98.5}})\\
    \hline
    & MVTecAD 
        & (94.0 / 94.5) & (96.8 / 97.2) & (97.0 / \textcolor{blue}{\textbf{97.5}}) & (94.2 / 94.9) & (\textcolor{blue}{\textbf{98.4}} / \textcolor{blue}{\textbf{97.5}}) & (\textcolor{DeepRed}{\textbf{98.5}} / \textcolor{DeepRed}{\textbf{97.7}})\\
    & VisA 
        & (84.3 / 94.3) & (91.3 / 97.9) & (91.5 / \textcolor{blue}{\textbf{98.0}}) & (88.3 / 94.5) & (\textcolor{blue}{\textbf{93.0}} / \textcolor{DeepRed}{\textbf{98.4}}) & (\textcolor{DeepRed}{\textbf{93.4}} / \textcolor{DeepRed}{\textbf{98.4}})\\
    \multirow{-3}{*}{\makecell[c]{\texttt{hete}}} & MPDD 
        & (72.6 / 94.8) & (80.1 / 98.0) & (80.3 / 98.3) & (83.0 / 96.1) & (\textcolor{blue}{\textbf{87.9}} / \textcolor{blue}{\textbf{98.5}}) & (\textcolor{DeepRed}{\textbf{88.2}} / \textcolor{DeepRed}{\textbf{98.6}})\\
    \hline\thickhline
    \end{tabular}
}
\endgroup
\end{table*}

\begin{table*}[t!]
\caption{\textbf{Cross-validation results (\textit{Hard}) with different reference anomaly types} under the \texttt{N1A1} setting, on two MVTecAD dataset's random objects (\texttt{[wood]} and \texttt{[capsule]}).}
\label{tab:tab_type}
\vspace{2mm}
\centering
\begingroup
\setlength{\tabcolsep}{0.2pt} 
\renewcommand{\arraystretch}{1.46} 
\setlength{\arrayrulewidth}{0.1mm} 
\resizebox{\linewidth}{!}{%
    \begin{tabular}{l I I c I c I c I c I c}
    \hline\thickhline
    \rowcolor{mygray} 
    \textsc{\textbf{Methods}} & \textsc{\textbf{\texttt{[wood]=color}}} & \textsc{\textbf{\texttt{combined}}} & \textsc{\textbf{\texttt{hole}}} & \textsc{\textbf{\texttt{liquid}}} 
    & \textsc{\textbf{\texttt{scratch}}}\\
    \hline
    NAGL~\cite{wang2025normal} 
        & (99.8$\pm$0.1 / 96.7$\pm$0.0)	
        & (99.8$\pm$0.0 / 94.6$\pm$1.5)	
        & (99.9$\pm$0.0 / 95.7$\pm$0.1)	
        & (100.0$\pm$0.0 / 94.7$\pm$0.1) 
        & (99.8$\pm$0.1 / 97.3$\pm$0.3) \\
    \cellcolor{lightyellow}\textbf{IDEAL} (ours) 
        & (99.9$\pm$0.0 / 97.0$\pm$0.4) 
        & (99.8$\pm$0.1 / 97.5$\pm$0.0)	
        & (99.9$\pm$0.0 / 97.1$\pm$0.4)	
        & (100.0$\pm$0.0 / 97.5$\pm$0.2) 
        & (99.9$\pm$0.1 / 97.9$\pm$0.1) \\
    \hline\hline
    \rowcolor{mygray} 
    \textsc{\textbf{Methods}} & \textsc{\textbf{\texttt{[capsule]=crack}}} & \textsc{\textbf{\texttt{imprint}}} & \textsc{\textbf{\texttt{poke}}} & \textsc{\textbf{\texttt{squeeze}}} 
    & \textsc{\textbf{\texttt{scratch}}}\\
    \hline
    NAGL~\cite{wang2025normal} 
        & (82.2$\pm$0.8 / 93.1$\pm$3.0)	
        & (82.5$\pm$4.2 / 95.4$\pm$0.6)	
        & (83.7$\pm$1.6 / 96.4$\pm$0.2)	
        & (85.2$\pm$1.0 / 92.6$\pm$1.1)	
        & (83.7$\pm$0.3 / 94.4$\pm$1.1) \\
    \cellcolor{lightyellow}\textbf{IDEAL} (ours)
        & (85.9$\pm$0.0 / 97.3$\pm$0.5)	
        & (86.1$\pm$0.2 / 98.0$\pm$0.1)	
        & (85.5$\pm$0.4 / 97.5$\pm$0.1)	
        & (85.8$\pm$0.6 / 97.9$\pm$0.2)	
        & (85.4$\pm$0.1 / 97.5$\pm$0.0) \\
    \hline\thickhline
    \end{tabular}
}
\endgroup
\end{table*}

\paragraph{Analysis of Pixel-level Anomaly Masks.}
\begin{wraptable}{r}{0.5\linewidth}
\vspace{-6.6mm}
\caption{\textbf{Results (\textit{General}) using pseudo (\texttt{Pse}) /truth anomaly masks}, under the \texttt{N1A1} setting.}
\label{tab:tab_mask}
\vspace{2mm}
\centering
\begingroup
\setlength{\tabcolsep}{3.6pt} 
\renewcommand{\arraystretch}{1.28} 
\setlength{\arrayrulewidth}{0.1mm} 
\resizebox{\linewidth}{!}{%
    \begin{tabular}{l I I c I c I c}
    \hline\thickhline
    \rowcolor{mygray} 
    \textsc{\textbf{Methods}} 
    & \textsc{\textbf{MVTecAD}} 
    & \textsc{\textbf{VisA}} 
    & \textsc{\textbf{MPDD}} \\
    \hline\hline
    NAGL~\cite{wang2025normal} 
        & (95.1 / 96.1)	
        & (88.5 / 97.5)	
        & (77.1 / 96.5)	\\
    NAGL (\texttt{Pse}) 
        & (94.3 / 95.9)	
        & (86.4 / 95.7)	
        & (75.7 / 96.0) \\
    \cellcolor{lightyellow}\textbf{IDEAL} 
        & (96.3 / 96.8)	
        & (90.8 / 97.8)	
        & (78.5 / 97.9)	\\
    \cellcolor{lightyellow}\textbf{IDEAL} (\texttt{Pse}) 
        & (95.4 / 96.6)	
        & (89.8 / 97.6)	
        & (77.2 / 97.4) \\
    \hline\thickhline
    \end{tabular}
}
\endgroup
\end{wraptable}
The pixel-level masks can provide stronger supervision than normal-only FSAD. 
IDEAL requires only one fixed abnormal-reference mask per dataset by default ($L_2 = 1$), annotated once and reused for all queries; no query-specific annotations are needed. 
To isolate this mask benefit, we evaluate an image-level-only variant that replaces ground-truth masks with \textit{pseudo-masks} derived from abnormal-normal feature distances (top-$1$\% set as $1$, others set as $0$). 
As shown in Tab.~\ref{tab:tab_mask}, IDEAL remains effective by using pseudo anomaly masks only, having small gaps to the variant using the ground-truth anomaly masks.

\section{Conclusion}
In this paper, we consider a practical few-shot anomaly detection (FSAD) setting, discriminative FSAD, where a few normal and anomalous examples are available as references at inference. 
To effectively leverage discriminative clues of abnormal references relative to normal references, we propose an Intrinsic Deviation Learning (IDEAL) framework to learn intrinsic deviation patterns characterizing generalizable abnormality as deviations from normality, for detecting both seen and unseen anomalies. 
At inference, anomalies are detected by measuring query-to-normal deviations preserved after projection onto learned intrinsic deviation vectors. 
Extensive experiments validate the effectiveness and efficiency of our proposed IDEAL.

\section*{Acknowledgments}
This work is supported by A*STAR under its MTC YIRG Grant (M24N8c0103), the Lee Kong Chian Fellowship (T050273), and the Singapore Ministry of Education (MOE) Academic Research Fund (AcRF) Tier 1 Grant (24-SIS-SMU-008).

{
    \small
    \bibliographystyle{plain}
    \bibliography{main_ref}

@String(CVPR= {IEEE Conf. Comput. Vis. Pattern Recog.})

@String(ICCV= {Int. Conf. Comput. Vis.})

@String(ECCV= {Eur. Conf. Comput. Vis.})

@String(NIPS= {Adv. Neural Inform. Process. Syst.})

@String(ICPR = {Int. Conf. Pattern Recog.})

@String(ACMMM= {ACM Int. Conf. Multimedia})

@String(ACCV  = {ACCV})

@String(ICLR = {Int. Conf. Learn. Represent.})

@String(AAAI = {AAAI})

@String(CVPR  = {CVPR})

@String(ICCV  = {ICCV})

@String(ECCV  = {ECCV})

@String(NIPS  = {NeurIPS})

@String(ICPR  = {ICPR})

@String(ACMMM = {ACM MM})

@String(WACV  = {WACV})

@String(ICLR  = {ICLR})

@String(ICML  = {ICML})

@String(CIKM = {CIKM})

@String(KDD = {ACM SIGKDD})

@article{zhang2020viral,
  title={Viral pneumonia screening on chest X-rays using confidence-aware anomaly detection},
  author={Zhang, Jianpeng and Xie, Yutong and Pang, Guansong and Liao, Zhibin and Verjans, Johan and Li, Wenxing and Sun, Zongji and He, Jian and Li, Yi and Shen, Chunhua and others},
  journal={IEEE transactions on medical imaging},
  volume={40},
  number={3},
  pages={879--890},
  year={2020},
  publisher={IEEE}
}

@article{cao2024survey,
  title={A survey on visual anomaly detection: Challenge, approach, and prospect},
  author={Cao, Yunkang and Xu, Xiaohao and Zhang, Jiangning and Cheng, Yuqi and Huang, Xiaonan and Pang, Guansong and Shen, Weiming},
  journal={arXiv preprint arXiv:2401.16402},
  year={2024}
}

@article{pang2021deep,
  title={Deep learning for anomaly detection: A review},
  author={Pang, Guansong and Shen, Chunhua and Cao, Longbing and Hengel, Anton Van Den},
  journal={ACM Computing Surveys},
  volume={54},
  number={2},
  pages={1--38},
  year={2021}
}

@inproceedings{xie2023pushing,
    title={Pushing the Limits of Fewshot Anomaly Detection in Industry Vision: Graphcore},
    author={Guoyang Xie and Jinbao Wang and Jiaqi Liu and Yaochu Jin and Feng Zheng},
    booktitle=ICLR,
    year={2023}
}

@article{chalapathy2019deep,
  title={Deep learning for anomaly detection: A survey},
  author={Chalapathy, Raghavendra and Chawla, Sanjay},
  journal={arXiv preprint arXiv:1901.03407},
  year={2019}
}

@article{samariya2023comprehensive,
  title={A comprehensive survey of anomaly detection algorithms},
  author={Samariya, Durgesh and Thakkar, Amit},
  journal={Annals of Data Science},
  volume={10},
  number={3},
  pages={829--850},
  year={2023}
}

@inproceedings{chen2022deep,
  title={Deep one-class classification via interpolated gaussian descriptor},
  author={Chen, Yuanhong and Tian, Yu and Pang, Guansong and Carneiro, Gustavo},
  booktitle=AAAI,
  pages={383--392},
  year={2022}
}

@article{bergman2020classification,
  title={Classification-based anomaly detection for general data},
  author={Bergman, Liron and Hoshen, Yedid},
  journal={arXiv preprint arXiv:2005.02359},
  year={2020}
}

@inproceedings{fang2023fastrecon,
  title={Fastrecon: Few-shot industrial anomaly detection via fast feature reconstruction},
  author={Fang, Zheng and Wang, Xiaoyang and Li, Haocheng and Liu, Jiejie and Hu, Qiugui and Xiao, Jimin},
  booktitle=ICCV,
  pages={17481--17490},
  year={2023}
}

@inproceedings{pang2019deep,
  title={Deep anomaly detection with deviation networks},
  author={Pang, Guansong and Shen, Chunhua and Van Den Hengel, Anton},
  booktitle=KDD,
  pages={353--362},
  year={2019}
}

@inproceedings{wang2023multimodal,
  title={Multimodal industrial anomaly detection via hybrid fusion},
  author={Wang, Yue and Peng, Jinlong and Zhang, Jiangning and Yi, Ran and Wang, Yabiao and Wang, Chengjie},
  booktitle=CVPR,
  pages={8032--8041},
  year={2023}
}

@inproceedings{bao2024bmad,
  title={Bmad: Benchmarks for medical anomaly detection},
  author={Bao, Jinan and Sun, Hanshi and Deng, Hanqiu and He, Yinsheng and Zhang, Zhaoxiang and Li, Xingyu},
  booktitle=CVPR,
  pages={4042--4053},
  year={2024}
}

@inproceedings{huang2024adapting,
  title={Adapting visual-language models for generalizable anomaly detection in medical images},
  author={Huang, Chaoqin and Jiang, Aofan and Feng, Jinghao and Zhang, Ya and Wang, Xinchao and Wang, Yanfeng},
  booktitle=CVPR,
  pages={11375--11385},
  year={2024}
}

@inproceedings{zhu2024toward,
  title={Toward generalist anomaly detection via in-context residual learning with few-shot sample prompts},
  author={Zhu, Jiawen and Pang, Guansong},
  booktitle=CVPR,
  pages={17826--17836},
  year={2024}
}

@inproceedings{guo2025dinomaly,
  title={Dinomaly: The less is more philosophy in multi-class unsupervised anomaly detection},
  author={Guo, Jia and Lu, Shuai and Zhang, Weihang and Chen, Fang and Li, Huiqi and Liao, Hongen},
  booktitle=CVPR,
  pages={20405--20415},
  year={2025}
}

@inproceedings{yao2024resad,
  title={Resad: A simple framework for class generalizable anomaly detection},
  author={Yao, Xincheng and Chen, Zixin and Gao, Chao and Zhai, Guangtao and Zhang, Chongyang},
  booktitle=NIPS,
  pages={125287--125311},
  year={2024}
}

@inproceedings{luo2025exploring,
  title={Exploring intrinsic normal prototypes within a single image for universal anomaly detection},
  author={Luo, Wei and Cao, Yunkang and Yao, Haiming and Zhang, Xiaotian and Lou, Jianan and Cheng, Yuqi and Shen, Weiming and Yu, Wenyong},
  booktitle=CVPR,
  pages={9974--9983},
  year={2025}
}

@inproceedings{hou2021divide,
  title={Divide-and-assemble: Learning block-wise memory for unsupervised anomaly detection},
  author={Hou, Jinlei and Zhang, Yingying and Zhong, Qiaoyong and Xie, Di and Pu, Shiliang and Zhou, Hong},
  booktitle=ICCV,
  pages={8791--8800},
  year={2021}
}

@inproceedings{liu2023diversity,
  title={Diversity-measurable anomaly detection},
  author={Liu, Wenrui and Chang, Hong and Ma, Bingpeng and Shan, Shiguang and Chen, Xilin},
  booktitle=CVPR,
  pages={12147--12156},
  year={2023}
}

@inproceedings{park2020learning,
  title={Learning memory-guided normality for anomaly detection},
  author={Park, Hyunjong and Noh, Jongyoun and Ham, Bumsub},
  booktitle=CVPR,
  pages={14372--14381},
  year={2020}
}

@article{schlegl2019f,
  title={f-AnoGAN: Fast unsupervised anomaly detection with generative adversarial networks},
  author={Schlegl, Thomas and Seeb{\"o}ck, Philipp and Waldstein, Sebastian M and Langs, Georg and Schmidt-Erfurth, Ursula},
  journal={Medical Image Analysis},
  volume={54},
  pages={30--44},
  year={2019},
  publisher={Elsevier}
}

@inproceedings{tayeh2020distance,
  title={Distance-based anomaly detection for industrial surfaces using triplet networks},
  author={Tayeh, Tareq and Aburakhia, Sulaiman and Myers, Ryan and Shami, Abdallah},
  booktitle={IEMCON},
  pages={0372--0377},
  year={2020},
  organization={IEEE}
}

@inproceedings{zhou2023anomalyclip,
  title={AnomalyCLIP: Object-agnostic Prompt Learning for Zero-shot Anomaly Detection},
  author={Zhou, Qihang and Pang, Guansong and Tian, Yu and He, Shibo and Chen, Jiming},
  booktitle=ICLR,
  year={2024}
}

@inproceedings{zhu2024fine,
  title={Fine-grained Abnormality Prompt Learning for Zero-shot Anomaly Detection},
  author={Zhu, Jiawen and Ong, Yew-Soon and Shen, Chunhua and Pang, Guansong},
  booktitle=ICCV,
  pages={22241--22251},
  year={2025}
}

@inproceedings{li2023zeroshot,
    title={Zero-Shot Anomaly Detection via Batch Normalization},
    author={Aodong Li and Chen Qiu and Marius Kloft and Padhraic Smyth and Maja Rudolph and Stephan Mandt},
    booktitle=NIPS,
    pages={40963--40993},
    year={2023}
}

@inproceedings{xu2025towards,
  title={Towards Zero-Shot Anomaly Detection and Reasoning with Multimodal Large Language Models},
  author={Xu, Jiacong and Lo, Shao-Yuan and Safaei, Bardia and Patel, Vishal M and Dwivedi, Isht},
  booktitle=CVPR,
  pages={20370--20382},
  year={2025}
}

@inproceedings{ma2025aaclip,
  title={AA-CLIP: Enhancing Zero-shot Anomaly Detection via Anomaly-Aware CLIP}, 
  author={Ma, Wenxin and Zhang, Xu and Yao, Qingsong and Tang, Fenghe and Wu, Chenxu and Li, Yingtai and Yan, Rui and Jiang, Zihang and Zhou, S. Kevin},
  booktitle=CVPR,
  pages={4744--4754},
  year={2025}
}

@inproceedings{gu2024filo,
  title={FiLo: Zero-Shot Anomaly Detection by Fine-Grained Description and High-Quality Localization},
  author={Gu, Zhaopeng and Zhu, Bingke and Zhu, Guibo and Chen, Yingying and Li, Hao and Tang, Ming and Wang, Jinqiao},
  booktitle=ACMMM,
  pages={2041--2049},
  year={2024}
}

@inproceedings{salehi2021multiresolution,
  title={Multiresolution knowledge distillation for anomaly detection},
  author={Salehi, Mohammadreza and Sadjadi, Niousha and Baselizadeh, Soroosh and Rohban, Mohammad H and Rabiee, Hamid R},
  booktitle=CVPR,
  pages={14902--14912},
  year={2021}
}

@inproceedings{zhang2023destseg,
  title={Destseg: Segmentation guided denoising student-teacher for anomaly detection},
  author={Zhang, Xuan and Li, Shiyu and Li, Xi and Huang, Ping and Shan, Jiulong and Chen, Ting},
  booktitle=CVPR,
  pages={3914--3923},
  year={2023}
}

@inproceedings{zhang2024realnet,
  title={Realnet: A feature selection network with realistic synthetic anomaly for anomaly detection},
  author={Zhang, Ximiao and Xu, Min and Zhou, Xiuzhuang},
  booktitle=CVPR,
  pages={16699--16708},
  year={2024}
}

@inproceedings{liu2023simplenet,
  title={Simplenet: A simple network for image anomaly detection and localization},
  author={Liu, Zhikang and Zhou, Yiming and Xu, Yuansheng and Wang, Zilei},
  booktitle=CVPR,
  pages={20402--20411},
  year={2023}
}

@article{alabdulatif2017privacy,
  title={Privacy-preserving anomaly detection in cloud with lightweight homomorphic encryption},
  author={Alabdulatif, Abdulatif and Kumarage, Heshan and Khalil, Ibrahim and Yi, Xun},
  journal={Journal of Computer and System Sciences},
  volume={90},
  pages={28--45},
  year={2017}
}

@inproceedings{mayer2020privacy,
  title={Privacy-preserving anomaly detection using synthetic data},
  author={Mayer, Rudolf and Hittmeir, Markus and Ekelhart, Andreas},
  booktitle={IFIP Annual Conference on Data and Applications Security and Privacy},
  pages={195--207},
  year={2020},
  organization={Springer}
}

@inproceedings{li2021cutpaste,
  title={Cutpaste: Self-supervised learning for anomaly detection and localization},
  author={Li, Chun-Liang and Sohn, Kihyuk and Yoon, Jinsung and Pfister, Tomas},
  booktitle=CVPR,
  pages={9664--9674},
  year={2021}
}

@inproceedings{zavrtanik2021draem,
  title={Draem-a discriminatively trained reconstruction embedding for surface anomaly detection},
  author={Zavrtanik, Vitjan and Kristan, Matej and Sko{\v{c}}aj, Danijel},
  booktitle=ICCV,
  pages={8330--8339},
  year={2021}
}

@inproceedings{cao2023anomaly,
  title={Anomaly detection under distribution shift},
  author={Cao, Tri and Zhu, Jiawen and Pang, Guansong},
  booktitle=ICCV,
  pages={6511--6523},
  year={2023}
}

@inproceedings{bergmann2020uninformed,
  title={Uninformed students: Student-teacher anomaly detection with discriminative latent embeddings},
  author={Bergmann, Paul and Fauser, Michael and Sattlegger, David and Steger, Carsten},
  booktitle=CVPR,
  pages={4183--4192},
  year={2020}
}

@inproceedings{ding2022catching,
  title={Catching both gray and black swans: Open-set supervised anomaly detection},
  author={Ding, Choubo and Pang, Guansong and Shen, Chunhua},
  booktitle=CVPR,
  pages={7388--7398},
  year={2022}
}

@inproceedings{akcay2018ganomaly,
  title={Ganomaly: Semi-supervised anomaly detection via adversarial training},
  author={Akcay, Samet and Atapour-Abarghouei, Amir and Breckon, Toby P},
  booktitle=ACCV,
  pages={622--637},
  year={2018}
}

@inproceedings{li2026iad,
  title={Iad-r1: Reinforcing consistent reasoning in industrial anomaly detection},
  author={Li, Yanhui and Cao, Yunkang and Liu, Chengliang and Xiong, Yuan and Dong, Xinghui and Huang, Chao},
  booktitle=AAAI,
  pages={6583--6591},
  year={2026}
}

@article{fernando2021deep,
  title={Deep learning for medical anomaly detection--a survey},
  author={Fernando, Tharindu and Gammulle, Harshala and Denman, Simon and Sridharan, Sridha and Fookes, Clinton},
  journal={ACM Computing Surveys},
  volume={54},
  number={7},
  pages={1--37},
  year={2021}
}

@inproceedings{zhu2024anomaly,
  title={Anomaly heterogeneity learning for open-set supervised anomaly detection},
  author={Zhu, Jiawen and Ding, Choubo and Tian, Yu and Pang, Guansong},
  booktitle=CVPR,
  pages={17616--17626},
  year={2024}
}

@inproceedings{wang2025distribution,
  title={Distribution prototype diffusion learning for open-set supervised anomaly detection},
  author={Wang, Fuyun and Zhang, Tong and Wang, Yuanzhi and Qiu, Yide and Liu, Xin and Guo, Xu and Cui, Zhen},
  booktitle=CVPR,
  pages={20416--20426},
  year={2025}
}

@inproceedings{wang2025normal,
  title={Normal-Abnormal Guided Generalist Anomaly Detection},
  author={Wang, Yuexin and Wang, Xiaolei and Gong, Yizheng and Xiao, Jimin},
  booktitle=NIPS,
  year={2025}
}

@inproceedings{wang2019gods,
  title={Gods: Generalized one-class discriminative subspaces for anomaly detection},
  author={Wang, Jue and Cherian, Anoop},
  booktitle=ICCV,
  pages={8201--8211},
  year={2019}
}

@inproceedings{deng2022anomaly,
  title={Anomaly detection via reverse distillation from one-class embedding},
  author={Deng, Hanqiu and Li, Xingyu},
  booktitle=CVPR,
  pages={9737--9746},
  year={2022}
}

@inproceedings{huang2022registration,
  title={Registration based few-shot anomaly detection},
  author={Huang, Chaoqin and Guan, Haoyan and Jiang, Aofan and Zhang, Ya and Spratling, Michael and Wang, Yan-Feng},
  booktitle=ECCV,
  pages={303--319},
  year={2022}
}

@inproceedings{li2024promptad,
  title={Promptad: Learning prompts with only normal samples for few-shot anomaly detection},
  author={Li, Xiaofan and Zhang, Zhizhong and Tan, Xin and Chen, Chengwei and Qu, Yanyun and Xie, Yuan and Ma, Lizhuang},
  booktitle=CVPR,
  pages={16838--16848},
  year={2024}
}

@inproceedings{wang2022few,
  title={Few-shot fast-adaptive anomaly detection},
  author={Wang, Ze and Zhou, Yipin and Wang, Rui and Lin, Tsung-Yu and Shah, Ashish and Lim, Ser Nam},
  booktitle=NIPS,
  pages={4957--4970},
  year={2022}
}

@article{cohen2020sub,
  title={Sub-image anomaly detection with deep pyramid correspondences},
  author={Cohen, Niv and Hoshen, Yedid},
  journal={arXiv preprint arXiv:2005.02357},
  year={2020}
}

@inproceedings{defard2021padim,
  title={Padim: a patch distribution modeling framework for anomaly detection and localization},
  author={Defard, Thomas and Setkov, Aleksandr and Loesch, Angelique and Audigier, Romaric},
  booktitle=ICPR,
  pages={475--489},
  year={2021}
}

@inproceedings{pang2018learning,
  title={Learning representations of ultrahigh-dimensional data for random distance-based outlier detection},
  author={Pang, Guansong and Cao, Longbing and Chen, Ling and Liu, Huan},
  booktitle=KDD,
  pages={2041--2050},
  year={2018}
}

@inproceedings{roth2022towards,
  title={Towards total recall in industrial anomaly detection},
  author={Roth, Karsten and Pemula, Latha and Zepeda, Joaquin and Sch{\"o}lkopf, Bernhard and Brox, Thomas and Gehler, Peter},
  booktitle=CVPR,
  pages={14318--14328},
  year={2022}
}

@inproceedings{he2024diffusion,
  title={A diffusion-based framework for multi-class anomaly detection},
  author={He, Haoyang and Zhang, Jiangning and Chen, Hongxu and Chen, Xuhai and Li, Zhishan and Chen, Xu and Wang, Yabiao and Wang, Chengjie and Xie, Lei},
  booktitle=AAAI,
  pages={8472--8480},
  year={2024}
}

@inproceedings{you2022unified,
  title={A unified model for multi-class anomaly detection},
  author={You, Zhiyuan and Cui, Lei and Shen, Yujun and Yang, Kai and Lu, Xin and Zheng, Yu and Le, Xinyi},
  booktitle=NIPS,
  pages={4571--4584},
  year={2022}
}

@article{liao2024coft,
  title={COFT-AD: Contrastive fine-tuning for few-shot anomaly detection},
  author={Liao, Jingyi and Xu, Xun and Nguyen, Manh Cuong and Goodge, Adam and Foo, Chuan Sheng},
  journal={IEEE Transactions on Image Processing},
  volume={33},
  pages={2090--2103},
  year={2024}
}

@inproceedings{tian2024foct,
  title={Foct: Few-shot industrial anomaly detection with foreground-aware online conditional transport},
  author={Tian, Long and Zhao, Hongyi and Lu, Ruiying and Wang, Rongrong and Wu, Yujie and Wang, Liming and He, Xiongpeng and Liu, Xiyang},
  booktitle=ACMMM,
  pages={6241--6249},
  year={2024}
}

@inproceedings{tao2025kernel,
  title={Kernel-aware graph prompt learning for few-shot anomaly detection},
  author={Tao, Fenfang and Xie, Guo-Sen and Zhao, Fang and Shu, Xiangbo},
  booktitle=AAAI,
  pages={7347--7355},
  year={2025}
}

@inproceedings{yao2023explicit,
  title={Explicit boundary guided semi-push-pull contrastive learning for supervised anomaly detection},
  author={Yao, Xincheng and Li, Ruoqi and Zhang, Jing and Sun, Jun and Zhang, Chongyang},
  booktitle=CVPR,
  pages={24490--24499},
  year={2023}
}

@inproceedings{Ruff2020Deep,
    title={Deep Semi-Supervised Anomaly Detection},
    author={Lukas Ruff and Robert A. Vandermeulen and Nico Görnitz and Alexander Binder and Emmanuel Müller and Klaus-Robert Müller and Marius Kloft},
    booktitle=ICLR,
    year={2020}
}

@inproceedings{yi2020patch,
  title={Patch svdd: Patch-level svdd for anomaly detection and segmentation},
  author={Yi, Jihun and Yoon, Sungroh},
  booktitle=ACCV,
  year={2020}
}

@article{abdi2010principal,
  title={Principal component analysis},
  author={Abdi, Herv{\'e} and Williams, Lynne J},
  journal={Wiley Interdisciplinary Reviews: Computational Statistics},
  volume={2},
  number={4},
  pages={433--459},
  year={2010}
}

@article{tipping1999prob,
  title={Probabilistic principal component analysis},
  author={Tipping, Michael E and Bishop, Christopher M},
  journal={Journal of the Royal Statistical Society Series B: Statistical Methodology},
  volume={61},
  number={3},
  pages={611--622},
  year={1999}
}

@inproceedings{yao2024hierarchical,
  title={Hierarchical gaussian mixture normalizing flow modeling for unified anomaly detection},
  author={Yao, Xincheng and Li, Ruoqi and Qian, Zefeng and Wang, Lu and Zhang, Chongyang},
  booktitle=ECCV,
  pages={92--108},
  year={2024}
}

@inproceedings{yao2023one,
  title={One-for-all: Proposal masked cross-class anomaly detection},
  author={Yao, Xincheng and Zhang, Chongyang and Li, Ruoqi and Sun, Jun and Liu, Zhenyu},
  booktitle=AAAI,
  pages={4792--4800},
  year={2023}
}

@inproceedings{lv2025one,
  title={One-for-all few-shot anomaly detection via instance-induced prompt learning},
  author={Lv, Wenxi and Su, Qinliang and Xu, Wenchao},
  booktitle=ICLR,
  year={2025}
}

@inproceedings{lv2025metacan,
  title={MetaCAN: Improving Generalizability of Few-shot Anomaly Detection with Meta-learning},
  author={Lv, Zhisheng and Zhang, Jianfeng and Jian, Songlei and Huang, Chenlin and Zhang, Hongguang and Pang, Guansong and Liu, Zhong},
  booktitle=CIKM,
  pages={2032--2041},
  year={2025}
}

@inproceedings{jeong2023winclip,
  title={Winclip: Zero-/few-shot anomaly classification and segmentation},
  author={Jeong, Jongheon and Zou, Yang and Kim, Taewan and Zhang, Dongqing and Ravichandran, Avinash and Dabeer, Onkar},
  booktitle=CVPR,
  pages={19606--19616},
  year={2023}
}

@inproceedings{dong2026dual,
    title={Dual Distillation for Few-Shot Anomaly Detection},
    author={Le Dong and Qinzhong Tan and Chunlei Li and Jingliang Hu and Yilei Shi and Weisheng Dong and Xiao Xiang Zhu and Lichao Mou},
    booktitle=ICLR,
    year={2026}
}

@inproceedings{vaswani2017attention,
  title={Attention is all you need},
  author={Vaswani, Ashish and Shazeer, Noam and Parmar, Niki and Uszkoreit, Jakob and Jones, Llion and Gomez, Aidan N and Kaiser, {\L}ukasz and Polosukhin, Illia},
  booktitle=NIPS,
  pages={5998--6008},
  year={2017}
}

@inproceedings{damm2025anomalydino,
  title={Anomalydino: Boosting patch-based few-shot anomaly detection with dinov2},
  author={Damm, Simon and Laszkiewicz, Mike and Lederer, Johannes and Fischer, Asja},
  booktitle=WACV,
  pages={1319--1329},
  year={2025}
}

@inproceedings{hwang2024anomaly,
  title={Anomaly score: Evaluating generative models and individual generated images based on complexity and vulnerability},
  author={Hwang, Jaehui and Lee, Junghyuk and Lee, Jong-Seok},
  booktitle=CVPR,
  pages={8754--8763},
  year={2024}
}

@inproceedings{lin2017focal,
  title={Focal loss for dense object detection},
  author={Lin, Tsung-Yi and Goyal, Priya and Girshick, Ross and He, Kaiming and Doll{\'a}r, Piotr},
  booktitle=ICCV,
  pages={2980--2988},
  year={2017}
}

@inproceedings{milletari2016v,
  title={V-net: Fully convolutional neural networks for volumetric medical image segmentation},
  author={Milletari, Fausto and Navab, Nassir and Ahmadi, Seyed-Ahmad},
  booktitle={International Conference on 3D Vision},
  pages={565--571},
  year={2016}
}

@inproceedings{mannor2005cross,
  title={The cross entropy method for classification},
  author={Mannor, Shie and Peleg, Dori and Rubinstein, Reuven},
  booktitle=ICML,
  pages={561--568},
  year={2005}
}

@inproceedings{mansour2008domain,
  title={Domain adaptation with multiple sources},
  author={Mansour, Yishay and Mohri, Mehryar and Rostamizadeh, Afshin},
  booktitle=NIPS,
  pages={1041--1048},
  year={2008}
}

@inproceedings{ben2006analysis,
  title={Analysis of representations for domain adaptation},
  author={Ben-David, Shai and Blitzer, John and Crammer, Koby and Pereira, Fernando},
  booktitle=NIPS,
  pages={137--144},
  year={2006}
}

@article{fang2014domain,
  title={Domain adaptation for sentiment classification in light of multiple sources},
  author={Fang, Fang and Dutta, Kaushik and Datta, Anindya},
  journal={INFORMS Journal on Computing},
  volume={26},
  number={3},
  pages={586--598},
  year={2014}
}

@article{ben2010theory,
  title={A theory of learning from different domains},
  author={Ben-David, Shai and Blitzer, John and Crammer, Koby and Kulesza, Alex and Pereira, Fernando and Vaughan, Jennifer Wortman},
  journal={Machine Learning},
  volume={79},
  number={1},
  pages={151--175},
  year={2010},
  publisher={Springer}
}

@inproceedings{sun2011two,
  title={A two-stage weighting framework for multi-source domain adaptation},
  author={Sun, Qian and Chattopadhyay, Rita and Panchanathan, Sethuraman and Ye, Jieping},
  booktitle=NIPS,
  pages={505--513},
  year={2011}
}

@inproceedings{zhang2019bridging,
  title={Bridging theory and algorithm for domain adaptation},
  author={Zhang, Yuchen and Liu, Tianle and Long, Mingsheng and Jordan, Michael},
  booktitle=ICML,
  pages={7404--7413},
  year={2019}
}

@article{zhang2020unsupervised,
  title={Unsupervised multi-class domain adaptation: Theory, algorithms, and practice},
  author={Zhang, Yabin and Deng, Bin and Tang, Hui and Zhang, Lei and Jia, Kui},
  journal={IEEE Transactions on Pattern Analysis and Machine Intelligence},
  volume={44},
  number={5},
  pages={2775--2792},
  year={2020}
}

@inproceedings{bergmann2019mvtec,
  title={MVTec AD--A comprehensive real-world dataset for unsupervised anomaly detection},
  author={Bergmann, Paul and Fauser, Michael and Sattlegger, David and Steger, Carsten},
  booktitle=CVPR,
  pages={9592--9600},
  year={2019}
}

@inproceedings{zou2022spot,
  title={Spot-the-difference self-supervised pre-training for anomaly detection and segmentation},
  author={Zou, Yang and Jeong, Jongheon and Pemula, Latha and Zhang, Dongqing and Dabeer, Onkar},
  booktitle=ECCV,
  pages={392--408},
  year={2022}
}

@article{silvestre2019public,
  title={A public fabric database for defect detection methods and results},
  author={Javier Silvestre-Blanes and Teresa Albero-Albero and Ignacio Miralles and Rubén Pérez-Llorens and Jorge Moreno},
  journal={Autex Research Journal},
  pages={363--374},
  volume={19},
  number={4},
  year={2019}
}

@inproceedings{jezek2021deep,
  title={Deep learning-based defect detection of metal parts: evaluating current methods in complex conditions},
  author={Jezek, Stepan and Jonak, Martin and Burget, Radim and Dvorak, Pavel and Skotak, Milos},
  booktitle={International Congress on Ultra Modern Telecommunications and Control Systems and Workshops},
  pages={66--71},
  year={2021}
}

@inproceedings{mishra2021vt,
  title={VT-ADL: A vision transformer network for image anomaly detection and localization},
  author={Mishra, Pankaj and Verk, Riccardo and Fornasier, Daniele and Piciarelli, Claudio and Foresti, Gian Luca},
  booktitle={International Symposium on Industrial Electronics},
  pages={01--06},
  year={2021}
}

@article{menze2014multimodal,
  title={The multimodal brain tumor image segmentation benchmark (BRATS)},
  author={Menze, Bjoern H and Jakab, Andras and Bauer, Stefan and Kalpathy-Cramer, Jayashree and Farahani, Keyvan and Kirby, Justin and Burren, Yuliya and Porz, Nicole and Slotboom, Johannes and Wiest, Roland and others},
  journal={IEEE Transactions on Medical Imaging},
  volume={34},
  number={10},
  pages={1993--2024},
  year={2015}
}

@article{hu2019automated,
  title={Automated segmentation of macular edema in OCT using deep neural networks},
  author={Hu, Junjie and Chen, Yuanyuan and Yi, Zhang},
  journal={Medical Image Analysis},
  volume={55},
  pages={216--227},
  year={2019}
}

@inproceedings{zhai2026foundation,
    title={Foundation Visual Encoders Are Secretly Few-Shot Anomaly Detectors},
    author={Guangyao Zhai and Yue Zhou and Xinyan Deng and Lars Heckler-Kram and Nassir Navab and Benjamin Busam},
    booktitle=ICLR,
    year={2026}
}

@inproceedings{dosovitskiy2021an,
    title={An Image is Worth 16x16 Words: Transformers for Image Recognition at Scale},
    author={Alexey Dosovitskiy and Lucas Beyer and Alexander Kolesnikov and Dirk Weissenborn and Xiaohua Zhai and Thomas Unterthiner and Mostafa Dehghani and Matthias Minderer and Georg Heigold and Sylvain Gelly and Jakob Uszkoreit and Neil Houlsby},
    booktitle=ICLR,
    year={2021}
}

@inproceedings{loshchilov2018decoupled,
    title={Decoupled Weight Decay Regularization},
    author={Ilya Loshchilov and Frank Hutter},
    booktitle=ICLR,
    year={2019}
}

@article{van2008visualizing,
  title={Visualizing data using t-SNE},
  author={Van der Maaten, Laurens and Hinton, Geoffrey},
  journal={Journal of Machine Learning Research},
  volume={9},
  number={11},
  pages={2579--2605},
  year={2008}
}
}

\clearpage
\newpage
\appendix
\renewcommand{\thefigure}{A\arabic{figure}}
\renewcommand{\thetable}{A\arabic{table}}
\renewcommand{\theequation}{A\arabic{equation}}
\makeatletter
\renewcommand{\theHfigure}{appendix.A.\arabic{figure}}
\renewcommand{\theHtable}{appendix.A.\arabic{table}}
\renewcommand{\theHequation}{appendix.A.\arabic{equation}}
\makeatother
\setcounter{figure}{0}
\setcounter{table}{0}
\setcounter{equation}{0}

\section{Algorithm Pseudo-code Flow} \label{sec:pcode}
In this section, we describe the pseudo-code of IDEAL in Algorithm~\ref{alg:ours}: 
1) \textit{Training on Auxiliary Datasets}: IDEAL is trained on auxiliary datasets in an episodic manner, where each training episode is constructed to simulate the few-shot inference scenario, consisting of a query image together with limited normal and abnormal references sampled from the source dataset. 
2) \textit{Inference}: IDEAL is evaluated on the target dataset without further training. For each target dataset, a few fixed normal and abnormal references are provided together with a test query. Based on learned intrinsic deviation vectors, IDEAL produces deviation-guided anomaly scores for each test query.

\begin{algorithm}[H]
\caption{\colorbox{myyellow}{\textbf{IDEAL}: Intrinsic Deviation Learning for Discriminative FSAD}}
\label{alg:ours}
\SetAlgoLined
\SetNoFillComment
\SetArgSty{textnormal}

{\color{LightBlue}{\tcc{\textbf{\colorbox{myblue}{Training on Auxiliary Datasets}}}}}

$\mathcal{D}_{tr}^{q} \Leftarrow \mathcal{D}_{tr}$ 
{\footnotesize{\color{DarkBlue}{\tcp{randomly select training queries from auxiliary datasets}}}}

\For {each training query $\mathbf{x}^{q}$ from $\mathcal{D}_{tr}^{q}$}{
    $\mathcal{S}^{n}, \mathcal{S}^{a} \Leftarrow \mathcal{D}_{tr}$ 
    {\footnotesize{\color{DarkBlue}{\tcp{randomly sample normal and abnormal references}}}}

    $\mathbf{E}=\{\mathbf{x}^{q}, \mathcal{S}^{n}, \mathcal{S}^{a}\}$ 
    {\footnotesize{\color{DarkBlue}{\tcp{construct training episode}}}}

    $\mathcal{F}^{q}, \mathcal{F}^{n}, \mathcal{F}^{a} \; \leftarrow \; \mathcal{E}(\mathbf{x}^{q}), \mathcal{E}(\mathcal{S}^{n}), \mathcal{E}(\mathcal{S}^{a})$ 
    {\footnotesize{\color{DarkBlue}{\tcp{extracting features}}}}
    
    {\footnotesize{\color{DarkBlue}{\tcc{Normal Variation Eraser (Sec.~\ref{sec:nve})}}}}
    
    $\mathcal{F}^{a}_{\mathtt{res}} = \{f^{a}_{i,\mathtt{res}}\}_{i=1}^{L_{2} \cdot N} \leftarrow \;$ Eq.~\eqref{eq:eq1} 
    {\footnotesize{\color{DarkBlue}{\tcp{residual deviations}}}}
    
    $\mathcal{F}^{a}_{\mathtt{den}} = \{f^{a}_{i,\mathtt{den}}\}_{i=1}^{L_{2} \cdot N} \leftarrow \;$ Eq.~\eqref{eq:eq4} 
    {\footnotesize{\color{DarkBlue}{\tcp{denoised deviations}}}}

    {\footnotesize{\color{DarkBlue}{\tcc{Intrinsic Deviation Encoder (Sec.~\ref{sec:ace})}}}}

    $\mathcal{T} \in \mathbb{R}^{M \times \mathbb{C}}$ 
    {\footnotesize{\color{DarkBlue}{\tcp{learnable vectors}}}}

    $\mathcal{T}^{*} = \{t_{m}\}_{1 \le m \le M} \; \leftarrow \;$ Eq.~\eqref{eq:eq6} 
    {\footnotesize{\color{DarkBlue}{\tcp{intrinsic deviation vectors}}}}

    {\footnotesize{\color{DarkBlue}{\tcc{Anomaly Scoring with Projection (Sec.~\ref{sec:score})}}}}

    $\mathcal{F}^{q}_{\mathtt{den}} = \{f^{q}_{i,\mathtt{den}}\}_{i=1}^{N} \leftarrow \;$ Eq.~\eqref{eq:eq4} 
    {\footnotesize{\color{DarkBlue}{\tcp{query's denoised deviations}}}}

    $\mathbf{\tilde{f}}^{q}_{i,\mathtt{den}} = \mathtt{Proj}(f^{q}_{i,\mathtt{den}}) \leftarrow \;$ Eq.~\eqref{eq:eq8} {\footnotesize{\color{DarkBlue}{\tcp{projection}}}}

    $A^q, \mathcal{A}^{q} \leftarrow \;$ Eq.~\eqref{eq:eq9} 
    {\footnotesize{\color{DarkBlue}{\tcp{anomaly scoring}}}}

    $\mathcal{L}_{Dual}(\mathcal{F}^{a}_{\mathtt{den}}, \mathcal{T}^{*}) \leftarrow \;$ Eq.~\eqref{eq:eq7} 
    {\footnotesize{\color{DarkBlue}{\tcp{dual-branch loss}}}}

    $\mathcal{L}_{Focal}(A^{q}, \mathcal{M}^{q}) \leftarrow \;$ Eq.~\eqref{eq:eq10} 
    {\footnotesize{\color{DarkBlue}{\tcp{Focal loss}}}}

    $\mathcal{L}_{Dice}(A^{q}, \mathcal{M}^{q}) \leftarrow \;$ Eq.~\eqref{eq:eq10} 
    {\footnotesize{\color{DarkBlue}{\tcp{Dice loss}}}}

    $\mathcal{L}_{BCE}(\mathcal{A}^{q}, \mathbf{y}^{q}) \leftarrow \;$ Eq.~\eqref{eq:eq10} 
    {\footnotesize{\color{DarkBlue}{\tcp{BCE loss}}}}

    $\mathcal{L}_{train}=\mathcal{L}_{Focal}+\mathcal{L}_{Dice}+\mathcal{L}_{BCE}+\mathcal{L}_{Dual} \leftarrow \;$ Eq.~\eqref{eq:eq10} 
    {\footnotesize{\color{DarkBlue}{\tcp{training objective}}}}

    $\Phi \leftarrow \Phi - \eta \nabla \mathcal{L}_{train}(\Phi; \mathbf{E})$ 
    {\footnotesize{\color{DarkBlue}{\tcp{update model}}}}
}

{\color{LightBlue}{\tcc{\textbf{\colorbox{myblue}{Inference on Target Dataset}}}}}

$\mathcal{S}^{n}_c, \mathcal{S}^{a}_c$ 
{\footnotesize{\color{DarkBlue}{\tcp{fixed normal and abnormal references for each target dataset $c$}}}}

\For {each test query $\mathbf{x}^{q}_c$ from the target dataset $c$}{
    $\mathbf{E}=\{\mathbf{x}^{q}_c, \mathcal{S}^{n}_c, \mathcal{S}^{a}_c\}$ 
    {\footnotesize{\color{DarkBlue}{\tcp{construct inference episode}}}}

    $\mathcal{F}^{q}_c, \mathcal{F}^{n}_c, \mathcal{F}^{a}_c \; \leftarrow \; \mathcal{E}(\mathbf{x}^{q}_c), \mathcal{E}(\mathcal{S}^{n}_c), \mathcal{E}(\mathcal{S}^{a}_c)$ 
    {\footnotesize{\color{DarkBlue}{\tcp{extracting features}}}}

    $\mathcal{T}^{*} = \{t_{m}\}_{1 \le m \le M} $ 
    {\footnotesize{\color{DarkBlue}{\tcp{learned intrinsic deviation vectors}}}}

    $\mathcal{F}^{q}_{c,\mathtt{den}} \leftarrow \;$ Eq.~\eqref{eq:eq4} 
    {\footnotesize{\color{DarkBlue}{\tcp{query's denoised deviations}}}}

    $\mathbf{\tilde{f}}^{q}_{c,i,\mathtt{den}} \leftarrow \;$ Eq.~\eqref{eq:eq8} 
    {\footnotesize{\color{DarkBlue}{\tcp{projection}}}}

    $A^q_c \; \leftarrow \;$ Eq.~\eqref{eq:eq9} 
    {\footnotesize{\color{DarkBlue}{\tcp{pixel-level anomaly score}}}}

    $\mathcal{A}^q_c \; \leftarrow \;$ Eq.~\eqref{eq:eq9} 
    {\footnotesize{\color{DarkBlue}{\tcp{image-level anomaly score}}}}
}
\end{algorithm}

\section{Generalization Analysis of IDEAL} \label{sec:bound}
In this section, as in standard theoretical treatments of domain adaptation, we analyze the generalization of IDEAL in two steps because target data are unavailable during training. 
We begin with a generalization bound on source datasets (Sec.~\ref{src_bound}), and subsequently extend this bound to the target dataset (Sec.~\ref{tar_bound}) based on the multi-source domain adaptation theory developed in \cite{mansour2008domain,ben2006analysis,ben2010theory,fang2014domain}.

\subsection{Preliminaries}
Here, we introduce some additional variables to better represent the process of IDEAL. 
We begin by formalizing the source training process at the episode level. 
Suppose there are $K$ labeled source domains $\{S_k\}_{k=1}^K$, and $\{\mathbb{P}_{S_1},\ldots,\mathbb{P}_{S_K}\}$ denote corresponding source episode distributions. Each input episode is of the form $\mathbf{E}=\{\mathbf{x}^{q}, \mathcal{S}^{n}, \mathcal{S}^{a}\}$, where the query $\mathbf{x}^{q}$ with a ground-truth mask $\mathcal{M}^{q}$ and a label $\mathbf{y}^{q}$, $\mathcal{S}^{n}$ and $\mathcal{S}^{a}$ are the normal and abnormal reference set. Let $\mathcal{L}(\Phi;\mathbf{E})$ represent the episode-wise loss induced by IDEAL, \textit{i.e.}, the training objective in Eq.~\eqref{eq:eq10}, and rescale it to [$0,1$]. Denote the model $\Phi$ as the hypothesis from $\mathcal{H}$ and $d$ be the VC-dimension of $\mathcal{H}$. For any source weighting vector $\beta=(\beta_1,\ldots,\beta_K) \in \Delta_K$, following \cite{sun2011two,ben2010theory}, define the weighted source risk is:
\begin{equation}
    \label{eq:a1}
    \epsilon_{\beta}(\Phi) = \sum_{k=1}^{K} \beta_k \mathbb{E}_{\mathbf{E} \sim \mathbb{P}_{S_k}} [\mathcal{L}(\Phi;\mathbf{E})].
\end{equation}
Suppose that, for the $k$-th source domain, we observe $n_k$ training episodes as $\{\mathbf{E}_{k,i}\}_{i=1}^{n_k}$, and let the total number as $n=\sum_{k=1}^{K} n_k$. Thus, the corresponding empirical weighted source risk is:
\begin{equation}
    \label{eq:a2}
    \hat{\epsilon}_{\beta}(\Phi) = \sum_{k=1}^{K} \beta_k \frac{1}{n_k} \sum_{i=1}^{n_k} \mathcal{L}(\Phi;\mathbf{E}_{k,i}).
\end{equation}
Based on learned intrinsic deviation patterns of IDEAL, the episode-loss class can be decomposed as $\mathcal{G}=\mathbf{\cup}_{m=1}^{M} \mathcal{G}_{m}$, where $M$ is the number of intrinsic deviation patterns and $\mathcal{G}_{m}$ is the episode-loss subclass associated with the $m$-th learned deviation pattern. Equivalently, $\mathcal{G}$ is the episode-loss class induced by IDEAL over all deviation patterns: $\mathcal{G}=\{\mathbf{E} \mapsto \mathcal{L}(\Phi;\mathbf{E}): \Phi \in \mathcal{H}\}$.

\subsection{Assumptions}
Let $\mathcal{L}(\Phi;\mathbf{E})$ represent the episode-wise loss induced by IDEAL, and let $\mathcal{G}$ be the induced loss class, our generalization analysis follows standard settings in general domain adaptation works \cite{zhang2019bridging,ben2006analysis,zhang2020unsupervised} and is based on the following assumptions:
\begin{assumption}[Episode Sampling within Source Domains] \label{ass1}
    For any labeled source domain $S_k$, the training episodes $\{\mathbf{E}_{k,i}\}_{i=1}^{n_k}$ are drawn i.i.d. from the underlying distribution $\mathbb{P}_{S_k}$. Moreover, the episode sets from different source domains $\{S_1,\ldots,S_K\}$ are mutually independent.
\end{assumption}
\begin{assumption}[Bounded Episode Loss] \label{ass2}
There exists a constant $B>0$ such that for all $\Phi \in \mathcal{H}$ and all episodes $\mathbf{E}$, have $0 \le \mathcal{L}(\Phi;\mathbf{E}) \le B$. Without loss of generality, we normalize $B=1$.
\end{assumption}
\begin{assumption}[Hypothesis Capacity] \label{ass3}
The hypothesis class $\mathcal{H}$ has finite VC-dimension $d$. 
Consequently, the induced episode-loss class $\mathcal{G}$ has finite pseudo-dimension.
\end{assumption}
\begin{assumption}[Lipschitz Continuity] \label{ass4}
Each real-valued loss function $\mathcal{L}(\Phi;\mathbf{E})$ is $L$-Lipschitz continuous with respect to the model output.
\end{assumption}

\subsection{Proof of Source-Side Generalization Bound} \label{src_bound}
\begin{theorem}[Source-Side Generalization Bound of IDEAL] \label{the1}
    Consider there are $K$ labeled source domains. Let $\mathbb{P}_{S_1},\ldots,\mathbb{P}_{S_K}$ be the corresponding source episode distributions. Denote the model $\Phi$ as the hypothesis from $\mathcal{H}$ and $d$ be the VC-dimension of $\mathcal{H}$. The total number of episodes over all $K$ source domains is $n$. Then, with probability at least $1-\delta$,
    \begin{equation}
        \label{eq:a3}
        \sup_{\Phi \in \mathcal{H}} | \epsilon_{\beta}(\Phi) - \hat{\epsilon}_{\beta}(\Phi) | \le 
        \sqrt{\frac{1}{2} \left( \sum_{k=1}^{K} \frac{\beta_k^2}{n_k} \right) \log \frac{M}{\delta}} + 
        2\sum_{k=1}^{K} \beta_k \sqrt{\frac{d}{n_k} \log \frac{e n_k}{d}}.
    \end{equation}
\end{theorem}
Theorem~\ref{the1} indicates that the source-side generalization bound of IDEAL is associated with the number $M$ of intrinsic deviation patterns and the weighting vector $\beta$ across $K$ source domains. 
The bound becomes tighter by using appropriate $\Phi \in \mathcal{H}$ and the number $M$ of intrinsic deviation patterns.
\begin{proof}[Proof of Theorem~\ref{the1}] \label{proof1}
For each learned intrinsic deviation pattern of IDEAL indexed by $m$, based on the episode-loss subclass $\mathcal{G}_{m}$, define the corresponding deviation quantity as:
\begin{equation}
    \label{eq:a4}
    \Gamma_{m} = \sup_{g \in \mathcal{G}_m} \left| \sum_{k=1}^{K} \beta_{k} \mathbb{E}_{\mathbf{E} \sim \mathbb{P}_{S_k}}[g(\mathbf{E})] - 
    \sum_{k=1}^{K} \beta_{k} \frac{1}{n_k} \sum_{i=1}^{n_k} g(\mathbf{E}_{k,i}) \right|,
\end{equation}
since the loss class $\mathcal{G}=\mathbf{\cup}_{m=1}^{M} \mathcal{G}_{m}$ is the union of the subclasses, we can further get:
\begin{equation}
    \label{eq:a5}
    \sup_{\Phi \in \mathcal{H}} | \epsilon_{\beta}(\Phi) - \hat{\epsilon}_{\beta}(\Phi) | \le \max_{1 \le m \le M} \Gamma_{m}.
\end{equation}
We first bound the deviation term under Assumption~\ref{ass1} and Assumption~\ref{ass2}. Replacing a single episode $\mathbf{E}_{k,i}$ changes the weighted empirical operator by at most $\beta_k / n_k$. Therefore, by McDiarmid's inequality, for any fixed $m$ and any variable $\xi > 0$, we have:
\begin{equation}
    \label{eq:a6}
    \mathrm{Pr}(\Gamma_{m} - \mathbb{E}[\Gamma_{m}] > \xi) \le \exp \left(-\frac{2\xi^{2}}{\sum_{k=1}^{K} \beta_k^{2} / n_k}\right),
\end{equation}
which means that with probability at least $1-\exp(-\frac{2\xi^{2}}{\sum_{k=1}^{K} \beta_k^{2} / n_k})$,
\begin{equation}
    \label{eq:a62}
    \Gamma_{m} - \mathbb{E}[\Gamma_{m}] \le \xi.
\end{equation}
Next, applying the union bound over the $M$ intrinsic deviation patterns, we can get:
\begin{equation}
    \label{eq:a7}
    \mathrm{Pr}\left(\max_{1 \le m \le M}\Gamma_{m} > \max_{1 \le m \le M}\mathbb{E}[\Gamma_{m}] + \xi \right) \le 
    M \exp \left(-\frac{2\xi^{2}}{\sum_{k=1}^{K} \beta_k^{2} / n_k}\right).
\end{equation}
Setting the right-hand side of Eq.~\eqref{eq:a7} to $\delta$, \textit{i.e.}, $\delta=M\exp(-\frac{2\xi^{2}}{\sum_{k=1}^{K} \beta_k^{2} / n_k})$. 
Then, we can obtain that with probability at least $1-\delta$,
\begin{equation}
    \label{eq:a8}
    \max_{1 \le m \le M} \Gamma_{m} \le \max_{1 \le m \le M} \mathbb{E}[\Gamma_{m}] + 
    \sqrt{\frac{1}{2} \left(\sum_{k=1}^{K} \frac{\beta_k^2}{n_k}\right) \log \frac{M}{\delta}}.
\end{equation}
Considering the expectation term of Eq.~\eqref{eq:a8}, by using Rademacher complexity, for each subclass $\mathcal{G}_m$, the following holds:
\begin{equation}
    \label{eq:a9}
    \mathbb{E}[\Gamma_{m}] \le 2 \sum_{k=1}^{K} \beta_k \mathbb{E}_{\sigma} 
    \left[ \sup_{g \in \mathcal{G}_m} \frac{1}{n_k} \sum_{i=1}^{n_k} \sigma_{i} g(\mathbf{E}_{k,i}) \right],
\end{equation}
where $\sigma=(\sigma_1,\ldots,\sigma_{n_{k}})$ represent the Rademacher random variables. Eq.~\eqref{eq:a9} denotes empirical Rademacher complexity of $\mathcal{G}_m$ on the source domain $S_k$. 
By Assumption~\ref{ass3}, the induced loss class has finite capacity controlled by $d$. Hence, using the standard complexity bound, we have:
\begin{equation}
    \label{eq:a10}
    \mathbb{E}_{\sigma} \left[ \sup_{g \in \mathcal{G}_m} \frac{1}{n_k} \sum_{i=1}^{n_k} \sigma_{i} g(\mathbf{E}_{k,i}) \right] \le 
    \sqrt{\frac{d}{n_k} \log \frac{e n_k}{d}}.
\end{equation}
Based on Eq.~\eqref{eq:a9} and Eq.~\eqref{eq:a10}, we can get:
\begin{equation}
    \label{eq:a11}
    \mathbb{E}[\Gamma_{m}] \le 2 \sum_{k=1}^{K} \beta_k \sqrt{\frac{d}{n_k} \log \frac{e n_k}{d}},
\end{equation}
since this upper bound is independent of $m$, we can further get:
\begin{equation}
    \label{eq:a12}
    \max_{1 \le m \le M} \mathbb{E}[\Gamma_{m}] \le 2 \sum_{k=1}^{K} \beta_k \sqrt{\frac{d}{n_k} \log \frac{e n_k}{d}}.
\end{equation}
Combining the concentration bound Eq.~\eqref{eq:a8} and the capacity bound Eq.~\eqref{eq:a12}, with probability at least $1-\delta$, we have:
\begin{equation}
    \label{eq:a13}
    \sup_{\Phi \in \mathcal{H}} | \epsilon_{\beta}(\Phi) - \hat{\epsilon}_{\beta}(\Phi) | \le 
    \sqrt{\frac{1}{2} \left( \sum_{k=1}^{K} \frac{\beta_k^2}{n_k} \right) \log \frac{M}{\delta}} + 
    2\sum_{k=1}^{K} \beta_k \sqrt{\frac{d}{n_k} \log \frac{e n_k}{d}}.
\end{equation}
Eq.~\eqref{eq:a13} indicates that the source-side generalization bound of IDEAL is associated with the number $M$ of intrinsic deviation patterns and the weighting vector $\beta$ across $K$ source domains.
\end{proof}

\subsection{Proof of Source-Target Generalization Bound} \label{tar_bound}
Since the target data are unavailable during training, we further extend this source-side generalization bound of IDEAL in Eq.~\eqref{eq:a13} to the target domain via multi-source domain adaptation \cite{mansour2008domain}.

Let $\mathbb{P}_{T}$ denote the target episode distribution, and each target episode $\mathbf{E}=\{\mathbf{x}^{q}_{T}, \mathcal{S}^{n}_{T}, \mathcal{S}^{a}_{T}\}$ has the same form as the source episode, where $\mathbf{x}^{q}_{T}$ is the target sample, and $\mathcal{S}^{n}_{T}$ and $\mathcal{S}^{a}_{T}$ are the corresponding normal and abnormal reference sets from the target domain $T$, respectively. 
Further, define the target risk of IDEAL by:
\begin{equation}
    \label{eq:a15}
    \epsilon_{T}(\Phi) = \mathbb{E}_{\mathbf{E} \sim \mathbb{P}_{T}} [\mathcal{L}(\Phi;\mathbf{E})].
\end{equation}
Based on the source weighting vector $\beta=(\beta_1,\ldots,\beta_K) \in \Delta_K$, define the weighted source mixture distribution as:
\begin{equation}
    \label{eq:a16}
    \mathbb{P}_{\beta} = \sum_{k=1}^{K} \beta_{k} \mathbb{P}_{S_k},
\end{equation}
where $\mathbb{P}_{S_k}$ means the source episode distribution on the source domain $S_k$.

Following the assumption that the target distribution is some mixture of the source distributions \cite{mansour2008domain}, let $\ell$ be a bounded symmetric loss that measures the disagreement between two hypotheses on an episode. 
For any episode distribution $\mathbb{P}$, define the pairwise expected disagreement as:
\begin{equation}
    \label{eq:a17}
    \epsilon_{\mathbb{P}}^{\ell}(\Phi, \Phi^{\prime}) = \mathbb{E}_{\mathbf{E} \sim \mathbb{P}}[\ell(\Phi(\mathbf{E}), \Phi^{\prime}(\mathbf{E}))].
\end{equation}
Then, the discrepancy between the target distribution and the weighted source mixture is:
\begin{equation}
    \label{eq:a18}
    \mathrm{D}_{\mathcal{H}, \ell}(\mathbb{P}_{T}, \mathbb{P}_{\beta}) = \sup_{\Phi, \Phi^{\prime} \in \mathcal{H}} 
    \left|\epsilon_{\mathbb{P}_T}^{\ell}(\Phi, \Phi^{\prime})-\epsilon_{\mathbb{P}_{\beta}}^{\ell}(\Phi, \Phi^{\prime})\right|.
\end{equation}
Further, define the weighted joint optimal risk as:
\begin{equation}
    \label{eq:a19}
    \Omega_{\beta} = \inf_{\Phi \in \mathcal{H}}(\epsilon_{T}(\Phi) + \epsilon_{\beta}(\Phi)),
\end{equation}
where $\epsilon_{\beta}(\Phi)$ in Eq.~\eqref{eq:a1} means the weighted source risk.

For convenience, let the source-side generalization bound in Theorem~\ref{the1} be denoted by:
\begin{equation}
    \label{eq:a20}
    \mathfrak{B}_{src}(\beta, n, M, d, \delta) = 
    \sqrt{\frac{1}{2} \left( \sum_{k=1}^{K} \frac{\beta_k^2}{n_k} \right) \log \frac{M}{\delta}} + 
    2\sum_{k=1}^{K} \beta_k \sqrt{\frac{d}{n_k} \log \frac{e n_k}{d}}.
\end{equation}

\begin{theorem}[Source-Target Generalization Bound of IDEAL] \label{the2}
    Let the conditions of Theorem~\ref{the1} and Assumptions \ref{ass1} to \ref{ass4} hold. Then, for any $\beta \in \Delta_K$, with probability at least $1-\delta$, the following holds for any $\Phi \in \mathcal{H}$,
    \begin{equation}
        \label{eq:a21}
        \epsilon_{T}(\Phi) \le \hat{\epsilon}_{\beta}(\Phi) + \mathfrak{B}_{src}(\beta, n, M, d, \delta) + 
        \mathrm{D}_{\mathcal{H}, \ell}(\mathbb{P}_{T}, \mathbb{P}_{\beta}) + \Omega_{\beta},
    \end{equation}
    where $\hat{\epsilon}_{\beta}(\Phi)$ as the empirical source risk in Eq.~\eqref{eq:a2}, $\mathfrak{B}_{src}$ as the source-side bound in Eq.~\eqref{eq:a20}, $\mathrm{D}_{\mathcal{H}, \ell}(\mathbb{P}_{T}, \mathbb{P}_{\beta})$ as the discrepancy between the target distribution $\mathbb{P}_{T}$ and the weighted source mixture $\mathbb{P}_{\beta}$ in Eq.~\eqref{eq:a18}, and $\Omega_{\beta}$ as the weighted joint optimal risk in Eq.~\eqref{eq:a19}.

    Equivalently, by expanding $\mathfrak{B}_{src}$, we can further get:
    \begin{equation}
        \label{eq:a22}
        \epsilon_{T}(\Phi) \le \hat{\epsilon}_{\beta}(\Phi) + 
        \sqrt{\frac{1}{2} \left( \sum_{k=1}^{K} \frac{\beta_k^2}{n_k} \right) \log \frac{M}{\delta}} + 
        2\sum_{k=1}^{K} \beta_k \sqrt{\frac{d}{n_k} \log \frac{e n_k}{d}} + 
        \mathrm{D}_{\mathcal{H}, \ell}(\mathbb{P}_{T}, \mathbb{P}_{\beta}) + \Omega_{\beta}.
    \end{equation}
\end{theorem}
Theorem~\ref{the2} indicates that the target-domain generalization bound of IDEAL is associated with the empirical weighted source risk $\hat{\epsilon}_{\beta}(\Phi)$, the source-side estimation bound $\mathfrak{B}_{src}$, the discrepancy between the target distribution and weighted source mixture, and the weighted joint optimal risk. Theorem~\ref{the2} further reveals a capacity trade-off controlled by the number $M$ of intrinsic deviation patterns. Increasing $M$ can improve the coverage of diverse intrinsic deviation vectors and potentially reduce the empirical and approximation errors, but it also enlarges the induced loss class and increases the source-side estimation term through $\log M$. Therefore, $M$ should be properly chosen to balance expressiveness and estimation complexity.

Since the target data are unavailable during training, neither the $\mathrm{D}_{\mathcal{H}, \ell}(\mathbb{P}_{T}, \mathbb{P}_{\beta})$ nor $\Omega_{\beta}$ is directly estimable from the source data. To further simplify the target-side generalization bound, we consider a stronger transfer setting under which the discrepancy term and the weighted joint optimal risk can be tightly controlled. The key is that the target episode distribution $\mathbb{P}_{T}$ is assumed to lie in the weighted convex hull of the source episode distributions, and that the hypothesis class contains $\Phi^{*} \in \mathcal{H}$ that performs well on both the target domain and the weighted source mixture.

We now impose the following two assumptions:
\begin{assumption}[Source-Mixture Target] \label{ass5}
There exists a source weighting vector $\beta \in \Delta_K$ such that the target episode distribution $\mathbb{P}_{T}$ is exactly the weighted mixture of the source episode distributions:
    \begin{equation}
        \label{eq:a23}
        \mathbb{P}_{T} = \mathbb{P}_{\beta} = \sum_{k=1}^{K} \beta_k \mathbb{P}_{S_k}.
    \end{equation}
\end{assumption}
\begin{assumption}[Optimal Hypothesis] \label{ass6}
There exists a hypothesis $\Phi^{*} \in \mathcal{H}$ such that:
    \begin{equation}
        \label{eq:a24}
        \epsilon_{T}(\Phi^{*}) + \epsilon_{\beta}(\Phi^{*}) \le \Lambda_{\mathbb{P}_{T},\mathbb{P}_{\beta}},
    \end{equation}
    where $\Lambda_{\mathbb{P}_{T},\mathbb{P}_{\beta}} \ge 0$ means the joint approximation error. 
    The smaller $\Lambda_{\mathbb{P}_{T},\mathbb{P}_{\beta}}$ is, the better aligned the target domain $\mathbb{P}_{T}$ is with the weighted source mixture $\mathbb{P}_{\beta}$ under the hypothesis class $\mathcal{H}$.
\end{assumption}
\begin{corollary}[Simplified Source-Target Generalization Bound for IDEAL] \label{cor1}
    Let the conditions of Theorem~\ref{the2} and Assumptions \ref{ass5} to \ref{ass6} hold. 
    With probability at least $1-\delta$, for any $\Phi \in \mathcal{H}$,
    \begin{equation}
        \label{eq:a25}
        \epsilon_{T}(\Phi) \le \hat{\epsilon}_{\beta}(\Phi) + 
        \sqrt{\frac{1}{2} \left( \sum_{k=1}^{K} \frac{\beta_k^2}{n_k} \right) \log \frac{M}{\delta}} + 
        2\sum_{k=1}^{K} \beta_k \sqrt{\frac{d}{n_k} \log \frac{e n_k}{d}} + \Lambda_{\mathbb{P}_{T},\mathbb{P}_{\beta}}.
    \end{equation}
\end{corollary}
Corollary~\ref{cor1} shows that the target-domain error of IDEAL is upper bounded by the empirical weighted source risk, the source-side estimation term, and a joint approximation error $\Lambda_{\mathbb{P}_{T},\mathbb{P}_{\beta}}$. In particular, when $\Lambda_{\mathbb{P}_{T},\mathbb{P}_{\beta}}$ is small, the dominant estimable term becomes the source-side bound derived in Theorem~\ref{the1}.

\begin{proof}[Proof of Theorem~\ref{the2} and Corollary~\ref{cor1}] \label{proof2}
    For any source-side weighting vector $\beta \in \Delta_K$, define the weighted source mixture distribution as,
    \begin{equation}
        \label{eq:pro1}
        \mathbb{P}_{\beta} = \sum_{k=1}^{K} \beta_k \mathbb{P}_{S_k}.
    \end{equation}
    Following the multi-source domain adaptation \cite{mansour2008domain}, for any $\Phi \in \mathcal{H}$, we have:
    \begin{equation}
        \label{eq:a26}
        \epsilon_{T}(\Phi) \le \epsilon_{\beta}(\Phi) + \mathrm{D}_{\mathcal{H}, \ell}(\mathbb{P}_{T}, \mathbb{P}_{\beta}) + \Omega_{\beta},
    \end{equation}
    where $\mathrm{D}_{\mathcal{H}, \ell}(\mathbb{P}_{T}, \mathbb{P}_{\beta})$ as the discrepancy between the target distribution $\mathbb{P}_{T}$ and the weighted source mixture $\mathbb{P}_{\beta}$ in Eq.~\eqref{eq:a18}, and $\Omega_{\beta}$ as the weighted joint optimal risk in Eq.~\eqref{eq:a19}.

    Eq.~\eqref{eq:a26} remains to upper bound the weighted source risk $\epsilon_{\beta}(\Phi)$. 
    By Theorem~\ref{the1}, with probability at least $1-\delta$, the following source-side generalization bound holds,
    \begin{equation}
        \label{eq:pro2}
        \sup_{\Phi \in \mathcal{H}} | \epsilon_{\beta}(\Phi) - \hat{\epsilon}_{\beta}(\Phi) | \le 
        \mathfrak{B}_{src}(\beta, n, M, d, \delta),
    \end{equation}
    where $\mathfrak{B}_{src}(\beta, n, M, d, \delta)$ is defined in Eq.~\eqref{eq:a20}.
    
    Since the above inequality holds over all $\Phi \in \mathcal{H}$, with probability at least $1-\delta$, it follows that,
    \begin{equation}
        \label{eq:pro3}
        \epsilon_{\beta}(\Phi) \le \hat{\epsilon}_{\beta}(\Phi) + \mathfrak{B}_{src}(\beta, n, M, d, \delta).
    \end{equation}

    Substituting this bound into the multi-source adaptation inequality yields,
    \begin{equation}
        \label{eq:pro4}
        \epsilon_{T}(\Phi) \le \hat{\epsilon}_{\beta}(\Phi) + \mathfrak{B}_{src}(\beta, n, M, d, \delta) + 
        \mathrm{D}_{\mathcal{H}, \ell}(\mathbb{P}_{T}, \mathbb{P}_{\beta}) + \Omega_{\beta}.
    \end{equation}
    Then, by expanding $\mathfrak{B}_{src}(\beta, n, M, d, \delta)$, we have:
    \begin{equation}
        \label{eq:pro5}
        \epsilon_{T}(\Phi) \le \hat{\epsilon}_{\beta}(\Phi) + 
        \sqrt{\frac{1}{2} \left( \sum_{k=1}^{K} \frac{\beta_k^2}{n_k} \right) \log \frac{M}{\delta}} + 
        2\sum_{k=1}^{K} \beta_k \sqrt{\frac{d}{n_k} \log \frac{e n_k}{d}} + 
        \mathrm{D}_{\mathcal{H}, \ell}(\mathbb{P}_{T}, \mathbb{P}_{\beta}) + \Omega_{\beta}.
    \end{equation}
    We have completed the proof of Theorem~\ref{the2}.
    
    Next, under the Assumption~\ref{ass5}, $\mathbb{P}_{T}$ coincides with the weighted source mixture $\mathbb{P}_{\beta}$,
    \begin{equation}
        \label{eq:a27}
        \mathbb{P}_{T}=\mathbb{P}_{\beta}=\sum_{k=1}^{K} \beta_k \mathbb{P}_{S_k}.
    \end{equation}
    Therefore, we can get:
    \begin{equation}
        \label{eq:pro6}
        \mathrm{D}_{\mathcal{H}, \ell}(\mathbb{P}_{T}, \mathbb{P}_{\beta}) = 0.
    \end{equation}
    Next, by the definition of $\Omega_{\beta} = \inf_{\Phi \in \mathcal{H}}(\epsilon_{T}(\Phi) + \epsilon_{\beta}(\Phi))$ in Eq.~\eqref{eq:a19}, 
    since the Assumption~\ref{ass6} guarantees the existence of $\Phi^{*} \in \mathcal{H}$ such that,
    \begin{equation}
        \label{eq:a28}
        \epsilon_{T}(\Phi^{*}) + \epsilon_{\beta}(\Phi^{*}) \le \Lambda_{\mathbb{P}_{T},\mathbb{P}_{\beta}},
    \end{equation}
    which follows that $\Omega_{\beta} \le \Lambda_{\mathbb{P}_{T},\mathbb{P}_{\beta}}$.

    Substituting the above two inequalities into Eq.~\eqref{eq:a26}, we can get:
    \begin{equation}
        \label{eq:a29}
        \epsilon_{T}(\Phi) \le \epsilon_{\beta}(\Phi) + \Lambda_{\mathbb{P}_{T},\mathbb{P}_{\beta}}.
    \end{equation}
    On the other hand, by Theorem~\ref{the1}, with probability at least $1-\delta$, the source-side generalization bound holds for any $\Phi \in \mathcal{H}$,
    \begin{equation}
        \label{eq:a30}
        \epsilon_{\beta}(\Phi) \le \hat{\epsilon}_{\beta}(\Phi) + \mathfrak{B}_{src}(\beta, n, M, d, \delta).
    \end{equation}
    Combining the above two inequalities, we have:
    \begin{equation}
        \label{eq:a31}
        \epsilon_{T}(\Phi)\le \hat{\epsilon}_{\beta}(\Phi)+\mathfrak{B}_{src}(\beta, n, M, d, \delta)+\Lambda_{\mathbb{P}_{T},\mathbb{P}_{\beta}}.
    \end{equation}
    Equivalently,
    \begin{equation}
        \label{eq:eq_cor1}
        \epsilon_{T}(\Phi) \le \hat{\epsilon}_{\beta}(\Phi) + 
        \sqrt{\frac{1}{2} \left( \sum_{k=1}^{K} \frac{\beta_k^2}{n_k} \right) \log \frac{M}{\delta}} + 
        2\sum_{k=1}^{K} \beta_k \sqrt{\frac{d}{n_k} \log \frac{e n_k}{d}} + \Lambda_{\mathbb{P}_{T},\mathbb{P}_{\beta}},
    \end{equation}
    where $\mathfrak{B}_{src}$ means the source-side generalization bound in Eq.~\eqref{eq:a20}, and $\Lambda_{\mathbb{P}_{T},\mathbb{P}_{\beta}}$ means the joint approximation error between the target episode distribution and the weighted source mixture.

    A particularly clean special case arises when the hypothesis class contains a predictor with zero joint error on the target domain and the weighted source mixture.
    
    In this case, we have $\Lambda_{\mathbb{P}_{T},\mathbb{P}_{\beta}} = 0$, for any $\Phi \in \mathcal{H}$, the Corollary~\ref{cor1} reduces to,
    \begin{equation}
        \label{eq:a33}
        \epsilon_{T}(\Phi) \le \hat{\epsilon}_{\beta}(\Phi) + 
        \sqrt{\frac{1}{2} \left( \sum_{k=1}^{K} \frac{\beta_k^2}{n_k} \right) \log \frac{M}{\delta}} + 
        2\sum_{k=1}^{K} \beta_k \sqrt{\frac{d}{n_k} \log \frac{e n_k}{d}},
    \end{equation}
    where the source-target error is completely reduced to the empirical weighted source risk plus the source-side estimation term. Therefore, under this ideal setting, Eq.~\eqref{eq:a33} bounds the source-side term sufficiently to control the target-domain generalization performance.
\end{proof}

\section{Experimental Details} \label{sec:edetails}
In this section, we first describe the eight anomaly detection datasets utilized in our experiments (Sec.~\ref{sec:addata}), followed by the competing baselines under cross-dataset and in-dataset evaluation protocols (Sec.~\ref{sec:admethod}). We then detail the episode construction strategy for discriminative FSAD during training and inference (Sec.~\ref{sec:ourepisode}). Finally, we provide implementation details of IDEAL (Sec.~\ref{sec:ourdetail}).

\subsection{Anomaly Detection Datasets} \label{sec:addata}
We conduct extensive experiments on eight real-world Anomaly Detection (AD) datasets, including five industrial defect inspection datasets: MVTecAD \cite{bergmann2019mvtec}, VisA \cite{zou2022spot}, AITEX \cite{silvestre2019public}, MPDD \cite{jezek2021deep}, BTAD \cite{mishra2021vt}; and three medical image datasets: BraTS \cite{menze2014multimodal} (for brain tumor segmentation), Liver \cite{bao2024bmad} (for liver tumor segmentation), RESC \cite{hu2019automated} (for retinal edema segmentation). 
We then describe each anomaly detection dataset in turn below.

\textbf{MVTecAD}~\cite{bergmann2019mvtec} is a widely used benchmark for evaluating anomaly detection methods in industrial inspection applications. This dataset includes over $5000$ high-resolution images spanning $15$ object and texture categories. For each category, the training set consists exclusively of anomaly-free images, while the test set includes both defective and defect-free samples.

\textbf{VisA}~\cite{zou2022spot} is a large-scale industrial anomaly detection dataset comprising $10821$ high-resolution color images, including $9621$ normal samples and $1200$ anomalous samples. This dataset covers $12$ object categories across $3$ domains and provides both image-level and pixel-level annotations. The anomalous samples exhibit diverse defect types, including surface defects such as scratches, dents, color spots, and cracks, as well as structural defects such as misplacement and missing parts.

\textbf{AITEX}~\cite{silvestre2019public} is a textile fabric anomaly detection dataset comprising $245$ images from $7$ fabric types, including $140$ defect-free images ($20$ images for each fabric type) and $105$ images with various defects. Following the evaluation protocol, we use its test set for evaluation.

\textbf{MPDD}~\cite{jezek2021deep} is a metal-part defect detection dataset designed to benchmark visual anomaly detection methods under complex industrial conditions, which contains over $1000$ images of metal parts with pixel-precise defect annotation masks. It contains $1346$ images across $6$ metal-part categories, with $888$ normal images for training and $458$ test images ($176$ normal and $282$ anomalous samples).

\textbf{BTAD}~\cite{mishra2021vt} is a industrial anomaly detection dataset introduced in VT-ADL. It contains $2830$ images from $3$ industrial products, covering both body and surface defects. The training set consists of $1799$ normal images, while the test set contains $451$ normal images and $290$ anomalous images.

\textbf{BraTS}~\cite{menze2014multimodal} is a large-scale medical image segmentation dataset for multimodal brain tumor analysis. It contains $2000$ glioma cases, corresponding to $8000$ mpMRI scans across four modalities: T1, T1Gd, T2, and FLAIR. Each case provides expert-annotated tumor sub-regions, enabling evaluation of abnormal tissue detection and localization in medical imaging scenarios.

\textbf{Liver}~\cite{bao2024bmad} is a medical anomaly detection dataset reorganized from BMAD for liver tumor detection and localization. The dataset contains $3201$ axial CT slices, including $1542$ training slices, $1493$ test slices, and $166$ validation slices, with a spatial resolution of $512 \times 512$. 
Pixel-level segmentation masks are provided for anomaly localization evaluation.

\textbf{RESC}~\cite{hu2019automated} is a retinal OCT dataset originally developed for retinal edema segmentation. Following BMAD, we use its reorganized anomaly detection and localization split, where retinal edema regions are treated as anomalies. 
It contains $6217$ OCT images with a spatial resolution of $512 \times 1024$, including $4297$ training images, $1805$ test images, and $115$ validation images. Pixel-level segmentation masks are provided for edema regions, enabling evaluation of anomaly detection.

\subsection{Competing Baselines} \label{sec:admethod}
We compare IDEAL with two groups of closely related methods: (1)
\textit{Generalist Methods}, \ie, detectors are training-free or trained once on auxiliary data: based on normal-only references (InCTRL \cite{zhu2024toward}, ResAD \cite{yao2024resad}, WinCLIP \cite{jeong2023winclip}, FoundAD \cite{zhai2026foundation}), and normal and anomalous references (NAGL \cite{wang2025normal}); (2) \textit{Specalist Methods}, 
 \ie, detectors require dataset-specific training: based on normal-only references (RegAD \cite{huang2022registration}, PromptAD \cite{li2024promptad}), and based on full-shot normal data and abnormal-only references: DRA \cite{ding2022catching} and AHL \cite{zhu2024anomaly}.

The crucial configurations of different AD baselines as follows. 
1) \textbf{WinCLIP} \cite{jeong2023winclip}: We follow the original setting with the LAION-400M pre-trained CLIP ViT-B/16+ backbone. The dense window features are extracted with stride $1$ over ViT patch embeddings, and multi-scale predictions from $2 \times 2$ patch windows, $3 \times 3$ patch windows, and the image-level token are aggregated by harmonic averaging. 
2) \textbf{InCTRL} \cite{zhu2024toward}: We use the pre-trained ViT-B/16+ as backbone with the initial learning rate as $1e^{-3}$ by default. The prompt ensemble contains $231$ text prompts per object category, including $147$ normal prompts and $84$ anomalous prompts. The training epochs is set to $10$ with batch size of $4$. 
3) \textbf{ResAD} \cite{yao2024resad}: The number of the NF model's layers is set to $8$. The total training epochs are set as $100$ with a batch size of $16$. The learning rate is set to $1e^{-5}$ initially and dropped by $0.1$ after $75$ epochs. 
4) \textbf{FoundAD} \cite{zhai2026foundation}: The pre-trained DINOv3 ViT-B is adopted as the frozen visual encoder. All the training and testing images are resized to $512 \times 512$. The nonlinear manifold projector is configured with a prediction depth of $6$ and an embedding dimension of $384$. The model is trained for $100$ epochs with a batch size of $8$, an initial learning rate of $1e^{-3}$, and a weight decay of $1e^{-4}$. The learning rate is kept constant by default. 
5) \textbf{NAGL} \cite{wang2025normal}: We follow the official implementation and use the pre-trained DINOv2 ViT-S/14 as the frozen visual backbone. All the training and testing images are resized to $448 \times 448$. We tune the number of learnable proxies $\{10,25,30,40\}$ and set it as $25$. The model is trained for $20$ epochs with a batch size of $16$. The initial learning rate is set to $1e^{-5}$ and optimized by AdamW. The learning rate is decayed by $0.1$ at the $10$-th and $15$-th epochs. 
6) \textbf{RegAD} \cite{huang2022registration}: We use ResNet-18 as the backbone, followed by a convolution-based encoder and predictor. The input images are resized to $224 \times 224$. The model is trained for $50$ epochs with a batch size of $16$ using momentum SGD. The initial learning rate is set to $1e^{-4}$ and decayed by a single-cycle cosine schedule. 
7) \textbf{PromptAD} \cite{li2024promptad}: We follow the official implementation and use ViT-B/16+ as the visual backbone. The input resolution is set to $240$ for the shorter edge after bicubic resizing. The loss weight $\lambda$ of PromptAD is set to $0.001$. The prompt learning optimizer uses an initial learning rate of $0.002$, momentum of $0.9$, and weight decay of $5e^{-4}$. The $3$-rd and $8$-th layer features are used for visual feature memory in PromptAD. 
8) \textbf{DRA} \cite{ding2022catching}: We follow the official implementation and use ImageNet pre-trained ResNet-18 as the backbone. The extracted $512$-dimensional feature map is fed into the abnormality and normality learning heads. The patch-wise classifier is implemented by a $1 \times 1$ convolutional layer. The model is trained for $30$ epochs with $20$ iterations per epoch and a batch size of $48$. Adam is used for optimization with an initial learning rate of $1e^{-3}$ and a weight decay of $1e^{-2}$. 
9) \textbf{AHL} \cite{zhu2024anomaly}: We follow the official setting and use DRA as the base OSAD detector. Normal samples are first partitioned into $3$ clusters by K-means, and $T=6$ anomaly distribution subsets are generated by combining randomly selected normal clusters with labeled anomaly examples. Adam is used as the optimizer, and the learning rate is set to $2e^{-4}$ for heterogeneous base models and $2e^{-3}$ for the unified anomaly detection model. The score estimator is implemented with a two-layer LSTM with hidden dimension $6$, followed by a fully connected layer with $12$ hidden nodes.

\begin{figure*}[t!]
    \centering
    \includegraphics[width=1.0\linewidth]{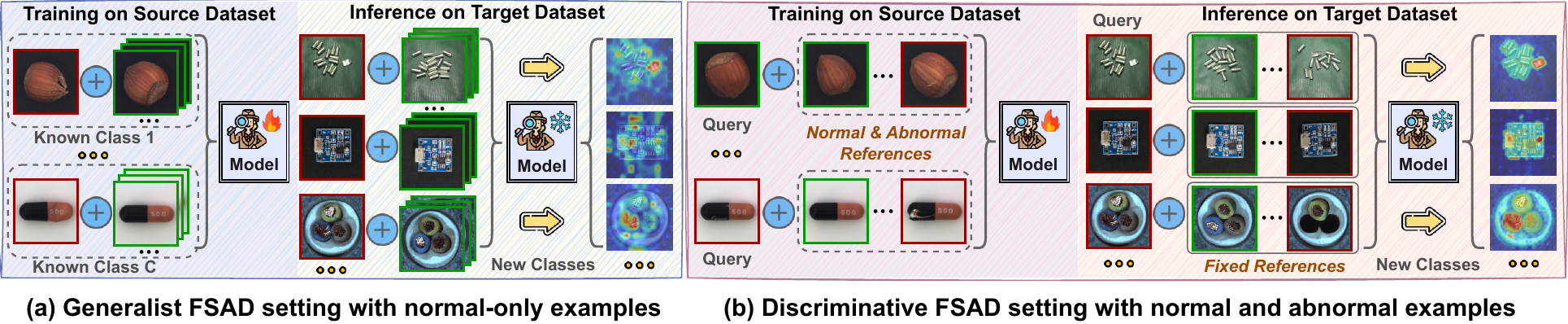}
    \caption{Illustration of \textbf{(a)} generalist FSAD setting (\textit{e.g.}, used in InCTRL \cite{zhu2024toward}); \textbf{(b)} discriminative FSAD setting with few-shot normal and abnormal references, our model effectively leverages the discriminative clues of abnormal references relative to normal references to enhance detection.}
    \label{fig:fig_ourset}
\end{figure*}

\subsection{Episode Preparation} \label{sec:ourepisode}
The proposed IDEAL is trained and evaluated in an episodic manner, as shown in Fig.~\ref{fig:fig_ourset}. 
In each episode, a limited reference set containing normal and anomalous reference samples is provided, and the model is required to detect anomalies in the corresponding query samples.

\paragraph{Training Episode.}
IDEAL is trained on the auxiliary dataset, where each training episode is constructed to simulate the few-shot inference scenario, consisting of a query image together with a few normal and anomalous reference images sampled from the auxiliary dataset. Specifically, we first randomly select $600$ samples by default from the entire test set as training queries, including both normal and abnormal samples. Based on the query, we randomly sample normal reference images from the corresponding training set and anomalous reference images from the corresponding test set, where all anomalous references are restricted to a single anomaly type. 
During training, each episode is constructed by grouping a query with its corresponding normal and anomalous reference samples, and the model is optimized episode by episode.
\paragraph{Inference Episode.}
IDEAL is evaluated on the target dataset without further retraining/tuning. For each target dataset, we pre-construct a fixed reference set containing normal and anomalous reference images, using the same sampling strategy as in training. These references are kept fixed throughout the evaluation. When an arbitrary query image from the target dataset, it is combined with the corresponding fixed normal and anomalous references to form an inference episode, which is then fed into our model for anomaly detection.

\subsection{Implementation Details} \label{sec:ourdetail}
By default, we adopt the public pre-trained ViT/S-14 as the feature extractor in our experiments. 
We freeze the parameters of the feature extractor and only update the parameters of the proposed IDE component. 
The IDE component is implemented as a cross-attention block followed by a feed-forward MLP block. The cross-attention block has an embedding dimension of $384$, $12$ attention heads, an attention dropout rate of $0.1$, and sinusoidal positional encoding. The MLP consists of two linear layers with input/output dimension $384$ and hidden dimension $1536$, with GELU used as the activation function. For the learnable vectors $\mathcal{T}$ in Eq.~\eqref{eq:eq5}, we instantiate them as a learnable embedding table with $M$ (default by $45$) vectors, each with a dimension of $384$. 
During training, all the training and testing images are resized to $448 \times 448$ resolution for injection into the feature extractor. AdamW is used as the optimizer with $\beta_1=0.9$ and $\beta_2=0.999$. The weight decay is set to $1e^{-4}$, and AMSGrad is enabled for more stable optimization. 
The initial learning rate is set to $0.001$ with a warm-up cosine scheduler, where the learning rate is first linearly increased from $1e^{-5}$ to $0.001$ during the first $2$ epochs and then cosine-decayed to $1\%$ of the base learning rate. The training epochs of IDEAL is set to $20$ with a batch size of $16$ on a single NVIDIA GeForce RTX 3090 GPU. By default, the number $k$ of normal neighbors in Eq.~\eqref{eq:eq2} is set to $12$, the number $r$ of retained principal components is set to $4$, $\alpha$ in Eq.~\eqref{eq:eq4} is set to $0.8$, and the number of intrinsic deviation vectors $M$ in Eq.~\eqref{eq:eq6} is set to $45$. The loss weights $\lambda_1$ and $\lambda_2$ in Eq.~\eqref{eq:eq7} are set to $1.0$ and $0.8$, respectively. More hyper-parameter ablations, please see Appendix~\ref{sec:hyper}.

\section{NVE as a Plug-and-Play Component} \label{sec:nveplug}
\begin{wraptable}[18]{r}{0.5\linewidth}
\vspace{-6.6mm}
\caption{\textbf{Enabling the existing residual-based FSAD methods}, in which their residual features are replaced with NVE-denoised residual features.}
\label{tab:nve_plug}
\vspace{1.5mm}
\centering
\begingroup
\setlength{\tabcolsep}{1.6pt} 
\renewcommand{\arraystretch}{1.2} 
\setlength{\arrayrulewidth}{0.1mm} 
\resizebox{\linewidth}{!}{%
    \begin{tabular}{cl II c I c}
    \hline\thickhline
    \rowcolor{mygray} 
    \multicolumn{2}{cII}{\textsc{\textbf{Methods}}} 
    & \textsc{\textbf{MVTecAD}} & \textsc{\textbf{VisA}} \\
    \hline\hline
    & InCTRL 
        & (93.3 \myred{0.8} / 94.7 \myred{1.3}) 
        & (83.4 \myred{1.2} / 93.5 \myred{0.4}) \\
    & ResAD 
        & (83.5 \myred{2.8} / 86.0 \myred{1.7}) 
        & (79.2 \myred{2.0} / 90.0 \myred{1.7}) \\
    \multirow{-3}{*}{\makecell[c]{\texttt{N1}}} & RegAD 
        & (93.1 \myred{1.5} / 93.8 \myred{0.7}) 
        & (85.2 \myred{1.9} / 92.0 \myred{0.4}) \\
    \hline
    \texttt{N1A1} & NAGL 
        & (95.3 \myred{0.2} / 96.6 \myred{0.5}) 
        & (89.4 \myred{0.9} / 97.6 \myred{0.1}) \\
    \hline\hline
    & InCTRL 
        & (94.7 \myred{0.4} / 94.5 \myred{0.5}) 
        & (90.2 \myred{0.7} / 94.5 \myred{0.3}) \\
    & ResAD 
        & (89.1 \myred{2.0} / 88.1 \myred{1.2}) 
        & (84.7 \myred{1.6} / 91.0 \myred{0.9}) \\
    \multirow{-3}{*}{\makecell[c]{\texttt{N4}}} & RegAD 
        & (94.9 \myred{0.9} / 95.5 \myred{0.4}) 
        & (89.9 \myred{0.4} / 94.6 \myred{1.3}) \\
    \hline
    \texttt{N4A1} & NAGL 
        & (97.0 \myred{0.1} / 97.3 \myred{0.3}) 
        & (91.8 \myred{0.6} / 97.8 \myred{0.1}) \\
    \hline\hline
    & InCTRL 
        & (96.1 \myred{0.3} / 95.0 \myred{0.6}) 
        & (90.7 \myred{0.9} / 95.5 \myred{0.2}) \\
    & ResAD 
        & (90.5 \myred{1.1} / 88.5 \myred{0.7}) 
        & (85.0 \myred{1.3} / 91.2 \myred{0.3}) \\
    \multirow{-3}{*}{\makecell[c]{\texttt{N8}}} & RegAD 
        & (95.3 \myred{0.3} / 96.1 \myred{0.3}) 
        & (92.5 \myred{0.6} / 95.1 \myred{0.4}) \\
    \hline
    \texttt{N8A4} & NAGL 
        & (97.4 \myred{0.2} / 97.5 \myred{0.4}) 
        & (92.0 \myred{0.5} / 98.1 \myred{0.2}) \\
    \hline\thickhline
    \end{tabular}
}
\endgroup
\vspace{8mm}
\end{wraptable}
The purpose of NVE (Sec.~\ref{sec:nve}) is to isolate anomaly-relevant deviation representations by eliminating nuisance normal variations that may lead to noisy deviations from normality. 
Since our NVE component operates at the feature level, it can be readily incorporated with existing residual-based FSAD methods without modifying their core architectures. 
Tab.~\ref{tab:nve_plug} presents image-level and pixel-level AUROCs of four residual-based FSAD methods (ResAD \cite{yao2024resad}, InCTRL \cite{zhu2024toward}, RegAD \cite{huang2022registration}, NAGL \cite{wang2025normal}) after replacing their original residual features with NVE-denoised residual features. 
We can observe that our NVE component consistently improves the performance of these methods in MVTecAD and VisA datasets. 
This confirms that nuisance normal variations are indeed present in the residual features used by these methods, and that NVE suppresses such variations to produce cleaner residual features.

\section{Hyper-parameter Ablation} \label{sec:hyper}
\begin{figure*}[h]
    \centering
    \includegraphics[width=0.52\textwidth]{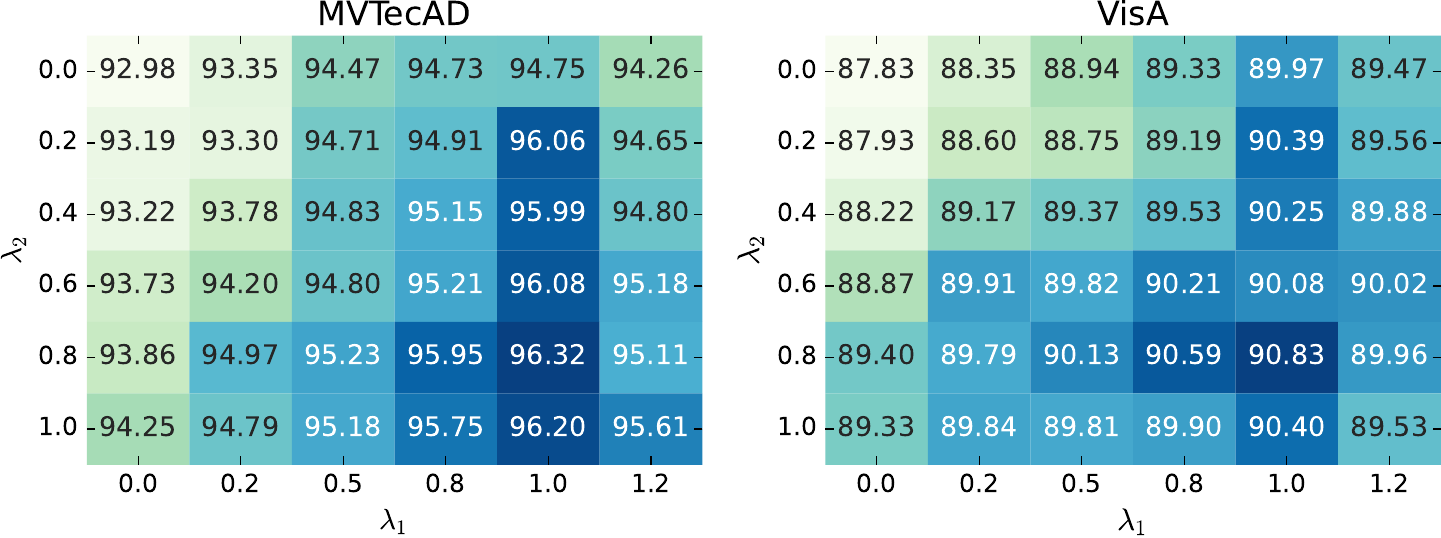}
    \hfill
    \includegraphics[width=0.45\textwidth]{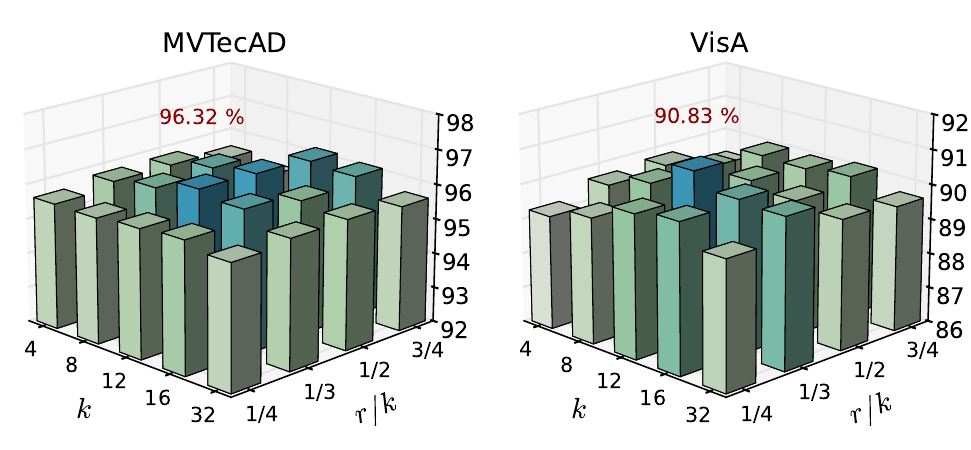}
    \caption{Under the \texttt{N1A1} setting. \textbf{Left}: hyper-parameter study with different $\lambda_1$ and $\lambda_2$ in Eq.~\eqref{eq:eq7} on MVTecAD and VisA. 
    \textbf{Right}: hyper-parameter study with the number $k$ of normal neighbors in Eq.~\eqref{eq:eq2} and the number $r$ of retained principal components in the process of NVE (Sec.~\ref{sec:nve}).}
    \label{fig:fig_hyper}
\end{figure*}

\begin{wrapfigure}{r}{0.5\linewidth}
    \centering
    \includegraphics[width=\linewidth]{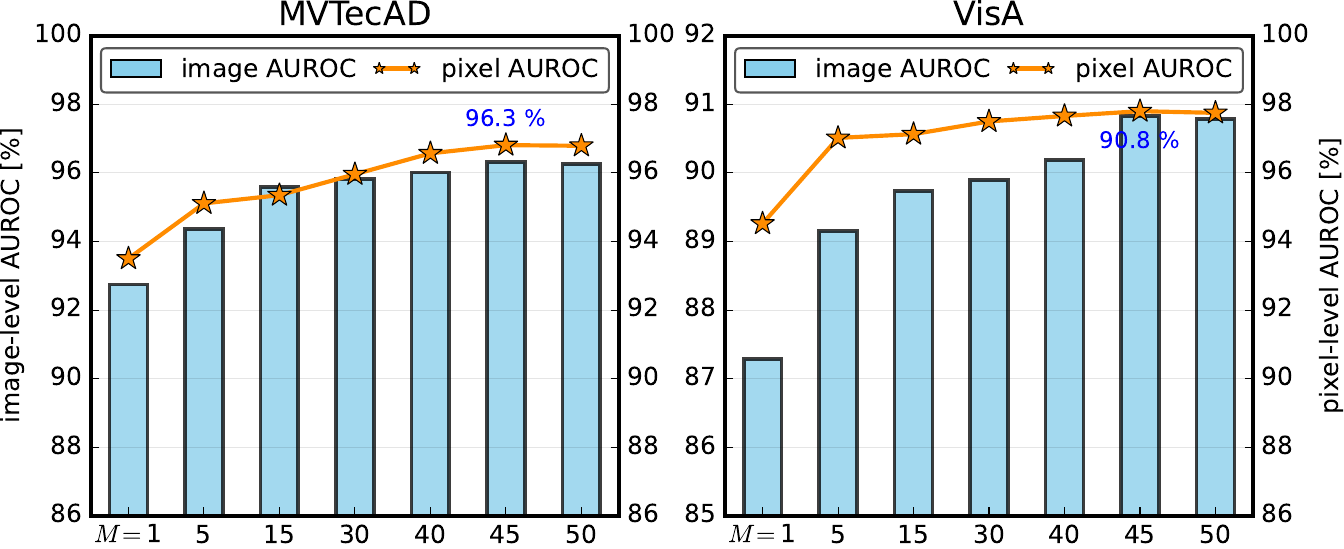}
    \caption{Under the \texttt{N1A1} setting. Ablation study with the number $M$ of intrinsic deviation vectors in Eq.~\eqref{eq:eq6}.}
\label{fig:fig_hyperM}
\vspace{-3mm}
\end{wrapfigure}
We first tune three key arguments in IDEAL: loss weights $\lambda_1$ and $\lambda_2$ in Eq.~\eqref{eq:eq7}, the number $M$ of intrinsic deviation vectors in Eq.~\eqref{eq:eq6}. Fig.~\ref{fig:fig_hyper} (\textbf{Left}) reports the image-level AUROC with different $\lambda_1$ and $\lambda_2$ on the MVTecAD and VisA datasets. 
We can observe that the performance improves gradually as $\lambda_1$ and $\lambda_2$ increase, indicating that both the discriminability and orthogonality constraints are beneficial for learning the intrinsic deviation vectors. However, an excessively large $\lambda_2$ becomes detrimental to training, because the orthogonality term may force the intrinsic deviation vectors apart even when they share correlated patterns. 
As shown in Fig.~\ref{fig:fig_hyperM}, we investigate the number $M$ of intrinsic deviation vectors in Eq.~\eqref{eq:eq6} for a wide range, and present image-level and pixel-level AUROCs on the MVTecAD and VisA datasets. 
We observe that increasing $M$ consistently improves the performance of IDEAL, with the improvement becoming marginal at $M=45$. 
This suggests that more intrinsic deviation vectors enable better capture of discriminative orthogonal deviation directions, while the benefit tends to saturate once sufficient capacity is reached. 
Moreover, we also study the effect of the number $k$ of normal neighbors in Eq.~\eqref{eq:eq2} and the number $r$ of retained principal components in the process of NVE (Sec.~\ref{sec:nve}). 
Fig.~\ref{fig:fig_hyper} (\textbf{Right}) shows that IDEAL is generally robust to both $k$ and $r$. 
A moderate number $k$ of normal neighbors yields better performance, as it provides a more reliable estimate of local normal variations, whereas an overly large $k$ may introduce less relevant neighbors. Similarly, retaining an appropriate proportion of principal components $r$ helps preserve dominant normal patterns while filtering out noisy directions. 
Tab.~\ref{tab:tab_loss} (\textbf{Left}) shows that IDEAL is robust to the choice of $\alpha \in [0.5,0.8]$ in Eq.~\eqref{eq:eq4}, and $\alpha=0.8$ yields the best overall performance, suggesting that a moderate elimination strength provides a favorable balance. 
Empirically, $\lambda_{1}=1.0$, $\lambda_{2}=0.8$, $M=45$, $\alpha=0.8$, $k=12$, $r=4$ is a recommended configuration.

\begin{table*}[t!]
\caption{Under the \texttt{N1A1} setting. \textbf{Left}: hyper-parameter study with $\alpha$ in Eq.~\eqref{eq:eq4} on MVTecAD and VisA. 
\textbf{Right}: the sensitivity among the loss terms in Eq.~\eqref{eq:eq10}.}
\label{tab:tab_loss}
\vspace{1.5mm}
\centering
\begin{minipage}{0.4\linewidth}
\centering
\begingroup
\renewcommand{\arraystretch}{1.15} 
\setlength{\arrayrulewidth}{0.1mm} 
\resizebox{\linewidth}{!}{%
    \begin{tabular}{c I I c I c}
    \hline\thickhline
    \rowcolor{mygray} 
    $\alpha$ in Eq.~\eqref{eq:eq4} 
    & \textsc{\textbf{MVTecAD}} & \textsc{\textbf{VisA}} \\
    \hline\hline
    $\alpha=0.2$ & (95.5 / 96.0) & (89.8 / 97.1) \\
    $\alpha=0.5$ 
    & (\textcolor{blue}{\textbf{96.2}} / \textcolor{blue}{\textbf{96.5}}) 
    & (\textcolor{blue}{\textbf{90.5}} / \textcolor{blue}{\textbf{97.6}}) \\
    $\alpha=0.8$ & (\textcolor{DeepRed}{\textbf{96.3}} / \textcolor{DeepRed}{\textbf{96.8}}) 
    & (\textcolor{DeepRed}{\textbf{90.8}} / \textcolor{DeepRed}{\textbf{97.8}}) \\
    $\alpha=1.0$ & (96.0 / 96.2) & (90.3 / 97.3) \\
    \hline\thickhline
    \end{tabular}
}
\endgroup
\end{minipage}
\hfill
\begin{minipage}{0.56\linewidth}
\centering
\begingroup
\renewcommand{\arraystretch}{1.22} 
\setlength{\arrayrulewidth}{0.1mm} 
\resizebox{\linewidth}{!}{%
    \begin{tabular}{cccc II c I c}
    \hline\thickhline
    \rowcolor{mygray} 
    $\mathcal{L}_{Focal}$ & $\mathcal{L}_{Dice}$ & $\mathcal{L}_{BCE}$ & $\mathcal{L}_{Dual}$ 
    & \textsc{\textbf{MVTecAD}} & \textsc{\textbf{VisA}} \\
    \hline\hline
    \ding{51} & \ding{51} & & & (93.8 / 95.9) & (86.7 / 95.6) \\
    \ding{51} & \ding{51} & \ding{51} & & (94.3 / 96.3) & (88.2 / \textcolor{blue}{\textbf{97.2}}) \\
    \ding{51} & \ding{51} & & \ding{51} & (\textcolor{blue}{\textbf{94.8}} / \textcolor{blue}{\textbf{96.6}}) 
    & (\textcolor{blue}{\textbf{88.7}} / \textcolor{blue}{\textbf{97.2}}) \\
    \ding{51} & \ding{51} & \ding{51} & \ding{51} 
    & (\textcolor{DeepRed}{\textbf{96.3}} / \textcolor{DeepRed}{\textbf{96.8}}) 
    & (\textcolor{DeepRed}{\textbf{90.8}} / \textcolor{DeepRed}{\textbf{97.8}}) \\
    \hline\thickhline
    \end{tabular}
}
\endgroup
\end{minipage}
\end{table*}

\section{Ablation Study on Loss Terms} \label{sec:baloss}
We conduct an ablation study among the loss terms of Eq.~\eqref{eq:eq10} in Tab.~\ref{tab:tab_loss} (\textbf{Right}). We observe that $\mathcal{L}_{Focal}$ and $\mathcal{L}_{Dice}$ provide the basic pixel-level supervision for anomaly localization, while $\mathcal{L}_{BCE}$ improves image-level anomaly discrimination. Moreover, adding the $\mathcal{L}_{Dual}$ consistently improves performance, suggesting that encouraging the learning of intrinsic deviation vectors helps extract more discriminative deviations from normality, thereby enhancing detection. Combining all loss terms yields the best results on both MVTecAD and VisA, confirming their complementary roles.

\section{Limitations and Social Impacts} \label{sec:lim}
The proposed IDEAL is currently evaluated only on image anomaly detection datasets, including industrial and medical images. It would be highly valuable to further apply our framework to other data modalities, such as tabular and time-series data, to more comprehensively assess the generalizability of IDEAL. This work does not introduce specific ethical concerns or social impacts beyond those commonly associated with anomaly detection. All datasets used in our experiments are publicly available, and no private user data or personally identifiable information is collected.

\section{Additional Results} \label{sec:allres}
We provide complete results on different datasets, including MVTecAD (Tab.~\ref{tab:tab_add_mvtec}), VisA (Tab.~\ref{tab:tab_add_visa}), AITEX (Tab.~\ref{tab:tab_add_aitex}), MPDD (Tab.~\ref{tab:tab_add_mpdd}), BTAD (Tab.~\ref{tab:tab_add_btad}), BraTS (Tab.~\ref{tab:tab_add_brats}), Liver (Tab.~\ref{tab:tab_add_liver}), and RESC (Tab.~\ref{tab:tab_add_resc}). The details of each dataset, please see Appendix~\ref{sec:addata}.

\begin{table*}[t!]
\caption{\textbf{Comparison results on MVTecAD} under various few-shot AD settings. Best results are in \textcolor{DeepRed}{\textbf{red bold}}, with second best results in \textcolor{blue}{\textbf{blue bold}}.}
\label{tab:tab_add_mvtec}
\vspace{1.5mm}
\centering
\begingroup
\setlength{\tabcolsep}{10pt} 
\renewcommand{\arraystretch}{1.1} 
\setlength{\arrayrulewidth}{0.1mm} 
\resizebox{\linewidth}{!}{%
\begin{tabular}{cl I I c I c I c I c I c I c}
    \hline\thickhline
    \rowcolor{mygray} 
    & & \multicolumn{3}{c I}{\textsc{\textbf{Image-level}}} 
    & \multicolumn{3}{c}{\textsc{\textbf{Pixel-level}}} \\
    \cline{3-5} \cline{6-8} 
    \rowcolor{mygray} 
    \multirow{-2}{*}{\centering \textsc{\textbf{Setups}}} 
    & \multirow{-2}{*}{\centering \textsc{\textbf{Methods}}} 
    & AUROC & AP & F1-Max & AUROC & PRO & F1-Max \\
    \hline\hline
    & WinCLIP~\cite{jeong2023winclip} 
        & 93.1 & 96.5 & 93.6 & 93.4 & 83.4 & 42.2\\
    & InCTRL~\cite{zhu2024toward} 
        & 92.5 & 96.2 & 93.7 & 93.4 & 83.3 & 45.9\\
    & ResAD~\cite{yao2024resad} 
        & 80.7 & 85.8 & 86.8 & 84.3 & 70.6 & 32.0\\
    & FoundAD~\cite{zhai2026foundation} 
        & 89.9 & 91.0 & 83.6 & 90.7 & 88.7 & 56.8\\
    & RegAD~\cite{huang2022registration} 
        & 91.6 & 90.8 & 90.9 & 93.1 & 87.3 & 54.4\\
    \multirow{-6}{*}{\makecell[c]{\texttt{N1}}} & PromptAD~\cite{li2024promptad} 
        & 92.8 & 92.7 & 91.3 & 93.4 & 90.9 & 56.6\\
    \hline
    & DRA~\cite{ding2022catching} 
        & 92.2 & 96.1 & 93.1 & 92.9 & 91.2 & 51.7\\
    \multirow{-2}{*}{\makecell[c]{\texttt{A1}}} & AHL~\cite{zhu2024anomaly} 
        & 91.7 & 94.3 & 93.3 & 93.6 & 90.8 & 54.2\\
    \hline
    & NAGL~\cite{wang2025normal} 
        & \textcolor{blue}{\textbf{95.1}} & \textcolor{blue}{\textbf{96.9}} 
        & \textcolor{blue}{\textbf{95.4}} & \textcolor{blue}{\textbf{96.1}} 
        & \textcolor{blue}{\textbf{92.9}} & \textcolor{blue}{\textbf{58.7}}\\
    \multirow{-2}{*}{\makecell[c]{\texttt{N1A1}}} & \cellcolor{lightyellow}\textbf{IDEAL} (ours) 
        & \textcolor{DeepRed}{\textbf{96.3}} & \textcolor{DeepRed}{\textbf{97.3}} 
        & \textcolor{DeepRed}{\textbf{95.7}} & \textcolor{DeepRed}{\textbf{96.8}}
        & \textcolor{DeepRed}{\textbf{93.4}} & \textcolor{DeepRed}{\textbf{59.1}}\\
    \hline\hline
    & WinCLIP~\cite{jeong2023winclip} 
        & 93.4 & 96.5 & 94.1 & 93.6 & 84.2 & 42.1\\
    & InCTRL~\cite{zhu2024toward} 
    & 92.7 & 96.5&93.8 &93.7&83.7&46.6\\
    & ResAD~\cite{yao2024resad} 
    &85.7 &89.0 &88.2&86.3&72.6&34.2 \\
    & FoundAD~\cite{zhai2026foundation} 
    & 90.3& 91.8& 94.1& 92.8& 89.9& 58.4 \\
    & RegAD~\cite{huang2022registration} 
    & 92.6& 93.2& 91.8& 93.5& 91.8& 56.5 \\
    \multirow{-6}{*}{\makecell[c]{\texttt{N2}}} & PromptAD~\cite{li2024promptad} 
    & 93.5& 92.9& 91.5& 93.8& 91.3& 57.2 \\
    \hline
    & NAGL~\cite{wang2025normal} 
    & \textcolor{blue}{\textbf{96.0}}& \textcolor{blue}{\textbf{97.8}}
    & \textcolor{blue}{\textbf{96.4}}& \textcolor{blue}{\textbf{96.6}}
    & \textcolor{blue}{\textbf{93.3}}& \textcolor{blue}{\textbf{59.8}} \\
    \multirow{-2}{*}{\makecell[c]{\texttt{N2A1}}} & \cellcolor{lightyellow}\textbf{IDEAL} (ours) 
    & \textcolor{DeepRed}{\textbf{97.4}}& \textcolor{DeepRed}{\textbf{98.8}}
    & \textcolor{DeepRed}{\textbf{96.9}} & \textcolor{DeepRed}{\textbf{97.5}}
    & \textcolor{DeepRed}{\textbf{93.5}}& \textcolor{DeepRed}{\textbf{60.0}} \\
    \hline\hline
    & WinCLIP~\cite{jeong2023winclip} 
    & 93.9& 97.3& 95.1& 93.8& 84.5& 43.1 \\
    & InCTRL~\cite{zhu2024toward} 
    & 94.3& 96.8& 95.0& 94.0& 83.8& 46.9 \\
    & ResAD~\cite{yao2024resad} 
    & 87.1& 89.4& 88.7& 86.9& 73.1& 36.2 \\
    & FoundAD~\cite{zhai2026foundation} 
    & 92.9& 92.9& 95.6& 93.2& 90.7& \textcolor{blue}{\textbf{60.5}} \\
    & RegAD~\cite{huang2022registration} 
    & 94.0& 94.9& 93.1& 95.1& 92.3& 60.1 \\
    \multirow{-6}{*}{\makecell[c]{\texttt{N4}}} & PromptAD~\cite{li2024promptad} 
    & 94.9& 93.0& 91.7& 94.6& 92.4& 57.5 \\
    \hline
    & NAGL~\cite{wang2025normal} 
    & \textcolor{blue}{\textbf{96.9}}& \textcolor{blue}{\textbf{97.5}}& \textcolor{blue}{\textbf{96.6}}& \textcolor{blue}{\textbf{97.0}}& \textcolor{blue}{\textbf{93.4}}& 59.9 \\
    \multirow{-2}{*}{\makecell[c]{\texttt{N4A1}}} & \cellcolor{lightyellow}\textbf{IDEAL} (ours) 
    & \textcolor{DeepRed}{\textbf{98.2}}& \textcolor{DeepRed}{\textbf{99.0}}& \textcolor{DeepRed}{\textbf{97.6}}& \textcolor{DeepRed}{\textbf{97.5}}& \textcolor{DeepRed}{\textbf{93.9}}& \textcolor{DeepRed}{\textbf{62.1}} \\
    \hline\hline
    & WinCLIP~\cite{jeong2023winclip} 
    & 95.5 & 97.8& 95.1& 94.2& 85.9& 45.6 \\
    & InCTRL~\cite{zhu2024toward} 
    & 95.8& 97.8& 95.5& 94.4& 84.7& 48.1 \\
    & ResAD~\cite{yao2024resad} 
    & 89.4& 90.2& 89.5& 87.8& 74.2& 37.8 \\
    & FoundAD~\cite{zhai2026foundation} 
    & 93.2& 93.9& 96.7& 93.5& 91.4& 60.3 \\
    & RegAD~\cite{huang2022registration} 
    & 95.0& 96.2& 94.2& 95.8& \textcolor{blue}{\textbf{93.8}}& \textcolor{blue}{\textbf{61.0}} \\
    \multirow{-6}{*}{\makecell[c]{\texttt{N8}}} & PromptAD~\cite{li2024promptad} 
    & 96.0& 93.3& 91.8& 94.8& 92.9& 57.6 \\
    \hline
    & DRA~\cite{ding2022catching} 
    & 93.7& 96.0& 93.3& 94.0& 93.3& 55.5 \\
    \multirow{-2}{*}{\makecell[c]{\texttt{A4}}} & AHL~\cite{zhu2024anomaly} 
    & 92.9& 94.7& 94.0& 94.5& 93.5& 57.4 \\
    \hline
    & NAGL~\cite{wang2025normal} 
    & \textcolor{blue}{\textbf{97.2}}& \textcolor{blue}{\textbf{98.7}}& \textcolor{blue}{\textbf{97.3}}& \textcolor{blue}{\textbf{97.1}}& \textcolor{blue}{\textbf{93.8}}& 60.7 \\
    \multirow{-2}{*}{\makecell[c]{\texttt{N8A4}}} & \cellcolor{lightyellow}\textbf{IDEAL} (ours) 
    & \textcolor{DeepRed}{\textbf{98.6}}& \textcolor{DeepRed}{\textbf{99.2}}& \textcolor{DeepRed}{\textbf{98.0}}& \textcolor{DeepRed}{\textbf{97.8}}& \textcolor{DeepRed}{\textbf{94.2}}& \textcolor{DeepRed}{\textbf{62.4}} \\
    \hline\thickhline
\end{tabular}
}
\endgroup
\end{table*}

\begin{table*}[t!]
\caption{\textbf{Comparison results on VisA} under various few-shot AD settings. Best results are in \textcolor{DeepRed}{\textbf{red bold}}, with second best results in \textcolor{blue}{\textbf{blue bold}}.}
\label{tab:tab_add_visa}
\vspace{1.5mm}
\centering
\begingroup
\setlength{\tabcolsep}{10pt} 
\renewcommand{\arraystretch}{1.1} 
\setlength{\arrayrulewidth}{0.1mm} 
\resizebox{\linewidth}{!}{%
\begin{tabular}{cl I I c I c I c I c I c I c}
    \hline\thickhline
    \rowcolor{mygray} 
    & & \multicolumn{3}{c I}{\textsc{\textbf{Image-level}}} 
    & \multicolumn{3}{c}{\textsc{\textbf{Pixel-level}}} \\
    \cline{3-5} \cline{6-8} 
    \rowcolor{mygray} 
    \multirow{-2}{*}{\centering \textsc{\textbf{Setups}}} 
    & \multirow{-2}{*}{\centering \textsc{\textbf{Methods}}} 
    & AUROC & AP & F1-Max & AUROC & PRO & F1-Max \\
    \hline\hline
    & WinCLIP~\cite{jeong2023winclip} 
    & 78.7& 80.3& 80.2& 94.7& 79.2& 32.9 \\
    & InCTRL~\cite{zhu2024toward} 
    & 82.2& 84.1& 81.6& 93.1& 65.2& 35.2 \\
    & ResAD~\cite{yao2024resad} 
    & 77.2& 72.1& 76.2& 88.3& 70.9& 32.1 \\
    & FoundAD~\cite{zhai2026foundation} 
    & 85.3& \textcolor{blue}{\textbf{89.9}}& 82.8& 96.2& 90.5& 40.6 \\
    & RegAD~\cite{huang2022registration} 
    & 83.3& 82.5& 85.6& 91.6& 89.8& 32.7 \\
    \multirow{-6}{*}{\makecell[c]{\texttt{N1}}} & PromptAD~\cite{li2024promptad} 
    & 83.8& 87.8& \textcolor{blue}{\textbf{86.1}}& 89.5& 87.8& \textcolor{blue}{\textbf{41.7}} \\
    \hline
    & DRA~\cite{ding2022catching} 
    & 82.3& 84.6& 80.7& 90.3& 77.3& 40.5 \\
    \multirow{-2}{*}{\makecell[c]{\texttt{A1}}} & AHL~\cite{zhu2024anomaly} 
    & 82.0& 85.5& 81.4& 91.9& 81.6& 40.8 \\
    \hline
    & NAGL~\cite{wang2025normal} 
    & \textcolor{blue}{\textbf{88.5}}& 89.2& 85.2& \textcolor{blue}{\textbf{97.5}}& \textcolor{blue}{\textbf{91.0}}& 41.5 \\
    \multirow{-2}{*}{\makecell[c]{\texttt{N1A1}}} & \cellcolor{lightyellow}\textbf{IDEAL} (ours) 
    & \textcolor{DeepRed}{\textbf{90.8}}& \textcolor{DeepRed}{\textbf{91.3}}& \textcolor{DeepRed}{\textbf{87.2}}& \textcolor{DeepRed}{\textbf{97.8}}& \textcolor{DeepRed}{\textbf{92.1}}& \textcolor{DeepRed}{\textbf{42.1}} \\
    \hline\hline
    & WinCLIP~\cite{jeong2023winclip} 
    & 84.3& 85.7& 82.7& 94.8& 80.1& 33.1 \\
    & InCTRL~\cite{zhu2024toward} 
    & 85.6& 89.0& 85.8& 93.7& 74.6& 36.3 \\
    & ResAD~\cite{yao2024resad} 
    &82.4&73.5&77.5&89.8&71.9&32.7 \\
    & FoundAD~\cite{zhai2026foundation} 
    & 88.5& 90.6& 83.2& 96.5& 90.9& 40.9 \\
    & RegAD~\cite{huang2022registration} 
    & 87.4& 88.4& 86.0& 92.7& 90.4& 38.1 \\
    \multirow{-6}{*}{\makecell[c]{\texttt{N2}}} & PromptAD~\cite{li2024promptad} 
    & 86.6& 89.7& \textcolor{blue}{\textbf{87.5}}& 89.8& 89.1& 42.1 \\
    \hline
    & NAGL~\cite{wang2025normal} 
    & \textcolor{blue}{\textbf{89.1}}& \textcolor{blue}{\textbf{90.7}}& 86.0& \textcolor{blue}{\textbf{97.5}}& \textcolor{blue}{\textbf{91.5}}& \textcolor{blue}{\textbf{43.2}} \\
    \multirow{-2}{*}{\makecell[c]{\texttt{N2A1}}} & \cellcolor{lightyellow}\textbf{IDEAL} (ours) 
    & \textcolor{DeepRed}{\textbf{91.3}}& \textcolor{DeepRed}{\textbf{91.9}}& \textcolor{DeepRed}{\textbf{87.9}}& \textcolor{DeepRed}{\textbf{97.9}}& \textcolor{DeepRed}{\textbf{92.9}}& \textcolor{DeepRed}{\textbf{44.5}} \\
    \hline\hline
    & WinCLIP~\cite{jeong2023winclip} 
    & 84.5& 85.8& 83.2& 95.1& 80.0& 34.9 \\
    & InCTRL~\cite{zhu2024toward} 
    & 89.5& 89.7& 86.5& 94.2& 74.3& 36.9 \\
    & ResAD~\cite{yao2024resad} 
    & 83.1& 75.4& 78.9& 90.1& 72.2& 36.4 \\
    & FoundAD~\cite{zhai2026foundation} 
    & 91.0& 90.9& 84.6& 96.7& 90.4& 42.3 \\
    & RegAD~\cite{huang2022registration} 
    & 89.5& 90.8& 87.1& 93.3& \textcolor{blue}{\textbf{91.5}}& 42.5 \\
    \multirow{-6}{*}{\makecell[c]{\texttt{N4}}} & PromptAD~\cite{li2024promptad} 
    & 87.1& 90.1& \textcolor{blue}{\textbf{87.9}}& 90.6& 89.5& 42.8 \\
    \hline
    & NAGL~\cite{wang2025normal} 
    & \textcolor{blue}{\textbf{91.2}}& \textcolor{blue}{\textbf{91.7}}& 87.1& \textcolor{blue}{\textbf{97.7}}& 91.3& \textcolor{blue}{\textbf{44.0}} \\
    \multirow{-2}{*}{\makecell[c]{\texttt{N4A1}}} & \cellcolor{lightyellow}\textbf{IDEAL} (ours) 
    & \textcolor{DeepRed}{\textbf{92.7}}& \textcolor{DeepRed}{\textbf{92.6}}& \textcolor{DeepRed}{\textbf{89.1}}& \textcolor{DeepRed}{\textbf{98.0}}& \textcolor{DeepRed}{\textbf{93.5}}& \textcolor{DeepRed}{\textbf{44.8}} \\
    \hline\hline
    & WinCLIP~\cite{jeong2023winclip} 
    & 86.1& 87.5& 83.7& 95.1& 81.5& 35.3 \\
    & InCTRL~\cite{zhu2024toward} 
    & 89.8& 90.9& 86.6& 95.3& 74.9& 37.8 \\
    & ResAD~\cite{yao2024resad} 
    & 83.7& 77.4& 79.5& 90.9& 73.9& 38.3 \\
    & FoundAD~\cite{zhai2026foundation} 
    & 91.3& 91.1& 84.8& 96.8& 90.4& 43.7 \\
    & RegAD~\cite{huang2022registration} 
    & \textcolor{blue}{\textbf{91.9}}& 92.0& 88.3& 94.7& 92.0& 43.4 \\
    \multirow{-6}{*}{\makecell[c]{\texttt{N8}}} & PromptAD~\cite{li2024promptad} 
    & 87.9& 90.7& 88.2& 92.2& 90.6& 43.5 \\
    \hline
    & DRA~\cite{ding2022catching} 
    & 86.8& 91.8& 87.8& 90.7& 87.9& 42.4 \\
    \multirow{-2}{*}{\makecell[c]{\texttt{A4}}} & AHL~\cite{zhu2024anomaly} 
    & 85.9& 90.6& 88.3& 92.5& 87.0& 41.9 \\
    \hline
    & NAGL~\cite{wang2025normal} 
    & 91.5& \textcolor{blue}{\textbf{92.8}}& \textcolor{blue}{\textbf{88.5}}& \textcolor{blue}{\textbf{97.9}}& \textcolor{blue}{\textbf{92.1}}& \textcolor{blue}{\textbf{45.3}} \\
    \multirow{-2}{*}{\makecell[c]{\texttt{N8A4}}} & \cellcolor{lightyellow}\textbf{IDEAL} (ours) 
    & \textcolor{DeepRed}{\textbf{93.1}}& \textcolor{DeepRed}{\textbf{92.9}}& \textcolor{DeepRed}{\textbf{89.4}}& \textcolor{DeepRed}{\textbf{98.3}}& \textcolor{DeepRed}{\textbf{93.6}}& \textcolor{DeepRed}{\textbf{45.9}} \\
    \hline\thickhline
\end{tabular}
}
\endgroup
\end{table*}

\begin{table*}[t!]
\caption{\textbf{Comparison results on AITEX} under various few-shot AD settings. Best results are in \textcolor{DeepRed}{\textbf{red bold}}, with second best results in \textcolor{blue}{\textbf{blue bold}}.}
\label{tab:tab_add_aitex}
\vspace{1.5mm}
\centering
\begingroup
\setlength{\tabcolsep}{10pt} 
\renewcommand{\arraystretch}{1.1} 
\setlength{\arrayrulewidth}{0.1mm} 
\resizebox{\linewidth}{!}{%
\begin{tabular}{cl I I c I c I c I c I c I c}
    \hline\thickhline
    \rowcolor{mygray} 
    & & \multicolumn{3}{c I}{\textsc{\textbf{Image-level}}} 
    & \multicolumn{3}{c}{\textsc{\textbf{Pixel-level}}} \\
    \cline{3-5} \cline{6-8} 
    \rowcolor{mygray} 
    \multirow{-2}{*}{\centering \textsc{\textbf{Setups}}} 
    & \multirow{-2}{*}{\centering \textsc{\textbf{Methods}}} 
    & AUROC & AP & F1-Max & AUROC & PRO & F1-Max \\
    \hline\hline
    & WinCLIP~\cite{jeong2023winclip} 
    & 73.1& \textcolor{blue}{\textbf{48.4}}& 51.1& 78.9& 58.3& \textcolor{blue}{\textbf{22.1}} \\
    & InCTRL~\cite{zhu2024toward} 
    & \textcolor{blue}{\textbf{74.9}}& 41.4& 60.3& 78.3& \textcolor{blue}{\textbf{60.5}}& 11.7 \\
    & ResAD~\cite{yao2024resad} 
    & 70.5& 37.4& 51.6& 72.7& 51.9& 10.3 \\
    & FoundAD~\cite{zhai2026foundation} 
    & 71.4& 42.0& 54.6& 80.4& 57.1& 10.5 \\
    & RegAD~\cite{huang2022registration} 
    & 72.0& 47.1& 58.2& 80.2& 56.6& 11.4 \\
    \multirow{-6}{*}{\makecell[c]{\texttt{N1}}} & PromptAD~\cite{li2024promptad} 
    & 71.9& 47.8& \textcolor{blue}{\textbf{60.5}}& \textcolor{blue}{\textbf{81.4}}& 58.2& 15.8 \\
    \hline
    & DRA~\cite{ding2022catching} 
    & 72.3& 45.1& 55.0& 75.3& 54.9& 15.7 \\
    \multirow{-2}{*}{\makecell[c]{\texttt{A1}}} & AHL~\cite{zhu2024anomaly} 
    & 70.7& 39.8& 53.3& 76.3& 53.6& 21.7 \\
    \hline
    & NAGL~\cite{wang2025normal} 
    & 71.6& 47.3& 55.3& 75.5& 58.4& 14.3 \\
    \multirow{-2}{*}{\makecell[c]{\texttt{N1A1}}} & \cellcolor{lightyellow}\textbf{IDEAL} (ours) 
    & \textcolor{DeepRed}{\textbf{75.8}}& \textcolor{DeepRed}{\textbf{50.2}}& \textcolor{DeepRed}{\textbf{61.2}}& \textcolor{DeepRed}{\textbf{82.7}}& \textcolor{DeepRed}{\textbf{62.0}}& \textcolor{DeepRed}{\textbf{25.5}} \\
    \hline\hline
    & WinCLIP~\cite{jeong2023winclip} 
    & 73.5& \textcolor{blue}{\textbf{52.2}}& 52.6& 79.3& 58.6& \textcolor{blue}{\textbf{22.9}} \\
    & InCTRL~\cite{zhu2024toward} 
    & \textcolor{blue}{\textbf{75.2}}& 41.8& 60.4 & 78.7& \textcolor{blue}{\textbf{60.4}}& 12.9 \\
    & ResAD~\cite{yao2024resad} 
    & 75.0& 47.2& 59.3& 75.1& 55.7& 13.4 \\
    & FoundAD~\cite{zhai2026foundation} 
    & 72.3& 43.1& 56.5& 81.4& 59.3& 12.1 \\
    & RegAD~\cite{huang2022registration} 
    & 73.5& 51.1& 59.9& \textcolor{blue}{\textbf{82.0}}& 59.1& 14.7 \\
    \multirow{-6}{*}{\makecell[c]{\texttt{N2}}} & PromptAD~\cite{li2024promptad} 
    & 73.8& 48.1& \textcolor{blue}{\textbf{60.6}}& 81.8& 58.8& 16.9 \\
    \hline
    & NAGL~\cite{wang2025normal} 
    & 73.5& 48.3& 58.2& 79.8& 58.7& 15.9 \\
    \multirow{-2}{*}{\makecell[c]{\texttt{N2A1}}} & \cellcolor{lightyellow}\textbf{IDEAL} (ours) 
    & \textcolor{DeepRed}{\textbf{77.6}}& \textcolor{DeepRed}{\textbf{53.1}}& \textcolor{DeepRed}{\textbf{61.7}}& \textcolor{DeepRed}{\textbf{85.3}}& \textcolor{DeepRed}{\textbf{62.6}}& \textcolor{DeepRed}{\textbf{26.9}} \\
    \hline\hline
    & WinCLIP~\cite{jeong2023winclip} 
    & 73.8& 52.5& 52.9& 79.5& 60.7& \textcolor{blue}{\textbf{23.2}} \\
    & InCTRL~\cite{zhu2024toward} 
    & \textcolor{blue}{\textbf{75.6}}& 43.7& 60.2& 80.4& \textcolor{blue}{\textbf{64.5}}& 13.7 \\
    & ResAD~\cite{yao2024resad} 
    & 75.2& 50.6& 59.4& 76.7& 57.1& 14.1 \\
    & FoundAD~\cite{zhai2026foundation} 
    & 73.2& 47.0& 58.2& 82.8& 63.3& 14.7 \\
    & RegAD~\cite{huang2022registration} 
    & 74.6& \textcolor{blue}{\textbf{54.0}}& 60.1& \textcolor{blue}{\textbf{83.7}}& 63.3& 15.5 \\
    \multirow{-6}{*}{\makecell[c]{\texttt{N4}}} & PromptAD~\cite{li2024promptad} 
    & 74.1& 50.8& \textcolor{blue}{\textbf{61.4}}& 83.4& 62.9& 18.6 \\
    \hline
    & NAGL~\cite{wang2025normal} 
    & \textcolor{blue}{\textbf{75.6}}& 50.2& 59.5& 80.9& 60.4& 19.0 \\
    \multirow{-2}{*}{\makecell[c]{\texttt{N4A1}}} & \cellcolor{lightyellow}\textbf{IDEAL} (ours) 
    & \textcolor{DeepRed}{\textbf{79.5}}& \textcolor{DeepRed}{\textbf{58.8}}& \textcolor{DeepRed}{\textbf{63.4}}& \textcolor{DeepRed}{\textbf{86.1}}& \textcolor{DeepRed}{\textbf{66.1}}& \textcolor{DeepRed}{\textbf{29.5}} \\
    \hline\hline
    & WinCLIP~\cite{jeong2023winclip} 
    & 74.8& 55.0& 53.4& 80.5& 62.6& 28.3 \\
    & InCTRL~\cite{zhu2024toward} 
    & 75.9& 44.1& 60.6& 80.6& 65.2& 14.6 \\
    & ResAD~\cite{yao2024resad} 
    & 75.4& 53.8& 59.4& 79.3& 57.4& 15.8 \\
    & FoundAD~\cite{zhai2026foundation} 
    & 74.3& 52.1& \textcolor{blue}{\textbf{64.6}}& 83.0& 64.1& 15.7 \\
    & RegAD~\cite{huang2022registration} 
    & 75.8& 56.5& 64.2& 83.9& \textcolor{blue}{\textbf{66.4}}& 17.5 \\
    \multirow{-6}{*}{\makecell[c]{\texttt{N8}}} & PromptAD~\cite{li2024promptad} 
    & 74.4& 56.0& 61.9& \textcolor{blue}{\textbf{84.7}}& 64.3& 20.2 \\
    \hline
    & DRA~\cite{ding2022catching} 
    & 75.7& \textcolor{blue}{\textbf{58.3}}& 56.5& 77.8& 56.5& 23.2 \\
    \multirow{-2}{*}{\makecell[c]{\texttt{A4}}} & AHL~\cite{zhu2024anomaly} 
    & 75.1& 56.3& 52.6& 78.8& 52.6& \textcolor{blue}{\textbf{29.8}} \\
    \hline
    & NAGL~\cite{wang2025normal} 
    & \textcolor{blue}{\textbf{80.5}}& 54.8& 60.5& 82.3& 65.2& 20.5 \\
    \multirow{-2}{*}{\makecell[c]{\texttt{N8A4}}} & \cellcolor{lightyellow}\textbf{IDEAL} (ours) 
    & \textcolor{DeepRed}{\textbf{83.9}}& \textcolor{DeepRed}{\textbf{61.1}}& \textcolor{DeepRed}{\textbf{68.2}}& \textcolor{DeepRed}{\textbf{86.8}}& \textcolor{DeepRed}{\textbf{73.0}}& \textcolor{DeepRed}{\textbf{30.3}} \\
    \hline\thickhline
\end{tabular}
}
\endgroup
\end{table*}

\begin{table*}[t!]
\caption{\textbf{Comparison results on MPDD} under various few-shot AD settings. Best results are in \textcolor{DeepRed}{\textbf{red bold}}, with second best results in \textcolor{blue}{\textbf{blue bold}}.}
\label{tab:tab_add_mpdd}
\vspace{1.5mm}
\centering
\begingroup
\setlength{\tabcolsep}{10pt} 
\renewcommand{\arraystretch}{1.1} 
\setlength{\arrayrulewidth}{0.1mm} 
\resizebox{\linewidth}{!}{%
\begin{tabular}{cl I I c I c I c I c I c I c}
    \hline\thickhline
    \rowcolor{mygray} 
    & & \multicolumn{3}{c I}{\textsc{\textbf{Image-level}}} 
    & \multicolumn{3}{c}{\textsc{\textbf{Pixel-level}}} \\
    \cline{3-5} \cline{6-8} 
    \rowcolor{mygray} 
    \multirow{-2}{*}{\centering \textsc{\textbf{Setups}}} 
    & \multirow{-2}{*}{\centering \textsc{\textbf{Methods}}} 
    & AUROC & AP & F1-Max & AUROC & PRO & F1-Max \\
    \hline\hline
    & WinCLIP~\cite{jeong2023winclip} 
    & 70.2& 70.6& 80.7& 94.3& 88.1& 31.3 \\
    & InCTRL~\cite{zhu2024toward} 
    & 74.8& 76.3& 80.4& 93.1& 84.5& 32.3 \\
    & ResAD~\cite{yao2024resad} 
    & 67.7& 67.5& 75.0& 92.7& 84.1& 29.8 \\
    & FoundAD~\cite{zhai2026foundation} 
    & 70.6& 70.1& 84.5& 95.4& 92.5& 37.9 \\
    & RegAD~\cite{huang2022registration} 
    & 69.7& 70.6& 83.6& 92.5& 85.7& 35.5 \\
    \multirow{-6}{*}{\makecell[c]{\texttt{N1}}} & PromptAD~\cite{li2024promptad} 
    & 75.8& 76.2& 83.1& 95.1& 91.4& \textcolor{blue}{\textbf{39.5}} \\
    \hline
    & DRA~\cite{ding2022catching} 
    & 73.4& 75.3& 82.3& 93.9& 92.1& 36.7 \\
    \multirow{-2}{*}{\makecell[c]{\texttt{A1}}} & AHL~\cite{zhu2024anomaly} 
    & 73.3& 76.8& 82.7& 94.3& 92.7& 39.1 \\
    \hline
    & NAGL~\cite{wang2025normal} 
    & \textcolor{blue}{\textbf{77.1}}& \textcolor{blue}{\textbf{79.3}}& \textcolor{blue}{\textbf{84.6}}& \textcolor{blue}{\textbf{96.5}}& \textcolor{blue}{\textbf{92.8}}& 39.0 \\
    \multirow{-2}{*}{\makecell[c]{\texttt{N1A1}}} & \cellcolor{lightyellow}\textbf{IDEAL} (ours) 
    & \textcolor{DeepRed}{\textbf{78.5}}& \textcolor{DeepRed}{\textbf{80.5}}& \textcolor{DeepRed}{\textbf{86.8}}& \textcolor{DeepRed}{\textbf{97.9}}& \textcolor{DeepRed}{\textbf{94.1}}& \textcolor{DeepRed}{\textbf{40.6}} \\
    \hline\hline
    & WinCLIP~\cite{jeong2023winclip} 
    & 70.8& 74.1& 82.5& 94.7& 88.8& 32.4 \\
    & InCTRL~\cite{zhu2024toward} 
    & 75.9& 78.6& 83.4& 93.5& 86.3& 34.4 \\
    & ResAD~\cite{yao2024resad} 
    & 68.4& 69.6& 76.1& 94.4& 86.8& 31.7 \\
    & FoundAD~\cite{zhai2026foundation} 
    & 73.9& 80.9& 84.6& 95.7& 92.6& 42.7 \\
    & RegAD~\cite{huang2022registration} 
    & 71.9& 78.1& \textcolor{blue}{\textbf{86.3}}& 94.5& 86.1& 39.2 \\
    \multirow{-6}{*}{\makecell[c]{\texttt{N2}}} & PromptAD~\cite{li2024promptad} 
    & 76.8& 76.9& 83.8& 96.1& 91.5& 41.6 \\
    \hline
    & NAGL~\cite{wang2025normal} 
    & \textcolor{blue}{\textbf{80.2}}& \textcolor{blue}{\textbf{82.8}}& 86.1& \textcolor{blue}{\textbf{97.3}}& \textcolor{blue}{\textbf{93.5}}& \textcolor{blue}{\textbf{43.3}} \\
    \multirow{-2}{*}{\makecell[c]{\texttt{N2A1}}} & \cellcolor{lightyellow}\textbf{IDEAL} (ours) 
    & \textcolor{DeepRed}{\textbf{81.6}}& \textcolor{DeepRed}{\textbf{84.5}}& \textcolor{DeepRed}{\textbf{88.5}}& \textcolor{DeepRed}{\textbf{98.0}}& \textcolor{DeepRed}{\textbf{94.7}}& \textcolor{DeepRed}{\textbf{44.5}} \\
    \hline\hline
    & WinCLIP~\cite{jeong2023winclip} 
    & 71.0 & 75.1& 83.0& 95.4& 89.3& 33.4 \\
    & InCTRL~\cite{zhu2024toward} 
    & 79.4& 79.1& 83.7& 93.7& 89.0& 35.2 \\
    & ResAD~\cite{yao2024resad} 
    & 70.3& 72.2& 77.9& 95.7& 86.9& 33.3 \\
    & FoundAD~\cite{zhai2026foundation} 
    & 77.6& 81.2& 85.4& 95.8& 92.8& 43.9 \\
    & RegAD~\cite{huang2022registration} 
    & 76.3& 78.6& \textcolor{blue}{\textbf{87.2}}& 95.2& 86.3& 42.3 \\
    \multirow{-6}{*}{\makecell[c]{\texttt{N4}}} & PromptAD~\cite{li2024promptad} 
    & 78.3& 80.9& 85.1& 96.5& 91.6& \textcolor{blue}{\textbf{45.6}} \\
    \hline
    & NAGL~\cite{wang2025normal} 
    & \textcolor{blue}{\textbf{81.4}}& \textcolor{blue}{\textbf{84.3}}& 86.7& \textcolor{blue}{\textbf{97.7}}& \textcolor{blue}{\textbf{93.9}}& 44.6 \\
    \multirow{-2}{*}{\makecell[c]{\texttt{N4A1}}} & \cellcolor{lightyellow}\textbf{IDEAL} (ours) 
    & \textcolor{DeepRed}{\textbf{87.6}}& \textcolor{DeepRed}{\textbf{89.4}}& \textcolor{DeepRed}{\textbf{89.1}}& \textcolor{DeepRed}{\textbf{98.5}}& \textcolor{DeepRed}{\textbf{95.7}}& \textcolor{DeepRed}{\textbf{46.7}} \\
    \hline\hline
    & WinCLIP~\cite{jeong2023winclip} 
    & 74.7& 75.5& 83.3& 95.6& 89.2& 34.3 \\
    & InCTRL~\cite{zhu2024toward} 
    & \textcolor{blue}{\textbf{83.2}}& 84.5& 85.5& 94.1& 89.7& 35.6 \\
    & ResAD~\cite{yao2024resad} 
    & 75.1& 75.7& 81.5& 96.1& 88.1& 34.7 \\
    & FoundAD~\cite{zhai2026foundation} 
    & 80.3& 83.3& 86.5& 96.2& 92.9& \textcolor{blue}{\textbf{47.3}} \\
    & RegAD~\cite{huang2022registration} 
    & 80.8& 79.1& 87.7& 95.7& 88.4& 43.9 \\
    \multirow{-6}{*}{\makecell[c]{\texttt{N8}}} & PromptAD~\cite{li2024promptad} 
    & 80.6& 86.2& 86.4& 96.5& 92.3& 45.9 \\
    \hline
    & DRA~\cite{ding2022catching} 
    & 81.8& 83.9& 82.0& 94.9& 92.5& 44.7 \\
    \multirow{-2}{*}{\makecell[c]{\texttt{A4}}} & AHL~\cite{zhu2024anomaly} 
    & 81.0& 84.5& 83.8& 95.2& 93.7& 46.3 \\
    \hline
    & NAGL~\cite{wang2025normal} 
    & 83.1& \textcolor{blue}{\textbf{86.5}}& \textcolor{blue}{\textbf{88.6}}& \textcolor{blue}{\textbf{97.9}}& \textcolor{blue}{\textbf{94.1}}& 45.7 \\
    \multirow{-2}{*}{\makecell[c]{\texttt{N8A4}}} & \cellcolor{lightyellow}\textbf{IDEAL} (ours) 
    & \textcolor{DeepRed}{\textbf{88.0}}& \textcolor{DeepRed}{\textbf{90.4}}& \textcolor{DeepRed}{\textbf{91.5}}& \textcolor{DeepRed}{\textbf{98.7}}& \textcolor{DeepRed}{\textbf{95.9}}& \textcolor{DeepRed}{\textbf{48.7}} \\
    \hline\thickhline
\end{tabular}
}
\endgroup
\end{table*}

\begin{table*}[t!]
\caption{\textbf{Comparison results on BTAD} under various few-shot AD settings. Best results are in \textcolor{DeepRed}{\textbf{red bold}}, with second best results in \textcolor{blue}{\textbf{blue bold}}.}
\label{tab:tab_add_btad}
\vspace{1.5mm}
\centering
\begingroup
\setlength{\tabcolsep}{10pt} 
\renewcommand{\arraystretch}{1.1} 
\setlength{\arrayrulewidth}{0.1mm} 
\resizebox{\linewidth}{!}{%
\begin{tabular}{cl I I c I c I c I c I c I c}
    \hline\thickhline
    \rowcolor{mygray} 
    & & \multicolumn{3}{c I}{\textsc{\textbf{Image-level}}} 
    & \multicolumn{3}{c}{\textsc{\textbf{Pixel-level}}} \\
    \cline{3-5} \cline{6-8} 
    \rowcolor{mygray} 
    \multirow{-2}{*}{\centering \textsc{\textbf{Setups}}} 
    & \multirow{-2}{*}{\centering \textsc{\textbf{Methods}}} 
    & AUROC & AP & F1-Max & AUROC & PRO & F1-Max \\
    \hline\hline
    & WinCLIP~\cite{jeong2023winclip} 
    & 84.5& 88.1& 75.7& 94.3& 66.5& 46.9 \\
    & InCTRL~\cite{zhu2024toward} 
    & 90.5& 91.1& 86.5& 92.9& 67.8& 54.7 \\
    & ResAD~\cite{yao2024resad} 
    & 78.4& 80.3& 71.7& 92.3& 72.9& 38.7 \\
    & FoundAD~\cite{zhai2026foundation} 
    & 91.4& 91.9& 87.9& 91.9& 73.9& 51.8 \\
    & RegAD~\cite{huang2022registration} 
    & 86.6& 87.8& 85.5& 94.5& 72.8& 45.7 \\
    \multirow{-6}{*}{\makecell[c]{\texttt{N1}}} & PromptAD~\cite{li2024promptad} 
    & 89.6& 90.8& 87.7& 92.4& 72.5& 59.5 \\
    \hline
    & DRA~\cite{ding2022catching} 
    & 90.5& 91.7& 91.3& 93.5& 72.3& 53.4 \\
    \multirow{-2}{*}{\makecell[c]{\texttt{A1}}} & AHL~\cite{zhu2024anomaly} 
    & 90.9& 92.3& 91.9& 93.8& 73.2& 57.8 \\
    \hline
    & NAGL~\cite{wang2025normal} 
    & \textcolor{blue}{\textbf{92.0}}& \textcolor{blue}{\textbf{95.6}}& \textcolor{blue}{\textbf{92.2}}& \textcolor{blue}{\textbf{95.8}}& \textcolor{blue}{\textbf{75.5}}& \textcolor{blue}{\textbf{60.6}} \\
    \multirow{-2}{*}{\makecell[c]{\texttt{N1A1}}} & \cellcolor{lightyellow}\textbf{IDEAL} (ours) 
    & \textcolor{DeepRed}{\textbf{93.4}}& \textcolor{DeepRed}{\textbf{96.0}}& \textcolor{DeepRed}{\textbf{93.1}}& \textcolor{DeepRed}{\textbf{97.0}}& \textcolor{DeepRed}{\textbf{76.9}}& \textcolor{DeepRed}{\textbf{62.3}} \\
    \hline\hline
    & WinCLIP~\cite{jeong2023winclip} 
    & 85.0& 88.9& 76.4& 95.5& 67.1& 48.4 \\
    & InCTRL~\cite{zhu2024toward} 
    & 93.4& 92.9& 88.1& 93.0& 69.2& 56.1 \\
    & ResAD~\cite{yao2024resad} 
    & 80.1& 81.4& 72.6& 92.9& 73.0& 41.4 \\
    & FoundAD~\cite{zhai2026foundation} 
    & \textcolor{blue}{\textbf{94.0}}& 92.9& 88.5& 93.2& 75.7& 53.6 \\
    & RegAD~\cite{huang2022registration} 
    & 90.8& 89.5& 86.9& 95.6& 73.4& 46.4 \\
    \multirow{-6}{*}{\makecell[c]{\texttt{N2}}} & PromptAD~\cite{li2024promptad} 
    & 90.3& 91.5& 88.5& 92.6& 72.8& 60.5 \\
    \hline
    & NAGL~\cite{wang2025normal} 
    & 93.5& \textcolor{blue}{\textbf{96.4}}& \textcolor{blue}{\textbf{92.5}}& \textcolor{blue}{\textbf{96.3}}& \textcolor{blue}{\textbf{76.7}}& \textcolor{blue}{\textbf{62.2}} \\
    \multirow{-2}{*}{\makecell[c]{\texttt{N2A1}}} & \cellcolor{lightyellow}\textbf{IDEAL} (ours) 
    & \textcolor{DeepRed}{\textbf{95.6}}& \textcolor{DeepRed}{\textbf{97.2}}& \textcolor{DeepRed}{\textbf{93.2}}& \textcolor{DeepRed}{\textbf{97.2}}& \textcolor{DeepRed}{\textbf{77.0}}& \textcolor{DeepRed}{\textbf{63.7}} \\
    \hline\hline
    & WinCLIP~\cite{jeong2023winclip} 
    & 87.4& 90.4& 77.4& 95.9& 69.5& 49.0 \\
    & InCTRL~\cite{zhu2024toward} 
    & 93.9& 93.5& 89.1& 93.6& 72.6& 57.8 \\
    & ResAD~\cite{yao2024resad} 
    & 81.5& 82.1& 74.9& 95.1& 76.5& 44.5 \\
    & FoundAD~\cite{zhai2026foundation} 
    & \textcolor{blue}{\textbf{94.3}}& 93.4& 89.7& 94.0& \textcolor{blue}{\textbf{77.4}}& 60.4 \\
    & RegAD~\cite{huang2022registration} 
    & 91.7& 91.2& 90.6& 96.3& 75.8& 51.7 \\
    \multirow{-6}{*}{\makecell[c]{\texttt{N4}}} & PromptAD~\cite{li2024promptad} 
    & 90.8& 92.3& 89.7& 93.2& 74.5& 61.3 \\
    \hline
    & NAGL~\cite{wang2025normal} 
    & 93.7& \textcolor{blue}{\textbf{96.7}}& \textcolor{blue}{\textbf{93.3}}& \textcolor{blue}{\textbf{96.5}}& 76.5& \textcolor{blue}{\textbf{62.3}} \\
    \multirow{-2}{*}{\makecell[c]{\texttt{N4A1}}} & \cellcolor{lightyellow}\textbf{IDEAL} (ours) 
    & \textcolor{DeepRed}{\textbf{95.9}}& \textcolor{DeepRed}{\textbf{97.8}}& \textcolor{DeepRed}{\textbf{93.5}}& \textcolor{DeepRed}{\textbf{97.5}}& \textcolor{DeepRed}{\textbf{77.8}}& \textcolor{DeepRed}{\textbf{65.1}} \\
    \hline\hline
    & WinCLIP~\cite{jeong2023winclip} 
    & 89.6& 91.4& 78.2& 96.0& 70.2& 49.4 \\
    & InCTRL~\cite{zhu2024toward} 
    & 94.3& 96.5& 91.4& 94.9& 77.9& 59.2 \\
    & ResAD~\cite{yao2024resad} 
    & 83.6& 83.0& 75.1& 95.5& 76.2& 47.3 \\
    & FoundAD~\cite{zhai2026foundation} 
    & \textcolor{blue}{\textbf{94.6}}& 93.9& 90.2& 95.1& \textcolor{blue}{\textbf{78.5}}& 60.6 \\
    & RegAD~\cite{huang2022registration} 
    & 92.2& 93.0& 91.9& \textcolor{blue}{\textbf{96.5}}& 76.6& 53.3 \\
    \multirow{-6}{*}{\makecell[c]{\texttt{N8}}} & PromptAD~\cite{li2024promptad} 
    & 91.2& 92.6& 90.9& 94.1& 76.1& 61.8 \\
    \hline
    & DRA~\cite{ding2022catching} 
    & 93.0& 92.2& 91.8& 93.9& 74.1& 61.3 \\
    \multirow{-2}{*}{\makecell[c]{\texttt{A4}}} & AHL~\cite{zhu2024anomaly} 
    & 93.3& 92.0& 92.4& 94.5& 75.8& 61.9 \\
    \hline
    & NAGL~\cite{wang2025normal} 
    & 94.0& \textcolor{blue}{\textbf{97.6}}& \textcolor{blue}{\textbf{93.8}}& \textcolor{blue}{\textbf{96.5}}& 77.7& \textcolor{blue}{\textbf{62.8}} \\
    \multirow{-2}{*}{\makecell[c]{\texttt{N8A4}}} & \cellcolor{lightyellow}\textbf{IDEAL} (ours) 
    & \textcolor{DeepRed}{\textbf{96.2}}& \textcolor{DeepRed}{\textbf{98.1}}& \textcolor{DeepRed}{\textbf{93.9}}& \textcolor{DeepRed}{\textbf{97.7}}& \textcolor{DeepRed}{\textbf{78.9}}& \textcolor{DeepRed}{\textbf{66.3}} \\
    \hline\thickhline
\end{tabular}
}
\endgroup
\end{table*}

\begin{table*}[t!]
\caption{\textbf{Comparison results on BraTS} under various few-shot AD settings. Best results are in \textcolor{DeepRed}{\textbf{red bold}}, with second best results in \textcolor{blue}{\textbf{blue bold}}.}
\label{tab:tab_add_brats}
\vspace{1.5mm}
\centering
\begingroup
\setlength{\tabcolsep}{10pt} 
\renewcommand{\arraystretch}{1.1} 
\setlength{\arrayrulewidth}{0.1mm} 
\resizebox{\linewidth}{!}{%
\begin{tabular}{cl I I c I c I c I c I c I c}
    \hline\thickhline
    \rowcolor{mygray} 
    & & \multicolumn{3}{c I}{\textsc{\textbf{Image-level}}} 
    & \multicolumn{3}{c}{\textsc{\textbf{Pixel-level}}} \\
    \cline{3-5} \cline{6-8} 
    \rowcolor{mygray} 
    \multirow{-2}{*}{\centering \textsc{\textbf{Setups}}} 
    & \multirow{-2}{*}{\centering \textsc{\textbf{Methods}}} 
    & AUROC & AP & F1-Max & AUROC & PRO & F1-Max \\
    \hline\hline
    & WinCLIP~\cite{jeong2023winclip} 
    & 68.8& 89.4& 90.9& 91.0& 67.6& 25.9 \\
    & InCTRL~\cite{zhu2024toward} 
    & 69.5& 89.2& 90.3& 93.1& 70.0& \textcolor{blue}{\textbf{36.7}} \\
    & ResAD~\cite{yao2024resad} 
    & 59.7& 83.1& 90.6& 89.5& 63.9& 23.8 \\
    & FoundAD~\cite{zhai2026foundation} 
    & 62.7& \textcolor{DeepRed}{\textbf{90.3}}& 90.7& 87.8& 65.3& 29.4 \\
    & RegAD~\cite{huang2022registration} 
    & 64.2& 88.3& \textcolor{DeepRed}{\textbf{91.6}}& 92.2& 72.8& 35.1 \\
    \multirow{-6}{*}{\makecell[c]{\texttt{N1}}} & PromptAD~\cite{li2024promptad} 
    & 68.7& 89.1& 90.7& 91.9& 70.9& 34.7 \\
    \hline
    & DRA~\cite{ding2022catching} 
    & 67.2& 83.9& 90.6& 87.5& 71.6& 28.7 \\
    \multirow{-2}{*}{\makecell[c]{\texttt{A1}}} & AHL~\cite{zhu2024anomaly} 
    & \textcolor{blue}{\textbf{70.4}}& 88.3& 91.1& 90.4& 71.1& 32.73 \\
    \hline
    & NAGL~\cite{wang2025normal} 
    & 67.7& 88.5& 91.2& \textcolor{blue}{\textbf{93.7}}& \textcolor{blue}{\textbf{75.2}}& 33.8 \\
    \multirow{-2}{*}{\makecell[c]{\texttt{N1A1}}} & \cellcolor{lightyellow}\textbf{IDEAL} (ours) 
    & \textcolor{DeepRed}{\textbf{74.9}}& \textcolor{blue}{\textbf{90.2}}& \textcolor{blue}{\textbf{91.4}}& \textcolor{DeepRed}{\textbf{94.8}}& \textcolor{DeepRed}{\textbf{76.3}}& \textcolor{DeepRed}{\textbf{40.1}} \\
    \hline\hline
    & WinCLIP~\cite{jeong2023winclip} 
    & 74.8& \textcolor{blue}{\textbf{92.4}}& 91.1& 92.7& 72.0& 30.8 \\
    & InCTRL~\cite{zhu2024toward} 
    & \textcolor{blue}{\textbf{75.0}}& 91.0& 90.3& 94.1& 74.9& \textcolor{blue}{\textbf{41.3}} \\
    & ResAD~\cite{yao2024resad} 
    & 65.4& 86.7& 90.8& 92.2& 69.6& 35.4 \\
    & FoundAD~\cite{zhai2026foundation} 
    & 72.5& 90.8& 91.1& 91.9& 69.4& 33.1 \\
    & RegAD~\cite{huang2022registration} 
    & 70.7& 89.1& \textcolor{blue}{\textbf{91.8}}& 94.1& 73.4& 39.5 \\
    \multirow{-6}{*}{\makecell[c]{\texttt{N2}}} & PromptAD~\cite{li2024promptad} 
    & 74.4& 90.5& 91.0& 93.7& 74.2& 41.1 \\
    \hline
    & NAGL~\cite{wang2025normal} 
    & 74.6& 91.1& 91.7& \textcolor{blue}{\textbf{94.3}}& \textcolor{blue}{\textbf{77.2}}& 39.1 \\
    \multirow{-2}{*}{\makecell[c]{\texttt{N2A1}}} & \cellcolor{lightyellow}\textbf{IDEAL} (ours) 
    & \textcolor{DeepRed}{\textbf{78.5}}& \textcolor{DeepRed}{\textbf{93.4}}& \textcolor{DeepRed}{\textbf{91.9}}& \textcolor{DeepRed}{\textbf{95.6}}& \textcolor{DeepRed}{\textbf{77.4}}& \textcolor{DeepRed}{\textbf{43.9}} \\
    \hline\hline
    & WinCLIP~\cite{jeong2023winclip} 
    & 75.7& 92.4& 91.5& 93.2& 72.6& 32.8 \\
    & InCTRL~\cite{zhu2024toward} 
    & \textcolor{blue}{\textbf{77.8}}& \textcolor{blue}{\textbf{93.2}}& 91.2& 93.8& 75.8& 41.6 \\
    & ResAD~\cite{yao2024resad} 
    & 69.6& 87.8& 91.7& 93.0& 72.2& 40.3 \\
    & FoundAD~\cite{zhai2026foundation} 
    & 73.5& 91.2& 91.7& 92.8& 72.6& 39.3 \\
    & RegAD~\cite{huang2022registration} 
    & 75.8& 89.6& \textcolor{blue}{\textbf{92.0}}& \textcolor{blue}{\textbf{94.5}}& 76.5& 40.1 \\
    \multirow{-6}{*}{\makecell[c]{\texttt{N4}}} & PromptAD~\cite{li2024promptad} 
    & 74.5& 91.1& 90.7& 94.0& 74.9& \textcolor{blue}{\textbf{42.3}} \\
    \hline
    & NAGL~\cite{wang2025normal} 
    & 74.9& 91.5& 91.8& 94.4& \textcolor{blue}{\textbf{78.6}}& 41.8 \\
    \multirow{-2}{*}{\makecell[c]{\texttt{N4A1}}} & \cellcolor{lightyellow}\textbf{IDEAL} (ours) 
    & \textcolor{DeepRed}{\textbf{80.5}}& \textcolor{DeepRed}{\textbf{94.1}}& \textcolor{DeepRed}{\textbf{92.1}}& \textcolor{DeepRed}{\textbf{95.8}}& \textcolor{DeepRed}{\textbf{79.2}}& \textcolor{DeepRed}{\textbf{46.6}} \\
    \hline\hline
    & WinCLIP~\cite{jeong2023winclip} 
    & 76.9& 92.6& 92.0& 93.8& 73.1& 33.1 \\
    & InCTRL~\cite{zhu2024toward} 
    & \textcolor{blue}{\textbf{79.5}}& \textcolor{blue}{\textbf{93.9}}& 91.6& 94.9& 76.6& \textcolor{blue}{\textbf{47.1}} \\
    & ResAD~\cite{yao2024resad} 
    & 71.7& 89.2& 91.9& 94.3& 73.2& 42.4 \\
    & FoundAD~\cite{zhai2026foundation} 
    & 76.7& 91.8& \textcolor{blue}{\textbf{92.1}}& 93.5& 73.4& 40.8  \\
    & RegAD~\cite{huang2022registration} 
    & 76.3& 89.8& 92.0& 95.0& \textcolor{blue}{\textbf{79.2}}& 45.2 \\
    \multirow{-6}{*}{\makecell[c]{\texttt{N8}}} & PromptAD~\cite{li2024promptad} 
    & 75.6& 92.8& 91.0& 94.6& 76.4& 43.9 \\
    \hline
    & DRA~\cite{ding2022catching} 
    & 73.5& 87.6& 91.7& 89.9& 72.7& 38.1 \\
    \multirow{-2}{*}{\makecell[c]{\texttt{A4}}} & AHL~\cite{zhu2024anomaly} 
    & 72.9& 91.1& 91.3& 90.3& 75.2& 35.4 \\
    \hline
    & NAGL~\cite{wang2025normal} 
    & 75.1& 92.4& 92.0 & \textcolor{blue}{\textbf{95.2}}& 78.9& 41.3 \\
    \multirow{-2}{*}{\makecell[c]{\texttt{N8A4}}} & \cellcolor{lightyellow}\textbf{IDEAL} (ours) 
    & \textcolor{DeepRed}{\textbf{82.8}}& \textcolor{DeepRed}{\textbf{94.2}}& \textcolor{DeepRed}{\textbf{92.4}}& \textcolor{DeepRed}{\textbf{96.9}}& \textcolor{DeepRed}{\textbf{79.5}}& \textcolor{DeepRed}{\textbf{53.1}} \\
    \hline\thickhline
\end{tabular}
}
\endgroup
\end{table*}

\begin{table*}[t!]
\caption{\textbf{Comparison results on Liver} under various few-shot AD settings. Best results are in \textcolor{DeepRed}{\textbf{red bold}}, with second best results in \textcolor{blue}{\textbf{blue bold}}.}
\label{tab:tab_add_liver}
\vspace{1.5mm}
\centering
\begingroup
\setlength{\tabcolsep}{10pt} 
\renewcommand{\arraystretch}{1.1} 
\setlength{\arrayrulewidth}{0.1mm} 
\resizebox{\linewidth}{!}{%
\begin{tabular}{cl I I c I c I c I c I c I c}
    \hline\thickhline
    \rowcolor{mygray} 
    & & \multicolumn{3}{c I}{\textsc{\textbf{Image-level}}} 
    & \multicolumn{3}{c}{\textsc{\textbf{Pixel-level}}} \\
    \cline{3-5} \cline{6-8} 
    \rowcolor{mygray} 
    \multirow{-2}{*}{\centering \textsc{\textbf{Setups}}} 
    & \multirow{-2}{*}{\centering \textsc{\textbf{Methods}}} 
    & AUROC & AP & F1-Max & AUROC & PRO & F1-Max \\
    \hline\hline
    & WinCLIP~\cite{jeong2023winclip} 
    & 59.5& 57.2& 63.2& \textcolor{blue}{\textbf{96.2}}& \textcolor{blue}{\textbf{92.0}}& 21.9 \\
    & InCTRL~\cite{zhu2024toward} 
    & 59.7& 55.1& 60.2& 96.0& 90.8& 22.1 \\
    & ResAD~\cite{yao2024resad} 
    & 53.3& 55.8& 57.8& 92.7& 88.3& 11.4 \\
    & FoundAD~\cite{zhai2026foundation} 
    & 58.7& 57.1& 63.7& 93.8& 89.1& 19.2 \\
    & RegAD~\cite{huang2022registration} 
    & 59.0& \textcolor{DeepRed}{\textbf{57.9}}& 61.9& 94.0& 85.1& 17.2 \\
    \multirow{-6}{*}{\makecell[c]{\texttt{N1}}} & PromptAD~\cite{li2024promptad} 
    & 58.3& 54.6& \textcolor{blue}{\textbf{65.6}}& 94.3& 83.4& 18.5 \\
    \hline
    & DRA~\cite{ding2022catching} 
    & 56.4& 54.4& 63.7& 93.5& 83.7& \textcolor{blue}{\textbf{24.5}} \\
    \multirow{-2}{*}{\makecell[c]{\texttt{A1}}} & AHL~\cite{zhu2024anomaly} 
    & \textcolor{blue}{\textbf{60.2}}& 55.8& 64.1& 93.7& 80.1& 21.6 \\
    \hline
    & NAGL~\cite{wang2025normal} 
    & 59.9& 57.6& 64.0& 94.8& 87.9& 17.5 \\
    \multirow{-2}{*}{\makecell[c]{\texttt{N1A1}}} & \cellcolor{lightyellow}\textbf{IDEAL} (ours) 
    & \textcolor{DeepRed}{\textbf{62.7}}& \textcolor{blue}{\textbf{57.8}}& \textcolor{DeepRed}{\textbf{65.9}}& \textcolor{DeepRed}{\textbf{97.2}}& \textcolor{DeepRed}{\textbf{93.5}}& \textcolor{DeepRed}{\textbf{25.6}} \\
    \hline\hline
    & WinCLIP~\cite{jeong2023winclip} 
    & 61.6& 57.6& 66.6& \textcolor{blue}{\textbf{96.5}}& \textcolor{blue}{\textbf{92.1}}& 22.4 \\
    & InCTRL~\cite{zhu2024toward} 
    & 63.7& 58.7& \textcolor{blue}{\textbf{66.8}}& 96.1& 90.8& \textcolor{blue}{\textbf{26.5}} \\
    & ResAD~\cite{yao2024resad} 
    & 58.6& 60.3& 64.6& 92.9& 89.5& 15.8 \\
    & FoundAD~\cite{zhai2026foundation} 
    & 63.9& 59.7& 65.9& 94.0& 90.3& 21.6 \\
    & RegAD~\cite{huang2022registration} 
    & 62.1& 58.4& 62.5& 94.2& 85.7& 17.7 \\
    \multirow{-6}{*}{\makecell[c]{\texttt{N2}}} & PromptAD~\cite{li2024promptad} 
    & 62.5& 58.6& 66.1& 94.4& 84.2& 19.7 \\
    \hline
    & NAGL~\cite{wang2025normal} 
    & \textcolor{blue}{\textbf{67.2}}& \textcolor{blue}{\textbf{66.3}}& 64.5& 95.0& 90.3& 18.6 \\
    \multirow{-2}{*}{\makecell[c]{\texttt{N2A1}}} & \cellcolor{lightyellow}\textbf{IDEAL} (ours) 
    & \textcolor{DeepRed}{\textbf{72.4}}& \textcolor{DeepRed}{\textbf{70.3}}& \textcolor{DeepRed}{\textbf{68.1}}& \textcolor{DeepRed}{\textbf{97.7}}& \textcolor{DeepRed}{\textbf{94.9}}& \textcolor{DeepRed}{\textbf{26.9}} \\
    \hline\hline
    & WinCLIP~\cite{jeong2023winclip} 
    & 63.7& 58.0& 67.4& \textcolor{blue}{\textbf{96.8}}& \textcolor{blue}{\textbf{92.5}}& 24.1 \\
    & InCTRL~\cite{zhu2024toward} 
    & 67.9& 63.2& \textcolor{blue}{\textbf{69.5}}& 96.3& 92.1& \textcolor{blue}{\textbf{27.2}} \\
    & ResAD~\cite{yao2024resad} 
    & 65.4& 64.5& 67.3& 94.7& 91.2& 18.1 \\
    & FoundAD~\cite{zhai2026foundation} 
    & 68.7& 64.1& 66.0& 94.2& 90.6& 21.9 \\
    & RegAD~\cite{huang2022registration} 
    & 69.6& 59.3& 63.2& 95.0& 89.3& 25.2 \\
    \multirow{-6}{*}{\makecell[c]{\texttt{N4}}} & PromptAD~\cite{li2024promptad} 
    & 68.9& 59.1& 68.1& 94.2& 84.9& 20.1 \\
    \hline
    & NAGL~\cite{wang2025normal} 
    & \textcolor{blue}{\textbf{70.4}}& \textcolor{blue}{\textbf{67.7}}& 68.0& 95.2& 91.4& 19.1 \\
    \multirow{-2}{*}{\makecell[c]{\texttt{N4A1}}} & \cellcolor{lightyellow}\textbf{IDEAL} (ours) 
    & \textcolor{DeepRed}{\textbf{73.9}}& \textcolor{DeepRed}{\textbf{70.8}}& \textcolor{DeepRed}{\textbf{71.5}}& \textcolor{DeepRed}{\textbf{97.7}}& \textcolor{DeepRed}{\textbf{95.3}}& \textcolor{DeepRed}{\textbf{28.4}} \\
    \hline\hline
    & WinCLIP~\cite{jeong2023winclip} 
    & 65.1& 58.4& 68.6& \textcolor{blue}{\textbf{96.8}}& \textcolor{blue}{\textbf{93.1}}& 26.2 \\
    & InCTRL~\cite{zhu2024toward} 
    & 68.1& 64.8& 70.0& 96.5& 92.4& \textcolor{blue}{\textbf{27.4}} \\
    & ResAD~\cite{yao2024resad} 
    & 69.7& 65.1& 67.6& 96.5& 91.8& 19.6 \\
    & FoundAD~\cite{zhai2026foundation} 
    & 70.9& 67.0& 69.2& 94.3& 91.1& 23.4 \\
    & RegAD~\cite{huang2022registration} 
    & 73.7& 61.1& 63.5& 95.3& 90.4& 27.3 \\
    \multirow{-6}{*}{\makecell[c]{\texttt{N8}}} & PromptAD~\cite{li2024promptad} 
    & 73.5& 70.2& 70.7& 94.9& 84.7& 21.7 \\
    \hline
    & DRA~\cite{ding2022catching} 
    & 70.1& 69.9& 65.9& 93.8& 85.9& 25.2 \\
    \multirow{-2}{*}{\makecell[c]{\texttt{A4}}} & AHL~\cite{zhu2024anomaly} 
    & 72.2& 68.5& 69.7& 93.1& 89.7& 23.9 \\
    \hline
    & NAGL~\cite{wang2025normal} 
    & \textcolor{blue}{\textbf{74.9}}& \textcolor{blue}{\textbf{73.9}}& \textcolor{blue}{\textbf{70.8}}& 95.5& 92.5& 19.7 \\
    \multirow{-2}{*}{\makecell[c]{\texttt{N8A4}}} & \cellcolor{lightyellow}\textbf{IDEAL} (ours) 
    & \textcolor{DeepRed}{\textbf{77.4}}& \textcolor{DeepRed}{\textbf{74.0}}& \textcolor{DeepRed}{\textbf{71.7}}& \textcolor{DeepRed}{\textbf{97.9}}& \textcolor{DeepRed}{\textbf{95.5}}& \textcolor{DeepRed}{\textbf{30.1}} \\
    \hline\thickhline
\end{tabular}
}
\endgroup
\end{table*}

\begin{table*}[t!]
\caption{\textbf{Comparison results on RESC} under various few-shot AD settings. Best results are in \textcolor{DeepRed}{\textbf{red bold}}, with second best results in \textcolor{blue}{\textbf{blue bold}}.}
\label{tab:tab_add_resc}
\vspace{1.5mm}
\centering
\begingroup
\setlength{\tabcolsep}{10pt} 
\renewcommand{\arraystretch}{1.1} 
\setlength{\arrayrulewidth}{0.1mm} 
\resizebox{\linewidth}{!}{%
\begin{tabular}{cl I I c I c I c I c I c I c}
    \hline\thickhline
    \rowcolor{mygray} 
    & & \multicolumn{3}{c I}{\textsc{\textbf{Image-level}}} 
    & \multicolumn{3}{c}{\textsc{\textbf{Pixel-level}}} \\
    \cline{3-5} \cline{6-8} 
    \rowcolor{mygray} 
    \multirow{-2}{*}{\centering \textsc{\textbf{Setups}}} 
    & \multirow{-2}{*}{\centering \textsc{\textbf{Methods}}} 
    & AUROC & AP & F1-Max & AUROC & PRO & F1-Max \\
    \hline\hline
    & WinCLIP~\cite{jeong2023winclip} 
    & 69.2& 65.9& 64.0& 90.8& 72.8& 42.1 \\
    & InCTRL~\cite{zhu2024toward} 
    & \textcolor{blue}{\textbf{80.0}}& 72.8& \textcolor{blue}{\textbf{74.5}}& 91.5& \textcolor{blue}{\textbf{75.8}}& 44.2 \\
    & ResAD~\cite{yao2024resad} 
    & 64.1& 57.7& 61.7& 80.1& 59.2& 28.3 \\
    & FoundAD~\cite{zhai2026foundation} 
    & 78.5& 73.2& 72.9& 87.1& 61.1& 45.9 \\
    & RegAD~\cite{huang2022registration} 
    & 76.1& 72.6& 71.2& 88.3& 69.8& 41.9 \\
    \multirow{-6}{*}{\makecell[c]{\texttt{N1}}} & PromptAD~\cite{li2024promptad} 
    & 77.4& 75.3& 68.6& 91.2& 69.3& \textcolor{blue}{\textbf{46.8}} \\
    \hline
    & DRA~\cite{ding2022catching} 
    & 76.4& 74.7& 69.6& 86.4& 69.8& 43.9 \\
    \multirow{-2}{*}{\makecell[c]{\texttt{A1}}} & AHL~\cite{zhu2024anomaly} 
    & 75.5& 74.6& 70.4& 85.8& 68.9& 44.9 \\
    \hline
    & NAGL~\cite{wang2025normal} 
    & 78.9& \textcolor{blue}{\textbf{75.9}}& 71.6& \textcolor{blue}{\textbf{91.7}}& 73.4& 46.6 \\
    \multirow{-2}{*}{\makecell[c]{\texttt{N1A1}}} & \cellcolor{lightyellow}\textbf{IDEAL} (ours) 
    & \textcolor{DeepRed}{\textbf{82.8}}& \textcolor{DeepRed}{\textbf{79.2}}& \textcolor{DeepRed}{\textbf{75.9}}& \textcolor{DeepRed}{\textbf{92.5}}& \textcolor{DeepRed}{\textbf{77.0}}& \textcolor{DeepRed}{\textbf{48.1}} \\
    \hline\hline
    & WinCLIP~\cite{jeong2023winclip} 
    & 73.7& 66.0& 64.5& 90.9& 73.6& 42.9 \\
    & InCTRL~\cite{zhu2024toward} 
    & 80.3& 74.8& \textcolor{blue}{\textbf{75.0}}& 92.1& \textcolor{blue}{\textbf{76.8}}& 46.8 \\
    & ResAD~\cite{yao2024resad} 
    & 69.3& 60.5& 62.4& 81.3& 63.1& 31.4 \\
    & FoundAD~\cite{zhai2026foundation} 
    & 79.0& 77.1& 73.8& 90.2& 62.7& 48.3 \\
    & RegAD~\cite{huang2022registration} 
    & 77.4& 73.4& 73.6& 89.7& 71.2& 46.4 \\
    \multirow{-6}{*}{\makecell[c]{\texttt{N2}}} & PromptAD~\cite{li2024promptad} 
    & 78.2& 75.6& 68.7& 91.6& 70.5& 47.2 \\
    \hline
    & NAGL~\cite{wang2025normal} 
    & \textcolor{blue}{\textbf{80.4}}& \textcolor{blue}{\textbf{81.7}}& 73.7& \textcolor{blue}{\textbf{92.2}}& 75.6& \textcolor{blue}{\textbf{49.1}} \\
    \multirow{-2}{*}{\makecell[c]{\texttt{N2A1}}} & \cellcolor{lightyellow}\textbf{IDEAL} (ours) 
    & \textcolor{DeepRed}{\textbf{84.4}}& \textcolor{DeepRed}{\textbf{83.2}}& \textcolor{DeepRed}{\textbf{76.4}}& \textcolor{DeepRed}{\textbf{93.5}}& \textcolor{DeepRed}{\textbf{77.8}}& \textcolor{DeepRed}{\textbf{50.7}} \\
    \hline\hline
    & WinCLIP~\cite{jeong2023winclip} 
    & 75.2& 70.5& 68.4& 91.3& 74.1& 45.8 \\
    & InCTRL~\cite{zhu2024toward} 
    & \textcolor{blue}{\textbf{86.5}}& 79.7& \textcolor{blue}{\textbf{80.2}}& \textcolor{blue}{\textbf{92.8}}& \textcolor{blue}{\textbf{79.2}}& 49.1 \\
    & ResAD~\cite{yao2024resad} 
    & 73.7& 67.1& 62.7& 85.3& 65.2& 35.3 \\
    & FoundAD~\cite{zhai2026foundation} 
    & 81.5& \textcolor{blue}{\textbf{83.1}}& 74.4& 90.8& 68.3& 49.2 \\
    & RegAD~\cite{huang2022registration} 
    & 83.3& 75.9& 74.2& 91.4& 74.6& 47.3 \\
    \multirow{-6}{*}{\makecell[c]{\texttt{N4}}} & PromptAD~\cite{li2024promptad} 
    & 81.3& 75.7& 68.9& 92.2& 71.5& 48.1 \\
    \hline
    & NAGL~\cite{wang2025normal} 
    & 84.8& 82.0& 75.6& 92.6& 79.0& \textcolor{blue}{\textbf{50.7}} \\
    \multirow{-2}{*}{\makecell[c]{\texttt{N4A1}}} & \cellcolor{lightyellow}\textbf{IDEAL} (ours) 
    & \textcolor{DeepRed}{\textbf{89.4}}& \textcolor{DeepRed}{\textbf{85.6}}& \textcolor{DeepRed}{\textbf{81.3}}& \textcolor{DeepRed}{\textbf{94.9}}& \textcolor{DeepRed}{\textbf{79.4}}& \textcolor{DeepRed}{\textbf{51.2}} \\
    \hline\hline
    & WinCLIP~\cite{jeong2023winclip} 
    & 76.7& 75.6& 69.4& 92.4& 74.6& 46.3 \\
    & InCTRL~\cite{zhu2024toward} 
    & \textcolor{blue}{\textbf{86.8}}& 80.7& \textcolor{blue}{\textbf{81.0}}& 93.6& \textcolor{blue}{\textbf{79.6}}& 50.1 \\
    & ResAD~\cite{yao2024resad} 
    & 74.8& 71.1& 63.3& 87.6& 69.3& 38.6 \\
    & FoundAD~\cite{zhai2026foundation} 
    & 84.5& 83.6& 78.3& 91.3& 69.5& \textcolor{blue}{\textbf{52.1}} \\
    & RegAD~\cite{huang2022registration} 
    & 83.9& 78.5& 75.9& 92.6& 72.8& 49.8 \\
    \multirow{-6}{*}{\makecell[c]{\texttt{N8}}} & PromptAD~\cite{li2024promptad} 
    & 83.3& 76.1& 72.8& 93.2& 73.9& 49.3 \\
    \hline
    & DRA~\cite{ding2022catching} 
    & 82.1& 78.4& 70.1& 89.3& 71.2& 50.4 \\
    \multirow{-2}{*}{\makecell[c]{\texttt{A4}}} & AHL~\cite{zhu2024anomaly} 
    & 85.7& 78.9& 70.4& 89.8& 73.4& 49.9 \\
    \hline
    & NAGL~\cite{wang2025normal} 
    & 85.2& \textcolor{blue}{\textbf{85.1}}& 77.6& \textcolor{blue}{\textbf{94.0}}& 79.2& 51.4 \\
    \multirow{-2}{*}{\makecell[c]{\texttt{N8A4}}} & \cellcolor{lightyellow}\textbf{IDEAL} (ours) 
    & \textcolor{DeepRed}{\textbf{90.6}}& \textcolor{DeepRed}{\textbf{86.8}}& \textcolor{DeepRed}{\textbf{83.7}}& \textcolor{DeepRed}{\textbf{95.6}}& \textcolor{DeepRed}{\textbf{79.9}}& \textcolor{DeepRed}{\textbf{56.8}} \\
    \hline\thickhline
\end{tabular}
}
\endgroup
\end{table*}

\clearpage
\newpage
\section*{NeurIPS Paper Checklist}

\begin{enumerate}

\item {\bf Claims}
    \item[] Question: Do the main claims made in the abstract and introduction accurately reflect the paper's contributions and scope?
    \item[] Answer: \answerYes{} 
    \item[] Justification: Our contributions are summarized in the Abstract and Introduction (Sec.~\ref{sec:myintro}) sections. The experiments (Sec.~\ref{sec:myexp}) also confirm our contributions.

\item {\bf Limitations}
    \item[] Question: Does the paper discuss the limitations of the work performed by the authors?
    \item[] Answer: \answerYes{} 
    \item[] Justification: Please see Appendix~\ref{sec:lim} for limitations.

\item {\bf Theory assumptions and proofs}
    \item[] Question: For each theoretical result, does the paper provide the full set of assumptions and a complete (and correct) proof?
    \item[] Answer: \answerYes{} 
    \item[] Justification: Please see Appendix~\ref{sec:bound} for the generalization analysis and proof.

    \item {\bf Experimental result reproducibility}
    \item[] Question: Does the paper fully disclose all the information needed to reproduce the main experimental results of the paper to the extent that it affects the main claims and/or conclusions of the paper (regardless of whether the code and data are provided or not)?
    \item[] Answer: \answerYes{} 
    \item[] Justification: Please see Implementation Details in Sec.~\ref{sec:myexp} and Appendix~\ref{sec:edetails}.

\item {\bf Open access to data and code}
    \item[] Question: Does the paper provide open access to the data and code, with sufficient instructions to faithfully reproduce the main experimental results, as described in supplemental material?
    \item[] Answer: \answerYes{} 
    \item[] Justification: The code is provided in the supplementary material. Our code and data are available at \href{https://github.com/mala-lab/IDEAL}{https://github.com/mala-lab/IDEAL}.

\item {\bf Experimental setting/details}
    \item[] Question: Does the paper specify all the training and test details (e.g., data splits, hyperparameters, how they were chosen, type of optimizer) necessary to understand the results?
    \item[] Answer: \answerYes{} 
    \item[] Justification: Please see Sec.~\ref{sec:exp_setups} and Appendix~\ref{sec:edetails}.

\item {\bf Experiment statistical significance}
    \item[] Question: Does the paper report error bars suitably and correctly defined or other appropriate information about the statistical significance of the experiments?
    \item[] Answer: \answerYes{} 
    \item[] Justification: As mentioned in Sec.~\ref{sec:exp_setups}, we report results averaged across three independent runs with different random seeds.

\item {\bf Experiments compute resources}
    \item[] Question: For each experiment, does the paper provide sufficient information on the computer resources (type of compute workers, memory, time of execution) needed to reproduce the experiments?
    \item[] Answer: \answerYes{} 
    \item[] Justification: Please see Implementation Details in Sec.~\ref{sec:exp_setups} and Efficiency Comparison in Sec.~\ref{sec:valida}.
    
\item {\bf Code of ethics}
    \item[] Question: Does the research conducted in the paper conform, in every respect, with the NeurIPS Code of Ethics \url{https://neurips.cc/public/EthicsGuidelines}?
    \item[] Answer: \answerYes{} 
    \item[] Justification: We have read NeurIPS Code of Ethics and confirm that our research complies with these ethical standards.

\item {\bf Broader impacts}
    \item[] Question: Does the paper discuss both potential positive societal impacts and negative societal impacts of the work performed?
    \item[] Answer: \answerYes{} 
    \item[] Justification: Please see Appendix~\ref{sec:lim}.
    
\item {\bf Safeguards}
    \item[] Question: Does the paper describe safeguards that have been put in place for responsible release of data or models that have a high risk for misuse (e.g., pre-trained language models, image generators, or scraped datasets)?
    \item[] Answer: \answerNA{} 
    \item[] Justification: We think our work poses no such risks.

\item {\bf Licenses for existing assets}
    \item[] Question: Are the creators or original owners of assets (e.g., code, data, models), used in the paper, properly credited and are the license and terms of use explicitly mentioned and properly respected?
    \item[] Answer: \answerYes{} 
    \item[] Justification: All datasets used in this paper are all open-source, and we have cited the corresponding papers.

\item {\bf New assets}
    \item[] Question: Are new assets introduced in the paper well documented and is the documentation provided alongside the assets?
    \item[] Answer: \answerNA{} 
    \item[] Justification: Our work proposes a novel few-shot anomaly detection framework without releasing new assets.

\item {\bf Crowdsourcing and research with human subjects}
    \item[] Question: For crowdsourcing experiments and research with human subjects, does the paper include the full text of instructions given to participants and screenshots, if applicable, as well as details about compensation (if any)? 
    \item[] Answer: \answerNA{} 
    \item[] Justification: Our paper do not involve crowdsourcing nor research with human subjects.

\item {\bf Institutional review board (IRB) approvals or equivalent for research with human subjects}
    \item[] Question: Does the paper describe potential risks incurred by study participants, whether such risks were disclosed to the subjects, and whether Institutional Review Board (IRB) approvals (or an equivalent approval/review based on the requirements of your country or institution) were obtained?
    \item[] Answer: \answerNA{} 
    \item[] Justification: Our paper do not involve crowdsourcing nor research with human subjects.

\item {\bf Declaration of LLM usage}
    \item[] Question: Does the paper describe the usage of LLMs if it is an important, original, or non-standard component of the core methods in this research? Note that if the LLM is used only for writing, editing, or formatting purposes and does \emph{not} impact the core methodology, scientific rigor, or originality of the research, declaration is not required.
    \item[] Answer: \answerNA{} 
    \item[] Justification: The LLM is used only for editing and polishing the manuscript.

\end{enumerate}
\end{document}